\documentclass{article}

\usepackage{microtype}
\usepackage{graphicx}
\usepackage{subcaption}
\usepackage{booktabs} 
\usepackage{algorithm}
\usepackage{algorithmic}
\let\algorithmicrequire\relax
\usepackage{mathtools}
\mathtoolsset{showonlyrefs}
\newcommand{\R}{\mathbb{R}}
\newcommand{\E}{\mathbb{E}}
\newcommand{\cV}{\mathcal{V}}
\usepackage{multirow}
\usepackage{tabularx}

\usepackage{float}



\usepackage{hyperref}
\providecommand{\RETURN}{\textbf{return}~}

\usepackage[preprint]{icml2026}

\usepackage{amsmath}
\usepackage{amssymb}
\usepackage{mathtools}
\usepackage{amsthm}
\usepackage{enumitem}

\usepackage[capitalize,noabbrev]{cleveref}

\theoremstyle{plain}
\newtheorem{theorem}{Theorem}[section]
\newtheorem{proposition}[theorem]{Proposition}
\newtheorem{lemma}[theorem]{Lemma}
\newtheorem{corollary}[theorem]{Corollary}
\theoremstyle{definition}

\newtheorem{assumption}[theorem]{Assumption}
\theoremstyle{remark}
\newtheorem{remark}[theorem]{Remark}
\newcommand{\cS}{\mathcal{S}}

\usepackage[textsize=tiny]{todonotes}

\usepackage[table]{xcolor}
\makeatletter
\g@addto@macro\normalsize{%
  \setlength\abovedisplayskip{4pt plus 2pt minus 2pt}
  \setlength\belowdisplayskip{4pt plus 2pt minus 2pt}
  \setlength\abovedisplayshortskip{2pt plus 2pt minus 2pt}
  \setlength\belowdisplayshortskip{2pt plus 2pt minus 2pt}
}
\makeatother

\icmltitlerunning{DARC: Disagreement-Aware Alignment via Risk-Constrained Decoding}
\makeatletter
\renewcommand\paragraph{\@startsection{paragraph}{4}{\z@}%
  {0.3ex \@plus .2ex \@minus .1ex}
  {-1em}%
  {\normalfont\normalsize\bfseries}}
\makeatother

\newcommand{\cY}{\mathcal{Y}}                 
\newcommand{\Vhat}{\widehat{V}_\beta}         
\newcommand{\RPhat}{\widehat{\mathrm{RP}}_{\beta}} 
\newcommand{\epstie}{\epsilon}                

\newcommand{\best}[1]{\cellcolor{blue!20}\textbf{#1}}
\newcommand{\second}[1]{\cellcolor{green!15}\textbf{#1}}

\begin{document}

\twocolumn[
  \icmltitle{DARC: Disagreement-Aware Alignment via Risk-Constrained Decoding}

  \begin{icmlauthorlist}
    \icmlauthor{Mingxi Zou}{fudan}
    \icmlauthor{Jiaxiang Chen}{fudan}
    \icmlauthor{Junfan Li}{independent}
    \icmlauthor{Langzhang Liang}{fudan}
    \icmlauthor{Qifan Wang}{meta}
    \icmlauthor{Yinghui Xu}{fudan}
    \icmlauthor{Zenglin Xu}{fudan,aiiii}
  \end{icmlauthorlist}

  \icmlaffiliation{fudan}{Fudan University, Shanghai, China}
  \icmlaffiliation{independent}{Independent Researcher}
  \icmlaffiliation{meta}{Meta AI}
  \icmlaffiliation{aiiii}{Artificial Intelligence Innovation and Incubation Institute, Fudan University, Shanghai, China}

  \icmlcorrespondingauthor{Zenglin Xu}{zenglinxu@fudan.edu.cn}

  \icmlkeywords{Machine Learning, ICML, Alignment, Robust Decoding, Preference Learning}

  \vskip 0.3in
]


\printAffiliationsAndNotice{}
\begin{abstract}
Preference-based alignment methods (e.g., RLHF, DPO) typically optimize a single scalar objective, implicitly averaging over heterogeneous human preferences. In practice, systematic annotator and user-group disagreement makes mean-reward maximization brittle and susceptible to proxy over-optimization. We propose \textbf{Disagreement-Aware Alignment via Risk-Constrained Decoding (DARC)}, a retraining-free inference-time method that frames response selection as distributionally robust, risk-sensitive decision making. Given multiple preference samples or scalable disagreement proxies, DARC reranks candidates by maximizing a \emph{KL-robust (entropic)} satisfaction objective, and provides simple deployment controls that cap or penalize the corresponding entropic risk premium relative to the mean, enabling explicit risk budgets without retraining. We provide theoretical characterization linking this decoding rule to principled pessimism and KL-based distributionally robust optimization. Experiments on alignment benchmarks show that DARC reduces disagreement and tail risk while maintaining competitive average quality under noisy, heterogeneous feedback.

\end{abstract}
{\setlength{\intextsep}{4pt}
\begin{figure}[t]
\centering
\includegraphics[width=\columnwidth]{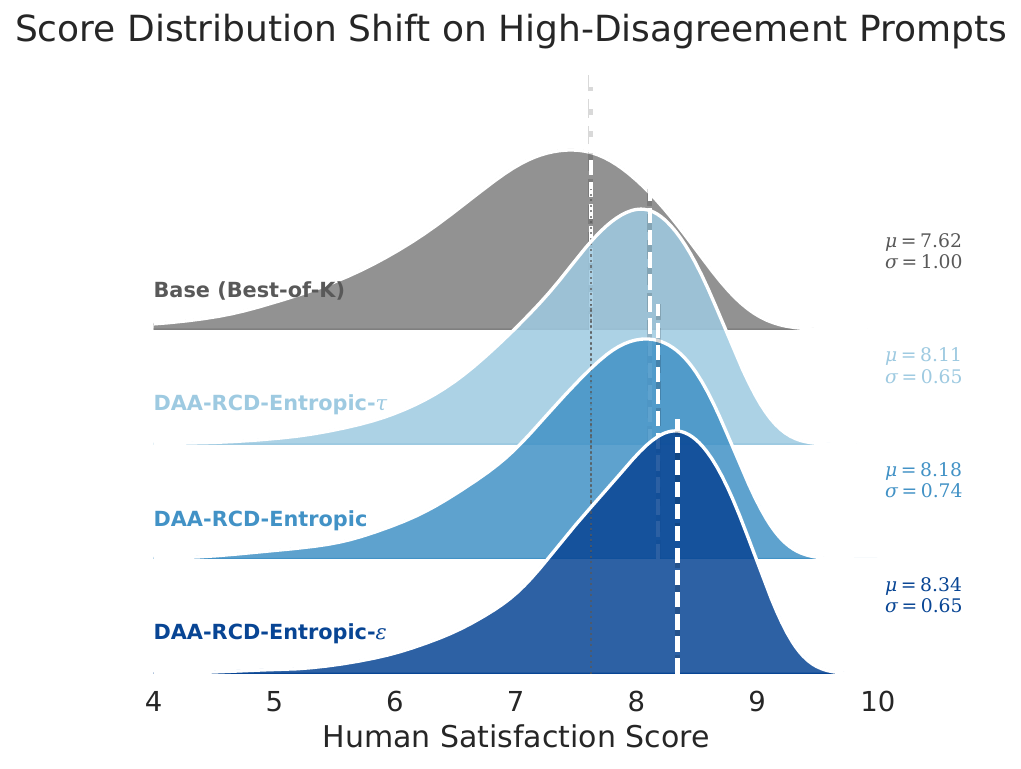}
\caption{\textbf{Score Distribution shift.} Ridge plot showing human score densities on the high-disagreement subset. DARC variants (blue) shift the distribution to the right (higher mean $\mu$) compared to the baseline (grey), with reduced spread (lower $\sigma$), indicating both increased satisfaction and reduced disagreement.}
\label{fig:ridge_dist}
\end{figure}
}

\section{Introduction}

Preference data has become the dominant supervision signal for aligning large language models~\citep{ouyang2022training,stiennon2020learning}.
Most pipelines---RLHF with reward modeling and RL optimization~\citep{christiano2017deep,schulman2017proximal},
offline preference objectives such as DPO and its refinements~\citep{rafailov2023direct,meng2024simpo},
and reference-free single-stage variants such as ORPO~\citep{hong2024orpo}---share a common abstraction:
preferences are treated as noisy observations of a \emph{single latent scalar utility} (e.g., Bradley--Terry)~\citep{bradley1952rank}.
This abstraction largely persists in newer reformulations such as KTO and IPO~\citep{ethayarajh2024kto,garg2025ipo},
and even when reward models are made multi-dimensional via multi-head/objective designs~\citep{wang2024interpretable,li2025multi,yang2024metaaligner}.
Yet, treating feedback as perturbations around a single scalar provides limited guidance for \emph{inference-time} response selection under heterogeneous preferences~\citep{hung2025reward},
and it also lacks a unifying robust-optimization account for common risk-penalized decoding heuristics.

However, real-world preferences are often \emph{heterogeneous} rather than i.i.d.\ noise:
annotators disagree for systematic reasons~\citep{zhang2024diverging,chen2024pal}.
Empirically, human ratings show substantial variance even on the raw Top-$K$ candidate pool (Appendix~\ref{app:human-disagreement-topk}),
suggesting that uncertainty is intrinsic rather than an artifact of our selection rule.
Under such plurality, maximizing the \emph{average} reward $\hat{\mu}$ can be brittle~\citep{casper2023open},
and the issue is exacerbated by proxy over-optimization, which can improve an imperfect preference proxy while degrading the underlying target~\citep{gao2023scaling,rafailov2024scaling}.

Recent work further shows that proxy misspecification can induce \emph{inference-time} reward hacking:
as best-of-$N$ (a widely used decoding primitive~\citep{sun2024fast}) or soft best-of-$N$ becomes greedier,
true utility can increase and then inevitably degrade~\citep{huang2025best}.
Best-of-Poisson and HedgeTune mitigate this effect by tuning inference-time parameters~\citep{khalaf2025inference};
however, they primarily model risk through proxy--distortion tradeoffs under a \emph{single} reward signal, rather than preference heterogeneity.
Closely related, pessimistic best-of-$N$ rules penalize atypical candidates via an auxiliary error model to mitigate reward hacking~\citep{mitigatingiclr26},
but they target distributional uncertainty of the reward model (e.g., atypicality/OOD) rather than disagreement-grounded risk across users.

In parallel, robustness has been pursued through \emph{training-time} objectives such as robust DPO under noisy preferences~\citep{wu2024towards},
which improve robustness via retraining and noise assumptions;
group-robust objectives that protect minority preference groups~\citep{ramesh2024group}, which rely on access to group structure;
and uncertainty-aware reward modeling~\citep{banerjee2024towards}, which quantifies estimation uncertainty but does not by itself specify how to select responses under inference-time proxy shift~\citep{ichihara2025evaluation}.
Taken together, these lines suggest a common lesson:
when preferences are plural, the relevant object is not a deterministic score, but a \emph{random variable} over users and annotation noise.
Yet, a principled \emph{inference-time} selection rule that is explicitly risk-constrained under heterogeneous preferences remains underdeveloped.

We therefore study inference-time alignment under heterogeneous preferences through the lens of \emph{risk-constrained decision making}~\citep{chow2018risk,tamar2015policy}.
Given a fixed candidate set and noisy preference or reward scores, we derive a finite-sample \emph{pessimistic} rule based on a lower confidence bound,
yielding high-probability guarantees for selecting a competitive response while controlling tail risk across prompts.
This leads to \textbf{Disagreement-Aware Alignment via Risk-Constrained Decoding (DARC)}, an inference-time-only, retraining-free procedure that plugs into any LM and preference estimator.
DARC grounds risk in \emph{multi-annotator disagreement}~\citep{zhang2024diverging} instantiated via validated proxy signals,
improving robustness on high-disagreement prompts (Fig.~\ref{fig:ridge_dist}).
We include representative cases where Best-of-$K$ is polarizing or unstable, whereas DARC yields consistently preferred responses (Appendix~\ref{app:positive_case}).

Beyond the statistical view, we give a distributionally robust optimization (DRO) characterization~\citep{wiesemann2014distributionally,rahimian2019distributionally},
viewing decoding as maximizing the worst-case expected satisfaction over local divergence neighborhoods~\citep{namkoong2016stochastic,duchi2021statistics}.
This yields a practical KL-robust instantiation and situates widely used mean--dispersion scoring rules within the same DRO perspective, 
clarifying the conditions under which they arise as principled risk-sensitive criteria~\citep{duchi2019variance}.

\paragraph{Contributions.}
\begin{itemize}[leftmargin=*, itemsep=1pt, topsep=2pt, parsep=0pt, partopsep=0pt]
    \item \textbf{Method.} We formulate inference-time alignment as \emph{risk-constrained} decision making under heterogeneous preferences, with risk induced by preference uncertainty and annotator disagreement.
    \item \textbf{Theory.} We connect LCB-based uniform pessimism to a KL-DRO view, yielding a closed-form entropic decoding objective and its constrained/penalized variants via an entropic risk premium.
    \item \textbf{Empirics.} Across benchmarks, DARC improves disagreement and prompt-level tail risk with competitive mean quality; a \emph{dual-robust} multi-scorer extension hedges scorer shift and proxy over-optimization, with a KL-regularized DRO interpretation over scorers.
\end{itemize}

\section{Problem setup}
\label{sec:prelim}

Let $s$ denote a prompt (context) and let $\cY(s)$ be a realized candidate set produced by a fixed generator (e.g., sampling, beam variants, or a proposal model), with $K:=|\cY(s)|$.

\paragraph{Conditioning on the candidate set.}
The generator may be stochastic; throughout the analysis we condition on the realized $\cY(s)$.
All probabilities below are taken over the evaluation randomness (human or otherwise), holding $\cY(s)$ fixed.

\paragraph{Latent satisfaction under heterogeneous preferences.}
For each $(s,y)$, let $R(s,y)\in\mathbb{R}$ denote a (latent) user-satisfaction random variable capturing preference heterogeneity and evaluation noise, with mean $\mu(s,y):=\mathbb{E}[R(s,y)]$.
Intuitively, $\mu(s,y)$ measures average quality.

\paragraph{KL-robust (entropic) value and risk premium.}
For $\beta>0$, define the entropic value
\begin{equation}
V_\beta(s,y)
:= -\frac{1}{\beta}\log \mathbb{E}\!\left[\exp\!\big(-\beta\,R(s,y)\big)\right],
\label{eq:true-entropic}
\end{equation}
which is equivalent to a KL-based distributionally robust objective (Section~\ref{sec:dro}).
We define the entropic risk premium. 
\begingroup
\setlength{\belowdisplayskip}{1pt}
\setlength{\belowdisplayshortskip}{1pt}
\begin{equation}
\mathrm{RP}_\beta(s,y):=\mu(s,y)-V_\beta(s,y)\ge 0    
\end{equation}
\endgroup
\paragraph{Decision problem (risk-aware decoding).}
Conditioning on $(s,\cY(s))$, each candidate $y\in\cY(s)$ induces an (unknown) satisfaction distribution over users/raters.
Our population objective is to select an output by solving
\begin{equation}
y^\star \in \arg\max_{y\in\cY(s)} V_\beta(s,y),
\label{eq:problem-entropic}
\end{equation}
where entropic risk measure $V_\beta$ is defined in \eqref{eq:true-entropic}.
We further consider explicit risk control via a budget or a penalty:
\begin{equation}
y^\star_\tau \in \arg\max_{y\in\cY(s)} V_\beta(s,y)
\quad \text{s.t.}\quad \mathrm{RP}_\beta(s,y)\le \tau,
\label{eq:problem-budget}
\end{equation}
or in penalized (Lagrangian) form
\[
  \arg\max_{y\in\cY(s)} V_\beta(s,y) - \lambda\,\mathrm{RP}_\beta(s,y).
\]

\section{Guarantees via Lower Confidence Bounds}
\label{sec:lcb}
This section provides a statistical justification for disagreement-aware decoding by deriving high-probability lower confidence bounds (LCBs) on expected satisfaction under heterogeneous preferences.
To contextualize the resulting pessimistic rules, we also give complementary distributionally robust optimization (DRO) characterizations.

\paragraph{Scalar satisfaction samples (guarantee setting).}
For each candidate $y\in\cY(s)$, we observe $n$ i.i.d.\ scalar satisfaction samples
$\{R_i(s,y)\}_{i=1}^n$ drawn from the (unknown) distribution of $R(s,y)$:
\begin{equation}
R_i(s,y) \overset{i.i.d.}{\sim} R(s,y),\qquad i=1,\ldots,n.
\label{eq:iid-samples}
\end{equation}
In this regime, the empirical mean and standard deviation
$\hat\mu_n(s,y)$ and $\hat\sigma_n(s,y)$ estimate $\mu(s,y)$ and $\sigma(s,y)$.
Moreover, the plug-in estimator of \eqref{eq:true-entropic} is
\begin{equation}
\widehat{V}_\beta(s,y)
:= -\frac{1}{\beta}\log\!\left(\frac{1}{n}\sum_{i=1}^n \exp\!\big(-\beta\,R_i(s,y)\big)\right),
\label{eq:emp-entropic}
\end{equation}
with the empirical risk premium
\begin{equation}
\RPhat(s,y) := \hat\mu_n(s,y) - \widehat{V}_\beta(s,y)\ \ge 0.
\label{eq:emp-entropic-risk-premium}
\end{equation}
In practice, one may operate with approximate (proxy) estimates of these empirical quantities; the analysis below does not rely on proxy scores being independent, and any such approximation induces only an additive slack in the resulting LCB objective (Appendix~\ref{app:proxy-robust}).

\paragraph{Theory--proxy interface.}
The guarantees in Section~\ref{sec:lcb} are stated for the clean
human-sample regime, where each candidate is evaluated by i.i.d. scalar
satisfaction samples. This regime provides the decision-theoretic basis for the
LCB/KL-DRO connection. In scalable deployment, direct multi-rater samples are
usually unavailable, so we instantiate the same decoding rules with
perturbation-based reward-model statistics as a surrogate for disagreement. We
do not claim that this proxy literally inherits the i.i.d. human-sample
guarantees. Rather, Appendix~\ref{app:proxy-robust} shows that if the proxy
uniformly tracks the empirical mean and risk statistics, the induced pessimistic
objective differs from the sample-based objective by only an additive slack.
Empirically, we validate this approximation by comparing proxy disagreement to
held-out human disagreement in Section~\ref{sec:experiments}.

\subsection{Estimation risk: uniform LCB and a mean--dispersion surrogate}

Fix a prompt $s$ and condition on $\mathcal{Y}=\cY(s)$ with $K:=|\mathcal{Y}|$.
For readability, write $R_i(y)$ for $R_i(s,y)$, and $\mu(y),\sigma(y)$ for $\mu(s,y),\sigma(s,y)$.
For each $y\in\mathcal{Y}$, we observe $n$ independent satisfaction samples with mean $\mu(y)$ and variance $\sigma^2(y)$, and denote the empirical mean and standard deviation by
$\hat\mu_n(y):=\frac{1}{n}\sum_{i=1}^n R_i(y)$ and
$\hat\sigma_n^2(y):=\frac{1}{n-1}\sum_{i=1}^n (R_i(y)-\hat\mu_n(y))^2$.

\paragraph{Bridge to pairwise preferences.}
Our scalar-sample analysis can be viewed as operating on standard scalarizations of pairwise preferences (e.g., win-rate or fitted BT/Thurstone scores); see Appendix~\ref{app:pairwise_bridge}. Empirically, our main conclusions are stable when replacing absolute scalar ratings with a pairwise scalarization against the base response (Table~\ref{tab:scalarization_robust}).
\begin{remark}[Shared annotators across candidates]
Our guarantee requires independence across $i$ \emph{for each fixed candidate $y$}.
Annotator overlap \emph{across candidates} may induce dependence between $\{(\hat\mu_n(y),\hat\sigma_n(y))\}_{y\in\mathcal{Y}}$,
but the uniform guarantee below follows from per-candidate concentration with a union bound and does not require independence across candidates.
\end{remark}
\vspace{-0.5em}
\begin{assumption}[Bounded rewards]
For all $y\in\mathcal{Y}$, the satisfaction samples are almost surely bounded:
\begin{equation}
R_i(y)\in[a,b]\quad\text{a.s.}
\label{eq:bounded}
\end{equation}
\end{assumption}

Boundedness matches typical rating scales and can be enforced for proxy scores via truncation (Appendix~H.1); an analogous LCB holds under sub-Gaussian noise (Appendix~\ref{app:tail}).

\begin{proposition}[Uniform LCB]
\label{prop:uniform-lcb}
There exists an absolute constant $c>0$ such that for any $\delta\in(0,1)$, with probability at least $1-\delta$, simultaneously for all $y\in\mathcal{Y}(s)$,
\begin{equation}
\mu(y) \ge \hat\mu_n(y)
- c\,\hat\sigma_n(y)\sqrt{\frac{\log(K/\delta)}{n}}
- c\,(b-a)\frac{\log(K/\delta)}{n}.
\label{eq:lcb-bounded}
\end{equation}
We denote the right-hand side by $\mathrm{LCB}_\delta(y)$.
\end{proposition}

\noindent\textit{Proof.} See Appendix~\ref{app:proof-lcb}.

\begin{remark}[Variance governs estimation hardness]
The LCB form \eqref{eq:lcb-bounded} highlights a statistical driver of selection risk: the dominant estimation error scales with the standard deviation $\sigma(y)$.
The lower-order term decays as $O(1/n)$, whereas the leading term scales as $O(\sigma(y)/\sqrt{n})$.
Consequently, low-disagreement candidates (small $\sigma(y)$) admit substantially tighter confidence bounds at the same sample size, while controversial candidates require many more samples to certify their mean.
Penalizing $\hat\sigma_n(y)$ therefore has an identification rationale: it discourages selecting candidates whose true quality is intrinsically harder to verify from limited feedback.
\end{remark}

\paragraph{Lower-tail interpretation.}
Maximizing a lower confidence bound is a principled way to avoid candidates with poor lower-tail satisfaction even when $\hat\mu_n(y)$ is high~\citep{Boucheron2013Concentration,vershynin2018high}.
This is a statistical conservatism argument rather than an equivalence to coherent tail-risk measures such as CVaR~\citep{artzner1999coherent,rockafellar2000optimization}.
See Appendix~\ref{app:tail} for details.
\paragraph{On constants and practical calibration.}
The uniform LCB in~\eqref{eq:lcb-bounded} is obtained via standard concentration plus a union bound over $K$ candidates and is therefore conservative.
We do not claim tight constants; instead, the bound motivates the \emph{functional form}
\begin{equation}
\lambda_\delta \;\propto\; \sqrt{\frac{\log(K/\delta)}{n}},
\end{equation}
and the lower-order term $c(b-a)\tfrac{\log(K/\delta)}{n}$ is uniform across candidates and does not affect $\arg\max_y$ decisions.
In practice, we treat the coefficient as a risk-budget knob (optionally scaled by a factor $\alpha$) and fix it via a small held-out calibration, while reporting sensitivity of the resulting trade-off. Combined with the $\hat\sigma_n(y)$ factor in~\eqref{eq:lcb-bounded}, the uncertainty penalty scales as $\widetilde O(\hat\sigma_n(y)\sqrt{\log K/n})$, matching the intuition that more controversial candidates require larger risk budgets to be selected.
\paragraph{LCB decoding and a mean--dispersion surrogate.}
We define the LCB decoder as 
\begin{equation}
    y_{\mathrm{LCB}} \in \arg\max_{y\in\mathcal{Y}} \mathrm{LCB}_\delta(y).
\end{equation}
On the high-probability event of Proposition~\ref{prop:uniform-lcb}, this rule maximizes a valid lower bound on the true mean satisfaction $\mu(y)$.
Crucially, under bounded ratings, maximizing $\mathrm{LCB}_\delta(y)$ is equivalent (up to a constant) to a \emph{mean--dispersion} surrogate:
\[
\arg\max_{y\in\mathcal{Y}} \big(\hat\mu_n(y)-\lambda\hat\sigma_n(y)\big).
\]
We provide the derivation in Corollary~\ref{cor:lcb-meanvar} (Appendix~\ref{app:proof-corollary}) and discuss the theoretical connection between this $\sigma$-penalty (estimation uncertainty) and the entropic $\sigma^2$-penalty (risk aversion) in Remark~\ref{rem:std-vs-var}.

\subsection{Distributional risk: DRO characterizations of pessimistic value}

\label{sec:dro}
\paragraph{KL-robust (entropic) decoding.}
We first consider a relative-entropy (KL) robustification of expected satisfaction.
Let $R\in\mathbb{R}$ be the (latent) satisfaction with reference distribution $\mathbb{P}$.
We consider worst-case distributions $\mathbb{Q}$ over $R$ that are absolutely continuous w.r.t.\ $\mathbb{P}$.
For $\beta>0$, define the KL-regularized robust value
\begin{equation}
\mathrm{Rob}^{\mathrm{KL}}_\beta(\mathbb{P};R)
:= \inf_{\mathbb{Q}\ll \mathbb{P}}
\Big\{ \mathbb{E}_{\mathbb{Q}}[R] + \beta^{-1} D_{\mathrm{KL}}(\mathbb{Q}\|\mathbb{P}) \Big\},
\label{eq:kl-rob-def}
\end{equation}
where $\mathbb{Q}\ll\mathbb{P}$ ensures the KL term is well-defined.
This is standard in risk-sensitive control and large deviations~\citep{dupuis2011weak,hansen2011robustness}.

\begin{theorem}[KL-robust value equals an entropic objective]
\label{thm:kl-entropic}
For any $\beta>0$ and any $\mathbb{P}$ such that $\mathbb{E}_{\mathbb{P}}[\exp(-\beta R)]<\infty$,
\begin{equation}
\mathrm{Rob}^{\mathrm{KL}}_\beta(\mathbb{P};R)
\;=\;
-\frac{1}{\beta}\log \mathbb{E}_{R\sim\mathbb{P}}\!\left[\exp(-\beta R)\right].
\label{eq:kl-entropic-closed}
\end{equation}
In particular, for the empirical rater distribution $\widehat{\mathbb{P}}_n^y$,
$\mathrm{Rob}^{\mathrm{KL}}_\beta(\widehat{\mathbb{P}}_n^y;R)
= -\frac{1}{\beta}\log\!\big(\frac{1}{n}\sum_{i=1}^n \exp(-\beta R_i(y))\big)
= \widehat{V}_\beta(s,y)$.See Appendix~\ref{app:proof-kl-entropic} for the detailed proof.
\end{theorem}

\paragraph{Relation to Expected Utility Theory (CARA).}
Maximizing the entropic value in \eqref{eq:kl-entropic-closed} is equivalent to maximizing expected utility under a \emph{constant absolute risk aversion} (CARA) utility, $u(x)=-\exp(-\beta x)$.
A useful invariance is \emph{translation equivariance}: for any constant $c$, $V_\beta(R+c)=V_\beta(R)+c$, hence the risk premium $\mu(R)-V_\beta(R)$ is invariant to additive reward shifts.
This is desirable in RLHF-style pipelines where learned reward models are often only identifiable up to an affine transformation.

\paragraph{$\chi^2$-DRO yields a mean--dispersion special case.}
We now give a complementary robust-optimization view~\citep{wiesemann2014distributionally,rahimian2019distributionally}:
under a $\chi^2$-divergence ambiguity set around $\mathbb{P}$~\citep{ben2002robust},
the worst-case expected satisfaction admits a mean--dispersion form.
(Definitions and tightness conditions are deferred to Appendix~\ref{app:chi2-dro}.)

\begin{proposition}[$\chi^2$-DRO robust mean admits a mean--dispersion form]
\label{prop:chi2-dro}
Let $R$ be square-integrable under $\mathbb{P}$ with mean $\mu_{\mathbb{P}}:=\mathbb{E}_{\mathbb{P}}[R]$
and variance $\sigma_{\mathbb{P}}^2:=\mathrm{Var}_{\mathbb{P}}(R)$.
For any $\rho\ge 0$,
\begin{equation}
\inf_{\mathbb{Q}\in \mathcal{U}_\rho(\mathbb{P})}\ \mathbb{E}_{\mathbb{Q}}[R]
\;\ge\;
\mu_{\mathbb{P}} \;-\; \sqrt{\rho}\,\sigma_{\mathbb{P}}.
\label{eq:chi2-dro-lower}
\end{equation}
Specializing to $\mathbb{P}=\widehat{\mathbb{P}}_n^y$ yields the empirical closed form
\begin{equation}
\inf_{\mathbb{Q}\in \mathcal{U}_\rho(\widehat{\mathbb{P}}_n^y)}\ \mathbb{E}_{\mathbb{Q}}[R]
=\widehat{\mu}_n(y)-\sqrt{\rho}\,\sqrt{\widehat{v}_n(y)},
\label{eq:chi2-dro-empirical}
\end{equation}
whenever the extremal density is nonnegative on the empirical support. Otherwise, the mean--dispersion form remains a valid lower bound (Remark~\ref{rem:dro-gap}), ensuring the rule remains pessimistic.
\end{proposition}

\noindent\textit{Proof and tightness conditions.}
See Appendix~\ref{app:chi2-dro} for the extremal density characterization and a sufficient nonnegativity regime (e.g., bounded ratings imply tightness for small $\rho$).
\paragraph{A unified pessimistic value (LCB as calibrated DRO).}
Define the candidate-wise pessimistic value
\begin{equation}
V_{\delta}(y)
:=
\widehat{\mu}_n(y)
- c\,\widehat{\sigma}_n(y)\sqrt{\frac{\log(K/\delta)}{n}}
- c\,(b-a)\frac{\log(K/\delta)}{n},
\label{eq:unified-pess-value}
\end{equation}
which is a uniform lower confidence bound (LCB) under Proposition~\ref{prop:uniform-lcb} (Appendix~\ref{app:proof-lcb}).
On the other hand, Proposition~\ref{prop:chi2-dro} implies that, up to the lower-order term in~\eqref{eq:unified-pess-value},
$V_\delta(y)$ matches the $\chi^2$-DRO robust mean of the empirical rater distribution with radius
\begin{equation}
\rho_\delta \asymp \frac{\log(K/\delta)}{n},
\label{eq:rho-calibration}
\end{equation}
since $\sqrt{\rho_\delta}\,\widehat{\sigma}_n(y)\asymp \widehat{\sigma}_n(y)\sqrt{\log(K/\delta)/n}$ (see Appendix~\ref{app:lcb-dro-calibration} for a discussion distinguishing statistical estimation risk from intrinsic disagreement).
This yields a unified interpretation of disagreement-aware decoding as \emph{calibrated pessimism}:a statistical LCB rule that is equivalent (up to lower-order terms) to optimizing a local DRO objective. Thus, disagreement-aware decoding can be viewed equivalently as
(i) maximizing a high-probability LCB on expected satisfaction (statistical pessimism),
or (ii) maximizing a distributionally robust worst-case expected satisfaction over a local $\chi^2$ neighborhood (adversarial robustness).


\section{Disagreement-Aware Risk-Constrained Decoding}
\label{sec:DARC}
Motivated by the LCB and DRO characterization in \S\ref{sec:lcb}--\S\ref{sec:dro}, we now present practical decoding rules that implement the \emph{KL-robust entropic objective} (and its constrained/penalized variants) over a finite candidate set.
For completeness, we also include a \emph{second-moment} pessimistic surrogate implied by the LCB analysis in \eqref{eq:meanvar} as a computationally convenient ablation.

\vspace{-0.5em}
\subsection{From preference maximization to risk controls}
\vspace{-0.5em}
\label{sec:formulation}
Standard reward-based reranking selects the highest estimated mean,
\begin{equation}
y_{\text{mean}}(s)\in\arg\max_{y\in\cY(s)} \hat\mu(s,y),
\label{eq:mean-decode}
\end{equation}
which optimizes \emph{average} preference but may favor brittle outputs when disagreement is large.
\vspace{-0.5em}
\paragraph{Risk-sensitive decoding (primary rule).}
Our primary decoder selects the candidate with the largest \emph{scorer-robust} entropic value:
\begin{equation}
y_{\text{Entropic}}(s)
\in
\arg\max_{y\in\cY(s)} \widehat V_{\beta}(s,y),
\label{eq:entropic-decode}
\end{equation}
where $\widehat V_{\beta}$ is the entropic value defined in~\eqref{eq:emp-entropic}.

\vspace{-0.5em}
\paragraph{Risk-constrained decoding via an entropic risk premium.}
We define explicit deployment knobs using the entropic risk premium
$\widehat{\mathrm{RP}}_{\beta}(s,y)$ from~\eqref{eq:emp-entropic-risk-premium}.
We consider two standard forms:
\begingroup
\setlength{\abovedisplayskip}{4pt}
\setlength{\belowdisplayskip}{4pt}
\setlength{\abovedisplayshortskip}{2pt}
\setlength{\belowdisplayshortskip}{2pt}
\begin{align}
y_{\tau}(s)
&\in \arg\max_{y\in\cY(s)} \widehat V_{\beta}(s,y)
\quad \text{s.t.}\quad \widehat{\mathrm{RP}}_{\beta}(s,y) \le \tau,
\label{eq:entropic-constraint}\\
y_{\lambda}(s)
&\in \arg\max_{y\in\cY(s)} \widehat V_{\beta}(s,y) - \lambda\, \widehat{\mathrm{RP}}_{\beta}(s,y),
\label{eq:entropic-penalty}
\end{align}
\endgroup
where the penalty form is a Lagrangian relaxation.
If the feasible set $\{y\in\cY(s):\widehat{\mathrm{RP}}_{\beta}(s,y)\le\tau\}$ is empty, we fall back to
$\cY(s)$, consistent with Algorithm~\ref{alg:DARC-full}.
The constrained and penalized forms provide two equivalent ways to tune the reward--risk trade-off; their finite-candidate-set relationship is given in Appendix~\ref{app:lagrangian}.


\vspace{-0.5em}
\paragraph{Disagreement as a risk proxy: interpretation and limitations.}
We use disagreement as a human-centric signal of preference heterogeneity, and implement risk sensitivity through the KL-robust entropic value~\eqref{eq:entropic-decode} (a standard convex risk measure).
The mean--dispersion form arises only as a finite-sample surrogate (via LCB pessimism) or as a $\chi^2$-DRO special case, rather than being the defining principle of our method.

Importantly, we do \emph{not} claim that variance is a coherent risk measure (e.g., CVaR)~\citep{artzner1999coherent,rockafellar2000optimization} or that $\sigma$ is \emph{equivalent} to a tail-risk functional in full generality.
Instead, \S\ref{sec:lcb} shows that penalizing $\hat\sigma$ arises naturally from maximizing a statistically justified \emph{lower confidence bound} (LCB) on $\mu(s,y)$, yielding a principled \emph{pessimistic} selection rule under finite data.
Moreover, \S\ref{sec:dro} provides a complementary \emph{distributionally robust optimization} characterization~\citep{wiesemann2014distributionally,rahimian2019distributionally}.

\begin{table*}[t]
\centering
\small
\setlength{\tabcolsep}{2.5pt} 
\renewcommand{\arraystretch}{1.1}
\resizebox{\textwidth}{!}{
\begin{tabular}{l l cccc cccc @{\hspace{12pt}} cccc cccc}
\toprule
\multirow{3}{*}{\textbf{Type}} & \multirow{3}{*}{\textbf{Method}}
& \multicolumn{8}{c}{\textbf{Llama-3.1-8B-Instruct}}
& \multicolumn{8}{c}{\textbf{Qwen2.5-7B-Instruct}} \\
\cmidrule(lr){3-10} \cmidrule(lr){11-18}

& & \multicolumn{4}{c}{\textbf{MT-Bench}} & \multicolumn{4}{c}{\textbf{AlpacaEval 2.0}} 
& \multicolumn{4}{c}{\textbf{MT-Bench}} & \multicolumn{4}{c}{\textbf{AlpacaEval 2.0}} \\
\cmidrule(lr){3-6} \cmidrule(lr){7-10} \cmidrule(lr){11-14} \cmidrule(lr){15-18}

& & \textbf{Reward} & \textbf{Risk}$\downarrow$ & \textbf{Tradeoff}$\uparrow$ & \textbf{Len(tok)} 
  & \textbf{Reward} & \textbf{Risk}$\downarrow$ & \textbf{Tradeoff}$\uparrow$ & \textbf{Len(tok)}
  & \textbf{Reward} & \textbf{Risk}$\downarrow$ & \textbf{Tradeoff}$\uparrow$ & \textbf{Len(tok)} 
  & \textbf{Reward} & \textbf{Risk}$\downarrow$ & \textbf{Tradeoff}$\uparrow$ & \textbf{Len(tok)} \\
\midrule

\multicolumn{18}{l}{\textit{\textbf{Dataset: Overall}}} \\ \addlinespace[2pt]

\multirow{5}{*}{\textit{Inference}} 
& Base (Best-of-$K$)        & 4.53 & 6.50 & -8.47 & 315 & 4.52 & 7.46 & -10.40 & 331 & 4.09 & 3.54 & -2.99 & 244 & 3.29 & 3.11 & -2.93 & 256 \\
& BoP(HedgeTune)        & 4.18 & 5.86 & -7.54 & 316 & 4.12 & 6.24 & -8.36 & 330 & 3.84 & 3.10 & -2.36 & 231 & 3.14 & 2.68 & -2.22 & 261 \\
& Caution          & 4.09 & 5.43 & -6.77 & 317 & 4.27 & 5.91 &  -7.55 & 320 & 3.89 & 3.21 & -2.53 & 243 & 3.09 & 2.62 & -2.15 & 259 \\
& DeAL                      & 5.59 & 7.98 & -10.37 & 309 & 5.66 & 8.74 &  -11.81 & 311 & 4.64 & 4.22 & -3.80 & 256 & 4.02 & 3.74 & -3.46 & 267 \\
& MC-Dropout                & 4.41 & 5.78 & -7.15 & 315 & 4.37 & 6.31 &  -8.25 & 326 & 3.87 & 3.38 & -2.89 & 247 & 3.13 & 2.79 & -2.45 & 253 \\
& RBoN                      & 4.43 & 5.62 & -6.81 & 309 & 4.31 & 5.82 &  \second{-7.33} & 321 & 3.93 & 3.22 & -2.51 & 241 & 3.17 & 2.71 & -2.25 & 257 \\
& DARC(ours)          & 4.46 & \second{5.41} & \second{-6.36} & 315 & 4.25 &\second{ 5.67} &  \second{-7.33} & 319 & 4.01 & 3.14 & -2.27 & 243 & 3.14 & 2.66 & -2.18 & 256 \\
& DARC-$\tau$(ours)    & 4.42 & 5.45 & -6.48 & 313 & 4.49 & 6.21 &  -7.93 & 321 & 3.94 & \second{2.99} & \second{-2.04} & 232 & 3.12 & \second{2.48} & \second{-1.85} & 257 \\
& DARC-$\epsilon$(ours)  & 4.46 & \best{5.29} & \best{-6.12} & 314 & 4.38 & \best{5.60} &  \best{-6.82} & 319 & 3.96 & \best{2.96} & \best{-1.96} & 243 & 3.18 & \best{2.36} & \best{-1.54} & 261 \\

\cmidrule(lr){1-18} \addlinespace[2pt]

\multirow{4}{*}{\textit{Training}} 
& cDPO (Best-of-$K$)        & 5.88 & 5.96 & -6.04 & 300 & 5.86 & 6.97 & -8.08 & 309 & 8.61 & 5.76 & -2.91 & 258  & 5.30 & 6.63 & -7.96 & 278 \\
& cDPO + DARC-$\epsilon$     & 5.81 & 5.65 & -5.49 & 311 & 5.83 & 6.57 & -7.31 & 290  & 8.53 & 5.49 & -2.45 & 278  & 5.22 & 6.40 & -7.58 & 274 \\
& rDPO (Best-of-$K$)        & 5.53 & 5.57 & -5.61 & 260 & 11.62 & 7.10 & -2.58 & 316 & 8.73 & 5.94 & -3.15 & 294  & 5.77 & 6.90 & -8.03 & 282 \\
& rDPO + DARC-$\epsilon$    & 5.42 & 5.16 & -4.90 & 258 & 11.58 & 6.78 & -1.98 & 325 & 8.62 & 5.69 & -2.76 & 288  & 5.68 & 6.68 & -7.68 & 232 \\

\specialrule{0.1pt}{2pt}{0pt}
\specialrule{0.1pt}{1.5pt}{0pt} 
\multicolumn{18}{l}{\textit{\textbf{Dataset: High-Variance (Top 20\%)}}} \\ \addlinespace[2pt]

\multirow{5}{*}{\textit{Inference}} 
& Base (Best-of-$K$)       & 5.79 & 9.91 & -14.03 & 382 & 7.48 & 10.22 & -12.96 & 381 & 5.12 & 6.00 &  -7.12 & 308 &  2.79 & 5.46 & -8.13 & 263 \\
& BoP(HedgeTune)       & 5.45 & 7.90 & -10.35 & 382 & 7.28 & \second{8.12} & \second{-8.96} & 381 & 4.67 & 5.12 &  -5.57 & 307 &  2.53 & 4.09 & -5.65 & 265 \\
& Caution         & 5.41 & 7.42 & -9.43  & 382 & 7.17 &  8.34 & -9.51  & 381 & 4.80 & 5.39 &  -5.98 & 307 &  2.30 & 4.07 & -5.84 & 264 \\
& DeAL                     & 5.96 & 9.71 & -13.46  & 374 & 7.89 &  10.15 & -12.41 & 371 & 5.75 & 8.43 &  -11.11 & 316 &  3.45 & 6.89 & -10.33 & 256 \\
& MC-Dropout               & 5.20 & 7.48 & -9.76  & 384 & 7.19 &  8.15 & -9.11  & 380 & 4.75 & 5.25 &  -5.25 & 306 &  2.35 & 4.00 & -5.65 & 262 \\
& RBoN                     & 5.28 & \second{7.30} & \second{-9.32}  & 372 & 7.22 &  8.40 & -9.58  & 376 & 4.90 & 5.45 &  -6.00 & 300 &  2.25 & 4.15 & -6.05 & 258 \\
& DARC(ours)           & 5.32 & 7.33 & -9.34  & 382 & 7.24 &  \best{8.10} & \second{-8.96}  & 381 & 4.99 & 5.05 &  -5.13 & 307 &  2.34 & 4.42 & -6.50 & 263 \\
& DARC-$\tau$(ours)    & 5.43 & 7.90 & -10.37 & 382 & 7.21 &  8.55 & -9.89  & 381 & 4.78 & \second{4.85} &  \second{-4.92} & 294 &  2.31 & \second{3.84} & \second{-5.37} & 249 \\
& DARC-$\epsilon$(ours)  & 5.49 & \best{7.00} & \best{-8.51} & 382 & 7.26 &  \best{8.10} & \best{-8.94}  & 381 & 5.03 & \best{4.77} &  \best{-4.51} & 306 &  2.59 & \best{3.76} & \best{-4.93} & 265 \\

\cmidrule(lr){1-18} \addlinespace[2pt]

\multirow{4}{*}{\textit{Training}} 
& cDPO (Best-of-$K$)        & 7.54 & 8.39 & -9.24 & 312 & 8.88 & 9.40 & -9.92 & 328 & 5.52 & 5.67 &  -5.82 & 317 & 7.17 & 7.30 & -7.43 & 196 \\
& cDPO + DARC-$\epsilon$     & 7.51 & 8.11 & -8.71 & 323 & 8.81 & 9.21 & -9.61 & 328 & 5.42 & 5.17 &  -4.92 & 316 & 7.12 & 7.29 & -7.46 & 196 \\
& rDPO (Best-of-$K$)        & 6.58 & 7.98 & -9.38 & 337  & 11.60 & 9.05 & -6.50 & 347 & 5.79 & 5.34 &  -4.89 & 304  & 7.80 & 7.72 & -7.64 & 172 \\
& rDPO + DARC-$\epsilon$     & 6.41 & 7.16 & -7.91 & 316 & 11.56 & 8.65 & -5.74 & 329 & 5.56 & 5.03 &  -4.50 & 301  & 7.70 & 7.28 & -6.86 & 164 \\

\bottomrule
\end{tabular}
}
\caption{\textbf{Evaluation on MT-Bench and AlpacaEval 2.0.} Results across two generators families . We report mean reward , proxy risk (Risk), risk--reward tradeoff score (Tradeoff), and average length (Len(tok)).Blue/green highlight best/runner-up.}
\label{tab:main_results_llama_qwen}
\vspace{-20pt}
\end{table*}

\subsection{Practical decoding: $\epsilon$-tie breaking}
\label{sec:eps}
\begingroup
\setlength{\parskip}{0pt}
\setlength{\abovedisplayskip}{3pt}
\setlength{\belowdisplayskip}{3pt}
\setlength{\abovedisplayshortskip}{2pt}
\setlength{\belowdisplayshortskip}{2pt}
\noindent\textbf{Proxy robustness.}
When $(\hat\mu,\hat\sigma)$ are obtained from scalable proxy scorers rather than i.i.d.\ human samples, a uniform proxy-error assumption yields a robust LCB guarantee; we defer the formal statement and proof to Appendix~\ref{app:proxy-robust}.

\paragraph{$\epsilon$-rule as a constrained optimization.}
Our $\epsilon$-tie-breaking is equivalent to
\begin{equation}
y_{\epsilon}(s)\in \arg\min_{y\in \mathcal{Y}(s)} \widehat\sigma(s,y)
\quad \text{s.t.}\quad
\widehat V_{\beta}(s,y)\ge \widehat V_{\max}(s)-\epsilon,
\label{eq:eps-constrained-entropic}
\end{equation}
where $\widehat V_{\max}(s):=\max_{y\in\cY(s)}\widehat V_{\beta}(s,y)$.

\paragraph{Pareto view.}
Eq.~\eqref{eq:eps-constrained-entropic} selects a Pareto-optimal point of the robust-value--disagreement trade-off: among candidates whose robust value is within $\epsilon$ of the best, it returns the least controversial one.
Moreover, when the chosen point is \emph{supported}, the same solution can be written as a scalarization
$\arg\max_y \big(\widehat V_{\beta,\gamma}(s,y)-\lambda_\epsilon \widehat\sigma(s,y)\big)$ for some $\lambda_\epsilon\ge 0$;
formal statements and proofs are deferred to Appendix~\ref{app:eps-pareto}.
\endgroup
\paragraph{Two-stage rule.}
To avoid sacrificing entropic value excessively for marginal robustness, we use:
\begin{equation}
\begin{aligned}
\mathcal{F}_\epsilon(s) &:= \{y\in\cY(s) \mid \widehat V_{\beta}(s,y)\ge \widehat V_{\max}(s)-\epsilon\}, \\
y_{\epsilon}(s) &\in \arg\min_{y\in\mathcal{F}_\epsilon(s)} \widehat\sigma(s,y).
\end{aligned}
\end{equation}
This approximates the scalarized decoder $\max_y \big(\widehat V_{\beta}(s,y)-\lambda\,\widehat\sigma(s,y)\big)$ by enforcing near-optimality in entropic value and then selecting the least controversial candidate.
In practice, we tune all risk parameters ($\beta, \tau, \epsilon$) on a held-out development set; detailed tuning protocols are provided in Appendix~\ref{app:knob_selection}.

\subsection{Multi-scorer robustness}
\paragraph{Scorer-robust decoding (multi-reward models).}
In scalable proxy settings, satisfaction scores may come from learned reward models and be sensitive to model shift or proxy over-optimization.
We optionally hedge this \emph{scorer ambiguity} by using a family of $M$ scorers indexed by $m\in[M]$.
For each candidate $y\in\cY(s)$, scorer $m$ provides $n$ scalar samples $\{R_{m,i}(s,y)\}_{i=1}^n$ (e.g., via perturbations), from which we compute the scorer-specific entropic value $\widehat V_{\beta,m}(s,y)$ and risk premium $\widehat{\mathrm{RP}}_{\beta,m}(s,y)$.
Because reward scales can differ across scorers, we apply a per-prompt affine normalization (Appendix~\ref{app:multi_scorer}).
We then aggregate scorer-specific entropic values using a soft worst-case operator:
\begin{equation}
\widetilde V_{\beta,\gamma}(s,y)
:=
-\frac{1}{\gamma}\log\!\left(\frac{1}{M}\sum_{m=1}^M \exp\!\big(-\gamma\,\widehat V_{\beta,m}(s,y)\big)\right),
\label{eq:scorer-robust-V}
\end{equation}

and aggregate risk premia pessimistically by $\widetilde{\mathrm{RP}}_{\beta}(s,y):=\max_{m\in[M]}\widehat{\mathrm{RP}}_{\beta,m}(s,y)$.
We decode by maximizing $\widetilde V_{\beta,\gamma}(s,y)$ subject to $\widetilde{\mathrm{RP}}_\beta(s,y)\le\tau$ (or its penalized form).

\begin{proposition}[Scorer aggregation as KL-regularized DRO over scorers]
\label{prop:scorer-dro}
Let $v_m := \widehat V_{\beta,m}(s,y)$ and $\mathbf{v}\in\mathbb{R}^M$.
Define $\mathbf{u}=(1/M,\ldots,1/M)$ and the simplex $\Delta_M:=\{\mathbf{q}\in\mathbb{R}^M_+:\sum_{m=1}^M q_m=1\}$.
Then the soft worst-case aggregation \eqref{eq:scorer-robust-V} admits the variational form
\begingroup
\setlength{\abovedisplayskip}{3pt}
\setlength{\belowdisplayskip}{3pt}
\setlength{\abovedisplayshortskip}{1pt}
\setlength{\belowdisplayshortskip}{1pt}
\begin{equation}
\widetilde V_{\beta,\gamma}(s,y)
=
\inf_{\mathbf{q}\in\Delta_M}
\left\{
\sum_{m=1}^M q_m v_m
+
\frac{1}{\gamma} D_{\mathrm{KL}}(\mathbf{q}\|\mathbf{u})
\right\}.
\label{eq:scorer-dro}
\end{equation}
\endgroup
\end{proposition}
\paragraph{Interpretation.}
Eq.~\eqref{eq:scorer-dro} shows that $\gamma$ interpolates between averaging across scorers ($\gamma\!\to\!0$) and a worst-case scorer ($\gamma\!\to\!\infty$), yielding a principled hedge against scorer shift. In particular, $\min_m v_m \le \widetilde V_{\beta,\gamma}(s,y) \le \frac{1}{M}\sum_m v_m$, and $\widetilde V_{\beta,\gamma}(s,y)$ is non-increasing in $\gamma$.

\section{Experiment}
\label{sec:experiments}
\subsection{Experimental Setup and Implementation Details}
\label{sec:exp_setup}
\paragraph{Candidate sets and evaluation protocol.}
For each prompt $s$, we generate a fixed candidate pool $\cY(s)$ of size $K$ (shared across all methods),
so differences arise solely from the inference-time \emph{selection rule}.
In the human-grounded setting, for each $y\in\cY(s)$ we collect $n$ scalar satisfaction ratings
$\{r_i(s,y)\}_{i=1}^n$~\citep{christiano2017deep,stiennon2020learning,ouyang2022training},
and split annotators into disjoint selection and evaluation sets
$\mathcal{I}_{\text{sel}}$ and $\mathcal{I}_{\text{eval}}$ to avoid selection--evaluation leakage.
All decoders operate on statistics computed from $\mathcal{I}_{\text{sel}}$.
Our entropic decoder uses the empirical entropic value in the human-grounded setting (where $M=1$ and $\widetilde V_{\beta,\gamma}=\widehat V_\beta$),
and uses the scorer-robust aggregated value $\widetilde V_{\beta,\gamma}(s,y)$ (and $\widetilde{\mathrm{RP}}_{\beta}(s,y)$) in proxy settings with multiple reward models.
Second-moment baselines use only $(\hat\mu_{\text{sel}}(s,y),\hat\sigma_{\text{sel}}(s,y))$ (or their scorer-indexed analogues under proxies).
All reported metrics are computed on held-out ratings in $\mathcal{I}_{\text{eval}}$ (definitions in Appendix~\ref{app:metrics}).

\paragraph{Estimating mean and disagreement.}
We define $(\hat\mu,\hat\sigma)$ as the mean and disagreement estimates, derived from either empirical samples or scorer proxies .
DARC is \emph{inference-time only}, modifying the selection over $\cY(s)$ without retraining.
While our theory assumes multi-annotator feedback, we validate proxy uncertainty as a scalable approximation for human disagreement.
\paragraph{Proxies, baselines, and hyperparameters.}
When multi-rater human scores are unavailable, we use scalable proxy scorers (e.g., ensembles / bootstrap reward models)
to obtain a proxy mean together with either (i) a score distribution for computing $\widehat V_\beta$, or
(ii) a proxy disagreement diagnostic $\hat\sigma_{\text{proxy}}$~\citep{kendall2017uncertainties,lakshminarayanan2017simple};
we validate proxy--human alignment via rank correlation and stratified analyses~\citep{uma2021learning}.
Inference-time baselines include mean Best-of-$K$,
inference-time hedging methods for reward hacking (Best-of-Poisson and HedgeTune)~\citep{khalaf2025inference},
pessimistic best-of-$N$ reranking via an auxiliary error/atypicality model (\emph{Caution})~\citep{mitigatingiclr26},
uncertainty-based reranking (e.g., MC-Dropout)~\citep{gal2016dropout,kendall2017uncertainties},
reward-regularized decoding (RBoN)~\citep{jinnai2024regularized} and \emph{DeAL} (top-$k$ lookahead reward-guided decoding)~\citep{huang2025deal};
training-time baselines include cDPO/rDPO~\citep{mitchell2023note,chowdhury2024provably}.
We report default settings and full tuning ranges (including our LCB- and DRO-instantiated decoding rules),
along with complete baseline configurations and implementation details,
in Appendix~\ref{app:hparams} (see also Appendix~\ref{app:decoding_rules}).

\paragraph{Perturbation families.}
Our default proxy uses style-preserving perturbations, but such rewrites may
under-detect surface-form reward-model artifacts when the rewarded cue itself is
preserved. We therefore additionally evaluate targeted perturbations that
normalize verbosity, formatting, and sycophantic/apologetic phrasing, as well as
a hybrid proxy combining style-preserving and targeted rewrites. Appendix~\ref{app:perturbation_families}
reports both proxy--human alignment and downstream DARC-$\epsilon$ robustness
under these perturbation families.

\paragraph{Inference overhead.}
Disagreement estimation adds only a modest inference-time cost: in our implementation, using $N_{\text{aug}}=4/8/16$ increases end-to-end latency by only $\approx 1.5\%/2.0\%/3.2\%$, since candidate generation dominates runtime (Appendix~\ref{app:latency}, Table~\ref{tab:latency}).

\paragraph{Metrics (mean, risk, and tail).}
For each prompt $s$ and candidate response $y$, we evaluate rewards on a held-out rater set $\mathcal{I}_{\text{eval}}$ and report the empirical mean $\hat\mu_{\text{eval}}(s,y)$.
As a scalable proxy for robustness, we compute a perturbation-sensitivity
statistic $\hat{\sigma}_{\mathrm{sel}}(s,y)$ as the standard deviation of
rewards across $N_{\mathrm{aug}}$ perturbations of $y$ from the specified
perturbation family. The main experiments use the style-preserving default,
while Appendix~\ref{app:proxy-robust}evaluates targeted and hybrid perturbation families.
Human evaluation details are in Appendix~\ref{app:human-eval}; the human mean is the average rating and the human variance is the sample variance across raters.
We summarize performance via $\mathrm{Tradeoff}_{\text{eval}}(s,y):=\hat\mu_{\text{eval}}(s,y)-\lambda\,\hat\sigma_{\text{sel}}(s,y)$, using the same $\lambda$ for all methods (Appendix~\ref{app:tradeoff-weight}).
We also report $\mathrm{CVaR}_{10\%}$ over prompts~\citep{chow2015risk}; see Appendix~\ref{app:metrics}. The Tradeoff score is partially aligned with the selection objective, so we use
it as a supporting diagnostic rather than as the primary evidence. Our main
external validation relies on held-out human satisfaction and prompt-level
CVaR.

\begin{table}
\centering
\small
\resizebox{\columnwidth}{!}{
\begin{tabular}{l cccc}
\toprule
\textbf{Method} & \textbf{Human Score} $\uparrow$ & \textbf{Risk} (Avg$\sigma$)$\downarrow$ & \textbf{Tradeoff}$\uparrow$ & \textbf{CVaR$_{10\%}$}$\uparrow$ \\
\midrule
\multicolumn{5}{l}{\textit{Dataset: Overall}} \\
Base (Best-of-$K$)   & 7.56 & 0.67 & 6.22 & 6.73 \\
BoP(HedgeTune)       & 7.79 & 0.63 & 6.53 & 7.12 \\
Caution     & 7.82 & 0.61 & 6.60 & 7.15 \\
DeAL                 & 7.83 & 0.72 & 6.39 & 7.08 \\
MC-Dropout           & 7.73 & 0.64 & 6.45 & 7.19 \\
RBoN                 & 7.71 & 0.62 & 6.47 & 7.26 \\
DARC(ours)        & \second{7.84} & \second{0.60} & 6.64 & 7.36 \\
DARC-$\tau$(ours)     & 7.83 & 0.58 & \second{6.67} & \second{7.46} \\
DARC-$\epsilon$(ours) &  \best{8.08} &  \best{0.55} &  \best{6.98} &  \best{7.62} \\
\addlinespace[0.25em]
\cmidrule(lr){1-5}
\addlinespace[0.25em]

cDPO (Best-of-$K$)   & 7.80 & 0.65 & 6.50 & 7.11 \\
cDPO + DARC-$\epsilon$ & 8.03 & 0.54 & 6.95 & 7.41 \\
rDPO (Best-of-$K$)   & 8.17 & 0.69 & 6.79 & 7.20 \\
rDPO + DARC-$\epsilon$ & 8.15 & 0.58 & 6.99 & 7.60 \\

\midrule
\multicolumn{5}{l}{\textit{Dataset: High-Disagreement (Top 20\%)}} \\
Base (Best-of-$K$)   & 7.62 & 1.00 & 5.62 & 6.10 \\
BoP(HedgeTune)       & 7.83 & 0.87 & 6.09 & 6.81 \\
Caution              & 7.81 & 0.85 & 6.11 & 6.94 \\
DeAL                 & 7.99 & 1.13 & 5.73 & 7.33 \\
MC-Dropout           & 7.88 & 0.84 & 6.20 & 7.25 \\
RBoN                  & 7.95 & 0.81 & 6.33 & 7.29 \\
DARC(ours)        & \second{8.18} & 0.74 & 7.00 & \second{7.43} \\
DARC-$\tau$(ours)     & 8.11 &\second{ 0.65} & \second{6.81} & 7.35 \\
DARC-$\epsilon$(ours) & \best{8.34} & \best{0.65} & \best{7.34} & \best{7.60} \\
\addlinespace[0.25em]
\cmidrule(lr){1-5}
\addlinespace[0.25em]

cDPO (Best-of-$K$)   & 7.96 & 0.89 & 6.18 & 7.15 \\
cDPO + DARC-$\epsilon$ & 8.31 & 0.63 & 7.05 & 7.54 \\
rDPO (Best-of-$K$)   & 8.25 & 0.91 & 6.43 & 7.02 \\
rDPO + DARC-$\epsilon$ & 8.72 & 0.67 & 7.38 & 7.48 \\
\bottomrule
\end{tabular}
}
\caption{\textbf{Human Evaluation Results.}
We report Human Mean Score, Disagreement Risk, Tradeoff, and CVaR$_{10\%}$ (worst $10\%$ prompt outcomes).
}
\label{tab:humanloop-concise}
\vspace{-1.5em}
\end{table}

\subsection{Results}

\paragraph{Automated proxy evaluation.}Table~\ref{tab:main_results_llama_qwen} reports MT-Bench~\citep{zheng2023judging} and AlpacaEval2.0~\citep{li2023alpacaeval,dubois2024length} results under two instruction-tuned generators (Llama-3.1-8B-Instruct~\citep{dubey2024llama}, Qwen2.5-7B-Instruct~\citep{qwen2025qwen25technicalreport}).
We evaluate each method by mean reward $\hat\mu$, proxy disagreement risk $\hat\sigma$ (reward sensitivity to style-preserving perturbations), and a risk--reward tradeoff score $\mathrm{Tradeoff}=\hat\mu-\lambda\hat\sigma$ with the same $\lambda$ across methods; we also report \textbf{Len(tok)} to control for length bias.
Across both generators, inference-time risk-aware selection improves robustness: DARC variants reduce $\hat\sigma$ while keeping $\hat\mu$ competitive, yielding higher $\mathrm{Tradeoff}$ than mean-only Best-of-$K$, with larger gains on the high-variance subset (top 20\% prompts by baseline $\hat\sigma$) and broadly stable \textbf{Len(tok)}.
Training-time robust policies (cDPO/rDPO) exhibit different trade-offs, and DARC remains complementary: applying DARC-$\epsilon$ as an inference-time plug-in can further reduce $\hat\sigma$ and improve $\mathrm{Tradeoff}$ in several cases, supporting a modular view where training shapes the policy while DARC calibrates risk on a fixed candidate set.

\begin{figure}
    \includegraphics[width=1.0\linewidth]{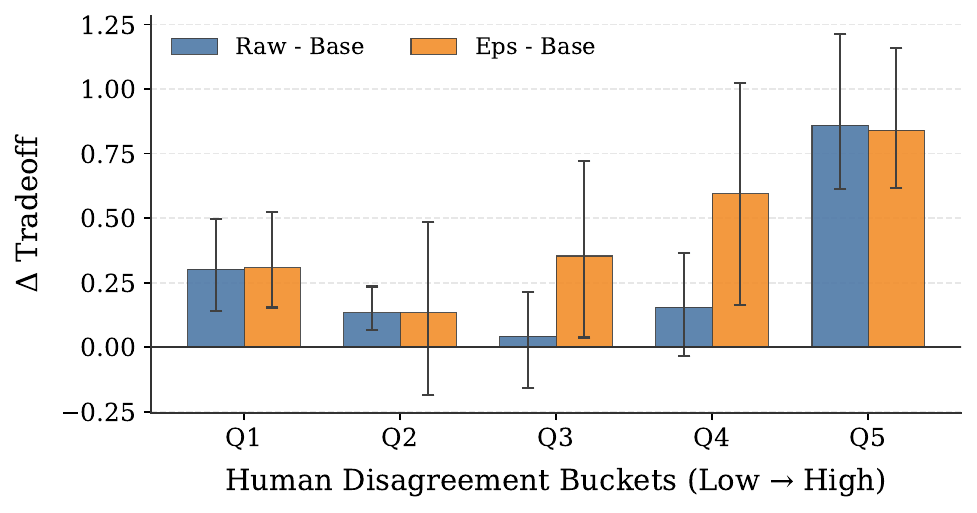}
    \vspace{-18pt}
    \caption{\textbf{Gains concentrate on high-disagreement prompts.}
    Mean improvement in lower-tail satisfaction ($\Delta$Tradeoff vs.\ base) across five prompt buckets ranked by baseline human disagreement $\hat\sigma$ (low$\rightarrow$high). Error bars denote 95\% CIs.}
    \label{fig:bucket_dtradeoff}
\end{figure}

\paragraph{Human-loop evaluation and tail robustness.}
We then close the loop with multi-annotator ratings on MT-Bench. 
Table~\ref{tab:humanloop-concise} shows that DARC improves risk-sensitive criteria (Tradeoff and CVaR$_{10\%}$) while maintaining competitive mean satisfaction.
On the high-disagreement subset, disagreement-aware decoding yields substantial gains in tail robustness (CVaR$_{10\%}$) and conservative quality (Tradeoff), indicating that risk-controlled selection is particularly beneficial on genuinely controversial prompts.
Figure~\ref{fig:bucket_dtradeoff} demonstrates that Risk-adjusted satisfaction($\Delta$Tradeoff) scale positively with disagreement, rising steadily from Q1 to Q5, peaking in the most controversial buckets (Q5). This confirms disagreement as an effective signal for allocating risk control, a trend corroborated by the CVaR results in Appendix Fig.~\ref{fig:bucket_dcvar}.(See representative cases in Appendix~\ref{app:positive_case}).)

\begin{table}[t]
\centering
\small
\caption{Bucketed human evaluation versus Base by baseline human disagreement.
DARC-$\epsilon$ remains close to mean-based selection in low-disagreement
prompts and becomes increasingly beneficial as disagreement grows.}
\label{tab:bucket_human}
\begin{tabular}{lcc}
\toprule
Bucket & Human Score $\Delta$ (DARC-$\epsilon$ -- Base) & W / T / L \\
\midrule
Q1 & 0.125 & 11 / 80 / 9 \\
Q2 & 0.236 & 17 / 71 / 12 \\
Q3 & 0.451 & 34 / 55 / 11 \\
Q4 & 0.755 & 42 / 50 / 8 \\
Q5 & 1.037 & 51 / 45 / 4 \\
\bottomrule
\end{tabular}
\end{table}
The low-disagreement bucket clarifies the apparent near-zero
$\Delta$Tradeoff in Figure~\ref{fig:bucket_dtradeoff}. As shown in
Table~\ref{tab:bucket_human}, Q1 does not exhibit broad degradation: the mean
human-score delta remains slightly positive and outcomes are dominated by ties.
The gains increase monotonically from Q2 to Q5, suggesting that DARC provides
selective robustness where disagreement is substantive while staying close to
mean-based selection when disagreement is weak.
\vspace{-0.5em}
\paragraph{Validity of the disagreement proxy.}
We validate our perturbation-sensitivity proxy $\hat{\sigma}$ against human disagreement measured by multiple independent rater scores on the same (prompt, response) pair.
To avoid selection effects, we compute both proxy and human disagreement on the \emph{baseline} candidate.
As summarized in Figure~\ref{fig:proxy_validity_combined}, proxy disagreement shows a statistically significant rank-level consistency with human disagreement and remains positively associated even after controlling for mean reward and response length (strong and common confounders for RM-based signals).
Moreover, the proxy serves as an effective \emph{risk filter} for identifying prompts likely to exhibit \emph{preference heterogeneity}, where risk-controlled decoding is designed to intervene: prompts identified as high-disagreement by the proxy substantially overlap with those identified by humans, far beyond random chance (Fig.~\ref{fig:proxy_validity_combined}).
While mismatch cases exist (App.~\ref{app:proxy_mismatch})---often reflecting \emph{surface-form sensitivity} (FP) or \emph{orthogonal validity/completeness issues} (FN)---this is consistent with using $\hat{\sigma}$ as a scalable \emph{screening} signal to prioritize human verification rather than a fully calibrated estimator of disagreement magnitude.
Consistently, this signal effectively allocates risk control: bucketed analysis shows a monotonic increase in $\Delta\mathrm{Tradeoff}$ from low- to high-disagreement groups when bucketing by either human or proxy disagreement (App.~\ref{app:proxy_validity}).
We also stratify prompts by proxy--human disagreement alignment and find that, even in the worst-alignment bucket, DARC still achieves the best Tradeoff among all baselines (Appendix~\ref{app:proxy_reliability_bucket}, Table~\ref{tab:error_bucket_tradeoff_multi}).

{\setlength{\intextsep}{10pt}
\begin{figure}[H]
  \centering
  \small
  \resizebox{\columnwidth}{!}{
  \begin{tabular}{lcc}
    \toprule
    \textbf{Metric} & \textbf{Value} & \textbf{Notes / $p$-value} \\
    \midrule
    Spearman $\rho$  & 0.6509 & 95\% boot.\ CI [0.42, 0.70]; $p < 10^{-10}$ \\
    Kendall $\tau$   & 0.3130 & $p < 10^{-6}$ \\
    Partial Spearman $\rho$ \scriptsize{(control $\hat{\mu}$, Len)} & 0.4084 & 95\% boot.\ CI [0.25, 0.45] \\
    \midrule
    Precision@20\% \scriptsize{(Top-$k$ overlap)} & 0.64 & (64/100;\ random=0.20) \\
    Recall@20\%                           & 0.64 & (symmetric at equal $k$) \\
    Jaccard@20\%                          & 0.47 & ($|A\cap B|/|A\cup B|$) \\
    Overlap significance              & 64 & hypergeom $p = 2.6\times 10^{-29}$ \\
    \bottomrule
  \end{tabular}
  }

  \includegraphics[width=0.49\columnwidth]{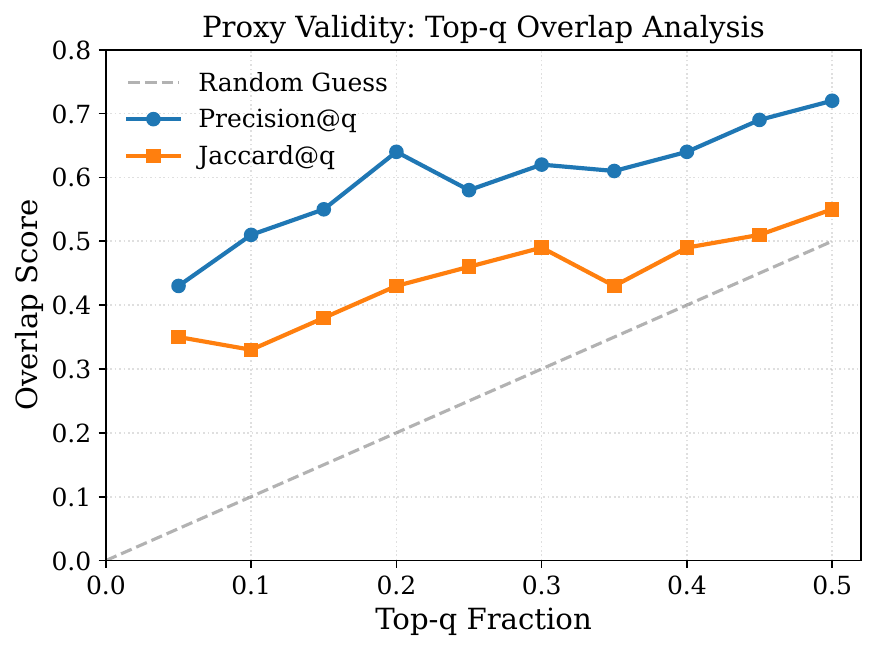}\hfill
  \includegraphics[width=0.49\columnwidth]{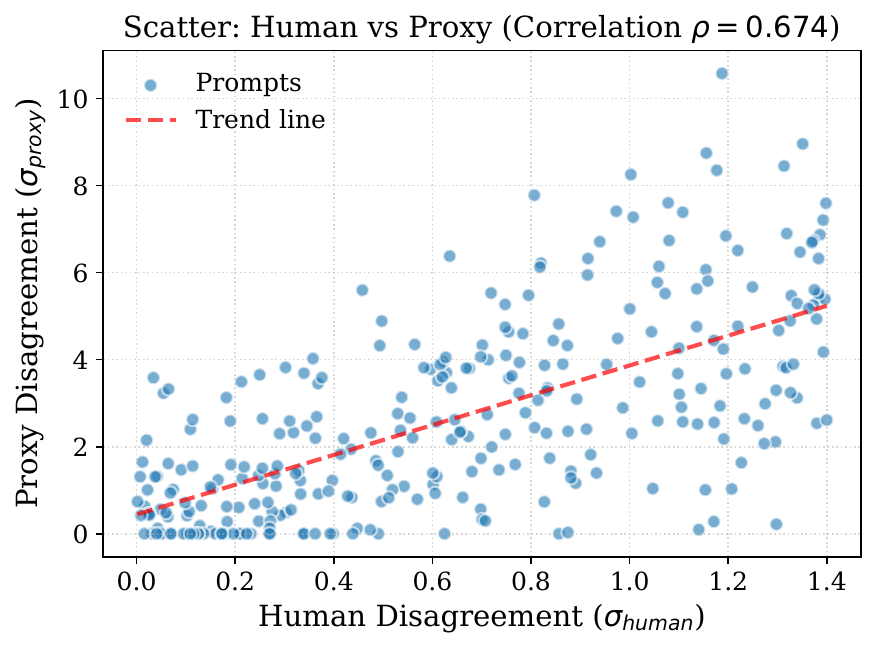}

    \caption{\textbf{Proxy validity diagnostics.}
    \textbf{(Top)} Rank correlation between proxy and human disagreement, with top-20\% overlap.
    \textbf{(Bottom)} Top-$q$ overlap (Left) and proxy vs.\ human disagreement scatter (Right).}
  \label{fig:proxy_validity_combined}
\end{figure}
}

\begin{figure*}
    \centering
    \includegraphics[width=\textwidth]{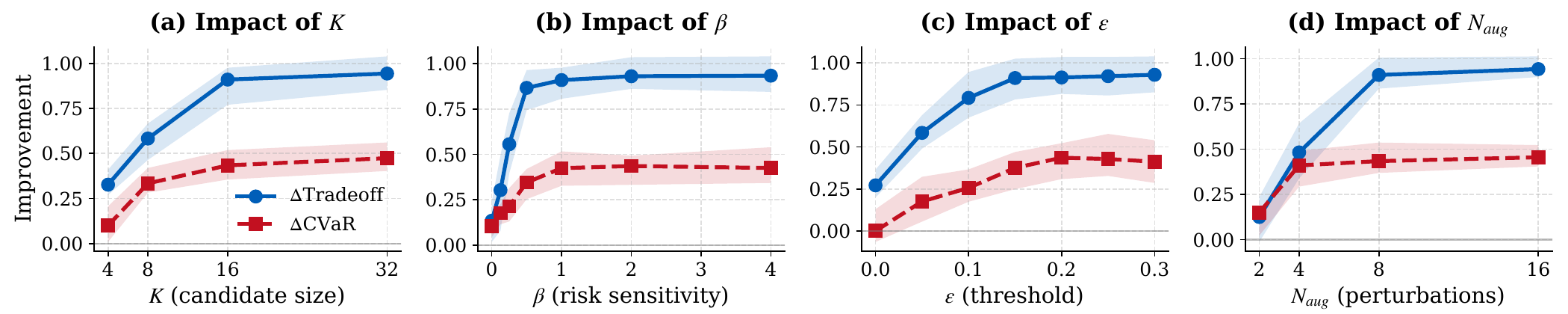}
    \caption{\textbf{Ablation Studies.} Impact of key hyperparameters on risk mitigation performance. 
    \textbf{(a)} Candidate pool size $K$.
    \textbf{(b)} Risk sensitivity coefficient $\beta$.
    \textbf{(c)} Constraint threshold $\epsilon$.
    \textbf{(d)} Perturbation budget $N_{\text{aug}}$.}
    \label{fig:ablations}
    \vspace{-18pt} 
\end{figure*}
\paragraph{Multi-scorer robust decoding}
To mitigate proxy over-optimization to a single reward model, we instantiate our decoder with a scorer family of size $M=3$:
RM1 (Skywork-reward-llama-3.1-8b~\citep{liu2024skywork}), RM2 (nicholasKluge RewardModel~\citep{nicholas22aira}) and RM3(OpenAssistant (DeBERTa-v3-Large-v2)~\citep{kopf2023openassistant}).
At selection time, we compute $\widehat V_{\beta,m}(s,y)$ for each $m\in[M]$ (with within-prompt normalization),
aggregate them into $\widetilde V_{\beta,\gamma}(s,y)$ via~\eqref{eq:scorer-robust-V},
and apply the same constrained/penalized rule using $\widetilde{\mathrm{RP}}_\beta$ in~\eqref{eq:scorer-robust-RP-app}.
For reporting, because absolute reward scales differ across RMs, we evaluate the selected outputs within each RM
and report only within-RM differences to the baseline (mean $\Delta$) and Win/Tie/Loss rates (Table~\ref{tab:multi_scorer_main}).We additionally provide a per-scorer breakdown of win/tie/loss and mean score differences under each evaluator reward model in Appendix~\ref{app:per_scorer_breakdown}, confirming consistent gains across scorers.

\begin{table}
\centering
\small
\resizebox{\columnwidth}{!}{
\begin{tabular}{l cc cc}
\toprule
& \multicolumn{2}{c}{\textbf{Overall}} & \multicolumn{2}{c}{\textbf{High-$\hat\sigma$ Subset}} \\
\cmidrule(lr){2-3} \cmidrule(lr){4-5}
\textbf{Method} & \textbf{W/T/L} & \textbf{Mean $\Delta$} & \textbf{W/T/L} & \textbf{Mean $\Delta$} \\
\midrule
\multicolumn{5}{l}{\textbf{Across scorers (aggregated)}} \\
\midrule 
DARC & 278 / 173 / 49 & 0.124 & 32 / 63 / 5 & 0.303 \\ 
DARC-$\tau$ & 257 / 192 / 51 & 0.128 & 36 / 60 / 4 & 0.259 \\
DARC-$\epsilon$ & 296 / 168 / 36 & 0.161 & 44 / 54 / 2 & 0.386 \\
\bottomrule
\end{tabular}
}
\caption{\textbf{Multi-scorer robust decoding.}
DARC variants vs.\ Base (mean Best-of-$K$).
$\Delta$: per-prompt improvement over Base, aggregated across scorers (mean / min / mean$-\gamma$std); W/T/L is on aggregated $\Delta$.
High-$\hat\sigma$ uses the baseline disagreement proxy.}
\vspace{-15pt}
\label{tab:multi_scorer_main}
\end{table}

\paragraph{Scaling robustness.}
To assess scale robustness, we additionally evaluate our inference-time decoding rules on a stronger generator, \textbf{Qwen2.5-14B-Instruct}, using the same candidate-pool protocol and metrics.
We observe consistent improvements over mean Best-of-$K$ in risk-sensitive criteria (Tradeoff and prompt-level CVaR), with gains again concentrating on high-disagreement prompts (Table~\ref{tab:qwen14b_appendix}).

\paragraph{Hyperparameter sensitivity and ablations.}
Figure~\ref{fig:ablations} summarizes a sensitivity study over the main knobs of DARC-$\epsilon$, reporting improvements over mean Best-of-$K$ in $\Delta\mathrm{Tradeoff}$ and $\Delta\mathrm{CVaR}$ on fixed candidate pools.
Increasing the candidate pool size $K$ improves both metrics with diminishing returns (Fig.~\ref{fig:ablations}a).
The risk sensitivity $\beta$ in the entropic robust value $\widehat V_\beta$ yields rapid robustness gains that plateau at moderate values (Fig.~\ref{fig:ablations}b).
Fixing $\beta$, enlarging the near-optimal set via $\epsilon$ further improves robustness up to saturation (with slight degradation when overly permissive) (Fig.~\ref{fig:ablations}c).
Finally, the perturbation budget $N_{\text{aug}}$ used to estimate proxy disagreement $\hat\sigma$ saturates quickly; a small budget (e.g., $4$--$8$) suffices, keeping inference overhead modest (Fig.~\ref{fig:ablations}d).

\section{Conclusion}

We cast decoding-time alignment under heterogeneous preferences as risk-constrained decision making.
DARC uses a KL-robust (entropic) decoding rule with a conservative LCB/DRO interpretation, improving lower-tail outcomes while preserving mean quality.
Limitations include a finite candidate pool and scorer bias; perturbation-based disagreement proxies are scalable screening signals, not calibrated human-disagreement estimates.
Future work includes richer robustness signals and user/group-conditional risk control.

\section*{Impact Statement}
This paper presents work whose goal is to advance the field of machine learning.
Our contributions focus on inference-time response selection under preference heterogeneity.
As with many machine learning methods, the proposed approach may exhibit uneven performance if the underlying proxies or learned scorers encode dataset- or population-specific biases.
We encourage reporting sensitivity to key design choices and auditing proxy/scorer behavior when the method is used in practice.
Consequently, the method should be used and evaluated with appropriate care in deployment.

\bibliography{icml26paper}

@article{ouyang2022training,
  title={Training language models to follow instructions with human feedback},
  author={Ouyang, Long and Wu, Jeffrey and Jiang, Xu and Almeida, Diogo and Wainwright, Carroll and Mishkin, Pamela and Zhang, Chong and Agarwal, Sandhini and Slama, Katarina and Ray, Alex and others},
  journal={Advances in neural information processing systems},
  volume={35},
  pages={27730--27744},
  year={2022}
}

@article{rafailov2023direct,
  title={Direct preference optimization: Your language model is secretly a reward model},
  author={Rafailov, Rafael and Sharma, Archit and Mitchell, Eric and Manning, Christopher D and Ermon, Stefano and Finn, Chelsea},
  journal={Advances in neural information processing systems},
  volume={36},
  pages={53728--53741},
  year={2023}
}

@article{schulman2017proximal,
  title={Proximal policy optimization algorithms},
  author={Schulman, John and Wolski, Filip and Dhariwal, Prafulla and Radford, Alec and Klimov, Oleg},
  journal={arXiv preprint arXiv:1707.06347},
  year={2017}
}

@article{stiennon2020learning,
  title={Learning to summarize with human feedback},
  author={Stiennon, Nisan and Ouyang, Long and Wu, Jeffrey and Ziegler, Daniel and Lowe, Ryan and Voss, Chelsea and Radford, Alec and Amodei, Dario and Christiano, Paul F},
  journal={Advances in neural information processing systems},
  volume={33},
  pages={3008--3021},
  year={2020}
}

@article{hong2024orpo,
  title={Orpo: Monolithic preference optimization without reference model},
  author={Hong, Jiwoo and Lee, Noah and Thorne, James},
  journal={arXiv preprint arXiv:2403.07691},
  year={2024}
}

@article{meng2024simpo,
  title={Simpo: Simple preference optimization with a reference-free reward},
  author={Meng, Yu and Xia, Mengzhou and Chen, Danqi},
  journal={Advances in Neural Information Processing Systems},
  volume={37},
  pages={124198--124235},
  year={2024}
}

@article{ethayarajh2024kto,
  title={Kto: Model alignment as prospect theoretic optimization},
  author={Ethayarajh, Kawin and Xu, Winnie and Muennighoff, Niklas and Jurafsky, Dan and Kiela, Douwe},
  journal={arXiv preprint arXiv:2402.01306},
  year={2024}
}

@article{bradley1952rank,
  title={Rank analysis of incomplete block designs: I. the method of paired comparisons},
  author={Bradley, Ralph Allan and Terry, Milton E},
  journal={Biometrika},
  volume={39},
  number={3/4},
  pages={324--345},
  year={1952},
  publisher={JSTOR}
}

@article{christiano2017deep,
  title={Deep reinforcement learning from human preferences},
  author={Christiano, Paul F and Leike, Jan and Brown, Tom and Martic, Miljan and Legg, Shane and Amodei, Dario},
  journal={Advances in neural information processing systems},
  volume={30},
  year={2017}
}

@article{garg2025ipo,
  title={Ipo: Your language model is secretly a preference classifier},
  author={Garg, Shivank and Singh, Ayush and Singh, Shweta and Chopra, Paras},
  journal={arXiv preprint arXiv:2502.16182},
  year={2025}
}

@article{zhang2024diverging,
  title={Diverging Preferences: When do Annotators Disagree and do Models Know?},
  author={Zhang, Michael JQ and Wang, Zhilin and Hwang, Jena D and Dong, Yi and Delalleau, Olivier and Choi, Yejin and Choi, Eunsol and Ren, Xiang and Pyatkin, Valentina},
  journal={arXiv preprint arXiv:2410.14632},
  year={2024}
}

@article{chen2024pal,
  title={Pal: Pluralistic alignment framework for learning from heterogeneous preferences},
  author={Chen, Daiwei and Chen, Yi and Rege, Aniket and Vinayak, Ramya Korlakai},
  journal={arXiv preprint arXiv:2406.08469},
  year={2024}
}

@article{casper2023open,
  title={Open problems and fundamental limitations of reinforcement learning from human feedback},
  author={Casper, Stephen and Davies, Xander and Shi, Claudia and Gilbert, Thomas Krendl and Scheurer, J{\'e}r{\'e}my and Rando, Javier and Freedman, Rachel and Korbak, Tomasz and Lindner, David and Freire, Pedro and others},
  journal={arXiv preprint arXiv:2307.15217},
  year={2023}
}

@inproceedings{gao2023scaling,
  title={Scaling laws for reward model overoptimization},
  author={Gao, Leo and Schulman, John and Hilton, Jacob},
  booktitle={International Conference on Machine Learning},
  pages={10835--10866},
  year={2023},
  organization={PMLR}
}

@article{rafailov2024scaling,
  title={Scaling laws for reward model overoptimization in direct alignment algorithms},
  author={Rafailov, Rafael and Chittepu, Yaswanth and Park, Ryan and Sikchi, Harshit Sushil and Hejna, Joey and Knox, Brad and Finn, Chelsea and Niekum, Scott},
  journal={Advances in Neural Information Processing Systems},
  volume={37},
  pages={126207--126242},
  year={2024}
}

@article{wu2024towards,
  title={Towards robust alignment of language models: Distributionally robustifying direct preference optimization},
  author={Wu, Junkang and Xie, Yuexiang and Yang, Zhengyi and Wu, Jiancan and Chen, Jiawei and Gao, Jinyang and Ding, Bolin and Wang, Xiang and He, Xiangnan},
  journal={arXiv preprint arXiv:2407.07880},
  year={2024}
}

@article{ramesh2024group,
  title={Group robust preference optimization in reward-free rlhf},
  author={Ramesh, Shyam Sundhar and Hu, Yifan and Chaimalas, Iason and Mehta, Viraj and Sessa, Pier Giuseppe and Bou Ammar, Haitham and Bogunovic, Ilija},
  journal={Advances in Neural Information Processing Systems},
  volume={37},
  pages={37100--37137},
  year={2024}
}

@article{banerjee2024towards,
  title={Towards reliable alignment: Uncertainty-aware rlhf},
  author={Banerjee, Debangshu and Gopalan, Aditya},
  journal={arXiv preprint arXiv:2410.23726},
  year={2024}
}

@article{chow2018risk,
  title={Risk-constrained reinforcement learning with percentile risk criteria},
  author={Chow, Yinlam and Ghavamzadeh, Mohammad and Janson, Lucas and Pavone, Marco},
  journal={Journal of Machine Learning Research},
  volume={18},
  number={167},
  pages={1--51},
  year={2018}
}

@article{tamar2015policy,
  title={Policy gradient for coherent risk measures},
  author={Tamar, Aviv and Chow, Yinlam and Ghavamzadeh, Mohammad and Mannor, Shie},
  journal={Advances in neural information processing systems},
  volume={28},
  year={2015}
}

@article{ichihara2025evaluation,
  title={Evaluation of best-of-n sampling strategies for language model alignment},
  author={Ichihara, Yuki and Jinnai, Yuu and Morimura, Tetsuro and Ariu, Kaito and Abe, Kenshi and Sakamoto, Mitsuki and Uchibe, Eiji},
  journal={arXiv preprint arXiv:2502.12668},
  year={2025}
}

@article{sun2024fast,
  title={Fast best-of-n decoding via speculative rejection},
  author={Sun, Hanshi and Haider, Momin and Zhang, Ruiqi and Yang, Huitao and Qiu, Jiahao and Yin, Ming and Wang, Mengdi and Bartlett, Peter and Zanette, Andrea},
  journal={Advances in Neural Information Processing Systems},
  volume={37},
  pages={32630--32652},
  year={2024}
}

@article{kopf2023openassistant,
  title={Openassistant conversations-democratizing large language model alignment},
  author={K{\"o}pf, Andreas and Kilcher, Yannic and Von R{\"u}tte, Dimitri and Anagnostidis, Sotiris and Tam, Zhi Rui and Stevens, Keith and Barhoum, Abdullah and Nguyen, Duc and Stanley, Oliver and Nagyfi, Rich{\'a}rd and others},
  journal={Advances in neural information processing systems},
  volume={36},
  pages={47669--47681},
  year={2023}
}

@misc{nicholas22aira,
  doi = {10.5281/zenodo.6989727},
  url = {https://github.com/Nkluge-correa/Aira},
  author = {Nicholas Kluge Corrêa},
  title = {Aira},
  year = {2023},
  publisher = {GitHub},
  journal = {GitHub repository},
}

@inproceedings{hung2025reward,
  title={Reward-Guided Tree Search for Inference Time Alignment of Large Language Models},
  author={Hung, Chia-Yu and Majumder, Navonil and Mehrish, Ambuj and Poria, Soujanya},
  booktitle={Proceedings of the 2025 Conference of the Nations of the Americas Chapter of the Association for Computational Linguistics: Human Language Technologies (Volume 1: Long Papers)},
  pages={12575--12593},
  year={2025}
}

@article{huang2025best,
  title={Is best-of-n the best of them? coverage, scaling, and optimality in inference-time alignment},
  author={Huang, Audrey and Block, Adam and Liu, Qinghua and Jiang, Nan and Krishnamurthy, Akshay and Foster, Dylan J},
  journal={arXiv preprint arXiv:2503.21878},
  year={2025}
}

@inproceedings{lambert2025rewardbench,
  title={Rewardbench: Evaluating reward models for language modeling},
  author={Lambert, Nathan and Pyatkin, Valentina and Morrison, Jacob and Miranda, Lester James Validad and Lin, Bill Yuchen and Chandu, Khyathi and Dziri, Nouha and Kumar, Sachin and Zick, Tom and Choi, Yejin and others},
  booktitle={Findings of the Association for Computational Linguistics: NAACL 2025},
  pages={1755--1797},
  year={2025}
}

@article{zhou2024rmb,
  title={RMB: Comprehensively benchmarking reward models in LLM alignment},
  author={Zhou, Enyu and Zheng, Guodong and Wang, Binghai and Xi, Zhiheng and Dou, Shihan and Bao, Rong and Shen, Wei and Xiong, Limao and Fan, Jessica and Mou, Yurong and others},
  journal={arXiv preprint arXiv:2410.09893},
  year={2024}
}

@article{liu2024rm,
  title={Rm-bench: Benchmarking reward models of language models with subtlety and style},
  author={Liu, Yantao and Yao, Zijun and Min, Rui and Cao, Yixin and Hou, Lei and Li, Juanzi},
  journal={arXiv preprint arXiv:2410.16184},
  year={2024}
}

@inproceedings{jin2025rag,
  title={Rag-rewardbench: Benchmarking reward models in retrieval augmented generation for preference alignment},
  author={Jin, Zhuoran and Yuan, Hongbang and Men, Tianyi and Cao, Pengfei and Chen, Yubo and Xu, Jiexin and Li, Huaijun and Jiang, Xiaojian and Liu, Kang and Zhao, Jun},
  booktitle={Findings of the Association for Computational Linguistics: ACL 2025},
  pages={17061--17090},
  year={2025}
}

@article{wang2024interpretable,
  title={Interpretable preferences via multi-objective reward modeling and mixture-of-experts},
  author={Wang, Haoxiang and Xiong, Wei and Xie, Tengyang and Zhao, Han and Zhang, Tong},
  journal={arXiv preprint arXiv:2406.12845},
  year={2024}
}

@inproceedings{wang2024helpsteer,
  title={Helpsteer: Multi-attribute helpfulness dataset for steerlm},
  author={Wang, Zhilin and Dong, Yi and Zeng, Jiaqi and Adams, Virginia and Sreedhar, Makesh Narsimhan and Egert, Daniel and Delalleau, Olivier and Scowcroft, Jane and Kant, Neel and Swope, Aidan and others},
  booktitle={Proceedings of the 2024 Conference of the North American Chapter of the Association for Computational Linguistics: Human Language Technologies (Volume 1: Long Papers)},
  pages={3371--3384},
  year={2024}
}

@article{wang2024helpsteer2,
  title={Helpsteer2-preference: Complementing ratings with preferences},
  author={Wang, Zhilin and Bukharin, Alexander and Delalleau, Olivier and Egert, Daniel and Shen, Gerald and Zeng, Jiaqi and Kuchaiev, Oleksii and Dong, Yi},
  journal={arXiv preprint arXiv:2410.01257},
  year={2024}
}

@article{song2024importance,
  title={The importance of online data: Understanding preference fine-tuning via coverage},
  author={Song, Yuda and Swamy, Gokul and Singh, Aarti and Bagnell, J and Sun, Wen},
  journal={Advances in Neural Information Processing Systems},
  volume={37},
  pages={12243--12270},
  year={2024}
}

@article{cen2024value,
  title={Value-incentivized preference optimization: A unified approach to online and offline rlhf},
  author={Cen, Shicong and Mei, Jincheng and Goshvadi, Katayoon and Dai, Hanjun and Yang, Tong and Yang, Sherry and Schuurmans, Dale and Chi, Yuejie and Dai, Bo},
  journal={arXiv preprint arXiv:2405.19320},
  year={2024}
}

@article{chakraborty2024maxmin,
  title={MaxMin-RLHF: Alignment with diverse human preferences},
  author={Chakraborty, Souradip and Qiu, Jiahao and Yuan, Hui and Koppel, Alec and Huang, Furong and Manocha, Dinesh and Bedi, Amrit Singh and Wang, Mengdi},
  journal={arXiv preprint arXiv:2402.08925},
  year={2024}
}

@article{yang2024metaaligner,
  title={Metaaligner: Towards generalizable multi-objective alignment of language models},
  author={Yang, Kailai and Liu, Zhiwei and Xie, Qianqian and Huang, Jimin and Zhang, Tianlin and Ananiadou, Sophia},
  journal={Advances in Neural Information Processing Systems},
  volume={37},
  pages={34453--34486},
  year={2024}
}

@article{guo2024controllable,
  title={Controllable preference optimization: Toward controllable multi-objective alignment},
  author={Guo, Yiju and Cui, Ganqu and Yuan, Lifan and Ding, Ning and Sun, Zexu and Sun, Bowen and Chen, Huimin and Xie, Ruobing and Zhou, Jie and Lin, Yankai and others},
  journal={arXiv preprint arXiv:2402.19085},
  year={2024}
}

@article{li2025multi,
  title   = {{ENCORE}: Entropy-guided Reward Composition for Multi-head Safety Reward Models},
  author  = {Li, Xiaomin and Chen, Xupeng and Fan, Jingxuan and Jiang, Eric Hanchen and Gao, Mingye},
  journal = {Proceedings of the AAAI Conference on Artificial Intelligence},
  volume  = {40},
  number  = {37},
  pages   = {31743--31750},
  year    = {2026},
  doi     = {10.1609/aaai.v40i37.40442}
}

@inproceedings{huang2025deal,
  title={Deal: Decoding-time alignment for large language models},
  author={Huang, James Y and Sengupta, Sailik and Bonadiman, Daniele and Lai, Yi-an and Gupta, Arshit and Pappas, Nikolaos and Mansour, Saab and Kirchhoff, Katrin and Roth, Dan},
  booktitle={Proceedings of the 63rd Annual Meeting of the Association for Computational Linguistics (Volume 1: Long Papers)},
  pages={26280--26300},
  year={2025}
}

@book{vershynin2018high,
  title     = {High-Dimensional Probability: An Introduction with Applications in Data Science},
  author    = {Vershynin, Roman},
  series    = {Cambridge Series in Statistical and Probabilistic Mathematics},
  volume    = {47},
  year      = {2018},
  publisher = {Cambridge University Press},
  address   = {Cambridge},
  doi       = {10.1017/9781108231596},
  isbn      = {9781108231596}
}

@article{artzner1999coherent,
  title={Coherent measures of risk},
  author={Artzner, Philippe and Delbaen, Freddy and Eber, Jean-Marc and Heath, David},
  journal={Mathematical finance},
  volume={9},
  number={3},
  pages={203--228},
  year={1999},
  publisher={Wiley Online Library}
}

@article{rockafellar2000optimization,
  title={Optimization of conditional value-at-risk},
  author={Rockafellar, R Tyrrell and Uryasev, Stanislav and others},
  journal={Journal of risk},
  volume={2},
  pages={21--42},
  year={2000}
}

@book{boyd2004convex,
  title={Convex optimization},
  author={Boyd, Stephen and Vandenberghe, Lieven},
  year={2004},
  publisher={Cambridge university press}
}

@book{miettinen1999nonlinear,
  title={Nonlinear multiobjective optimization},
  author={Miettinen, Kaisa},
  volume={12},
  year={1999},
  publisher={Springer Science \& Business Media}
}

@book{Boucheron2013Concentration,
    author = {Boucheron, Stéphane and Lugosi, Gábor and Massart, Pascal},
    title = {Concentration Inequalities: A Nonasymptotic Theory of Independence},
    publisher = {Oxford University Press},
    year = {2013},
    month = {02},
    isbn = {9780199535255},
    doi = {10.1093/acprof:oso/9780199535255.001.0001}
}

@article{kendall2017uncertainties,
  title={What uncertainties do we need in bayesian deep learning for computer vision?},
  author={Kendall, Alex and Gal, Yarin},
  journal={Advances in neural information processing systems},
  volume={30},
  year={2017}
}

@article{lakshminarayanan2017simple,
  title={Simple and scalable predictive uncertainty estimation using deep ensembles},
  author={Lakshminarayanan, Balaji and Pritzel, Alexander and Blundell, Charles},
  journal={Advances in neural information processing systems},
  volume={30},
  year={2017}
}

@article{uma2021learning,
  title={Learning from disagreement: A survey},
  author={Uma, Alexandra N and Fornaciari, Tommaso and Hovy, Dirk and Paun, Silviu and Plank, Barbara and Poesio, Massimo},
  journal={Journal of Artificial Intelligence Research},
  volume={72},
  pages={1385--1470},
  year={2021}
}

@article{chow2015risk,
  title={Risk-sensitive and robust decision-making: a cvar optimization approach},
  author={Chow, Yinlam and Tamar, Aviv and Mannor, Shie and Pavone, Marco},
  journal={Advances in neural information processing systems},
  volume={28},
  year={2015}
}

@misc{mitchell2023note,
  title={A note on dpo with noisy preferences \& relationship to ipo},
  author={Mitchell, Eric},
  year={2023}
}

@article{chowdhury2024provably,
  title={Provably robust dpo: Aligning language models with noisy feedback},
  author={Chowdhury, Sayak Ray and Kini, Anush and Natarajan, Nagarajan},
  journal={arXiv preprint arXiv:2403.00409},
  year={2024}
}

@article{namkoong2016stochastic,
  title={Stochastic gradient methods for distributionally robust optimization with f-divergences},
  author={Namkoong, Hongseok and Duchi, John C},
  journal={Advances in neural information processing systems},
  volume={29},
  year={2016}
}

@article{duchi2019variance,
  title={Variance-based regularization with convex objectives},
  author={Duchi, John and Namkoong, Hongseok},
  journal={Journal of Machine Learning Research},
  volume={20},
  number={68},
  pages={1--55},
  year={2019}
}

@article{duchi2021statistics,
  title={Statistics of robust optimization: A generalized empirical likelihood approach},
  author={Duchi, John C and Glynn, Peter W and Namkoong, Hongseok},
  journal={Mathematics of Operations Research},
  volume={46},
  number={3},
  pages={946--969},
  year={2021},
  publisher={INFORMS}
}

@article{ben2002robust,
  title={Robust optimization--methodology and applications},
  author={Ben-Tal, Aharon and Nemirovski, Arkadi},
  journal={Mathematical programming},
  volume={92},
  number={3},
  pages={453--480},
  year={2002},
  publisher={Springer}
}

@article{duchi2021learning,
  title={Learning models with uniform performance via distributionally robust optimization},
  author={Duchi, John C and Namkoong, Hongseok},
  journal={The Annals of Statistics},
  volume={49},
  number={3},
  pages={1378--1406},
  year={2021},
  publisher={Institute of Mathematical Statistics}
}

@article{wiesemann2014distributionally,
  title={Distributionally robust convex optimization},
  author={Wiesemann, Wolfram and Kuhn, Daniel and Sim, Melvyn},
  journal={Operations research},
  volume={62},
  number={6},
  pages={1358--1376},
  year={2014},
  publisher={INFORMS}
}

@article{rahimian2019distributionally,
  title={Distributionally robust optimization: A review},
  author={Rahimian, Hamed and Mehrotra, Sanjay},
  journal={arXiv preprint arXiv:1908.05659},
  year={2019}
}

@article{levy2020large,
  title={Large-scale methods for distributionally robust optimization},
  author={Levy, Daniel and Carmon, Yair and Duchi, John C and Sidford, Aaron},
  journal={Advances in neural information processing systems},
  volume={33},
  pages={8847--8860},
  year={2020}
}

@book{dupuis2011weak,
  title={A weak convergence approach to the theory of large deviations},
  author={Dupuis, Paul and Ellis, Richard S},
  year={2011},
  publisher={John Wiley \& Sons}
}

@book{hansen2011robustness,
  title     = {Robustness},
  author    = {Hansen, Lars Peter and Sargent, Thomas J.},
  year      = {2008},
  publisher = {Princeton University Press},
  address   = {Princeton, NJ},
  isbn      = {9780691114422},
  doi       = {10.1515/9781400829385}
}

@article{zheng2023judging,
  title={Judging llm-as-a-judge with mt-bench and chatbot arena},
  author={Zheng, Lianmin and Chiang, Wei-Lin and Sheng, Ying and Zhuang, Siyuan and Wu, Zhanghao and Zhuang, Yonghao and Lin, Zi and Li, Zhuohan and Li, Dacheng and Xing, Eric and others},
  journal={Advances in neural information processing systems},
  volume={36},
  pages={46595--46623},
  year={2023}
}

@article{dubois2024length,
  title={Length-controlled alpacaeval: A simple way to debias automatic evaluators},
  author={Dubois, Yann and Galambosi, Bal{\'a}zs and Liang, Percy and Hashimoto, Tatsunori B},
  journal={arXiv preprint arXiv:2404.04475},
  year={2024}
}

@misc{li2023alpacaeval,
  title        = {AlpacaEval: An Automatic Evaluator of Instruction-following Models},
  author       = {Li, Xuechen and Zhang, Tianyi and Dubois, Yann and Taori, Rohan and Gulrajani, Ishaan and Guestrin, Carlos and Liang, Percy and Hashimoto, Tatsunori B.},
  year         = {2023},
  month        = {5},
  publisher    = {GitHub},
  journal      = {GitHub repository},
  howpublished = {\url{https://github.com/tatsu-lab/alpaca_eval}}
}

@article{dubey2024llama,
  title={The llama 3 herd of models},
  author={Dubey, Abhimanyu and Jauhri, Abhinav and Pandey, Abhinav and Kadian, Abhishek and Al-Dahle, Ahmad and Letman, Aiesha and Mathur, Akhil and Schelten, Alan and Yang, Amy and Fan, Angela and others},
  journal={arXiv e-prints},
  pages={arXiv--2407},
  year={2024}
}

@misc{qwen2025qwen25technicalreport,
      title={Qwen2.5 Technical Report}, 
      author={Qwen and An Yang and Baosong Yang and Beichen Zhang and Binyuan Hui and Bo Zheng and Bowen Yu and Chengyuan Li and Dayiheng Liu and Fei Huang and Haoran Wei and Huan Lin and Jian Yang and Jianhong Tu and Jianwei Zhang and Jianxin Yang and Jiaxi Yang and Jingren Zhou and Junyang Lin and Kai Dang and Keming Lu and Keqin Bao and Kexin Yang and Le Yu and Mei Li and Mingfeng Xue and Pei Zhang and Qin Zhu and Rui Men and Runji Lin and Tianhao Li and Tianyi Tang and Tingyu Xia and Xingzhang Ren and Xuancheng Ren and Yang Fan and Yang Su and Yichang Zhang and Yu Wan and Yuqiong Liu and Zeyu Cui and Zhenru Zhang and Zihan Qiu},
      year={2025},
      eprint={2412.15115},
      archivePrefix={arXiv},
      primaryClass={cs.CL}
}

@article{liu2024skywork,
  title={Skywork-reward: Bag of tricks for reward modeling in llms},
  author={Liu, Chris Yuhao and Zeng, Liang and Liu, Jiacai and Yan, Rui and He, Jujie and Wang, Chaojie and Yan, Shuicheng and Liu, Yang and Zhou, Yahui},
  journal={arXiv preprint arXiv:2410.18451},
  year={2024}
}

@inproceedings{gal2016dropout,
  title={Dropout as a bayesian approximation: Representing model uncertainty in deep learning},
  author={Gal, Yarin and Ghahramani, Zoubin},
  booktitle={international conference on machine learning},
  pages={1050--1059},
  year={2016},
  organization={PMLR}
}

@inproceedings{jinnai2024regularized,
  title={Regularized best-of-n sampling to mitigate reward hacking for language model alignment},
  author={Jinnai, Yuu and Morimura, Tetsuro and Ariu, Kaito and Abe, Kenshi},
  booktitle={ICML 2024 Workshop on Models of Human Feedback for AI Alignment},
  year={2024}
}

@inproceedings{
mitigatingiclr26,
title={From Curiosity to Caution: Mitigating Reward Hacking for Best-of-\$N\$ with Pessimism},
author={Zhuohao Yu and Steven Wu and Adam Block},
booktitle={The Fourteenth International Conference on Learning Representations},
year={2026},
url={https://openreview.net/forum?id=EZn2TmBBfF}
}

@article{khalaf2025inference,
  title={Inference-Time Reward Hacking in Large Language Models},
  author={Khalaf, Hadi and Verdun, Claudio Mayrink and Oesterling, Alex and Lakkaraju, Himabindu and Calmon, Flavio du Pin},
  journal={arXiv preprint arXiv:2506.19248},
  year={2025}
}
\bibliographystyle{icml2026}

\newpage
\appendix

\onecolumn
\section{Additional Theoretical Details}
\label{app:theory}

\subsection{A self-normalized (empirical-variance) concentration bound}
\label{app:emp-bern}

We first record a standard empirical-Bernstein style inequality (a self-normalized concentration bound)
that controls the deviation of the empirical mean in terms of the empirical standard deviation.
This lemma is used to prove Proposition~\ref{prop:uniform-lcb}.

\begin{lemma}[Empirical Bernstein bound (bounded case)]
\label{lem:emp-bern}
Let $X_1,\dots,X_n$ be i.i.d.\ supported on $[a,b]$ with mean $\mu$.
Let
\[
\hat\mu := \frac{1}{n}\sum_{i=1}^n X_i,
\qquad
\hat\sigma^2 := \frac{1}{n-1}\sum_{i=1}^n (X_i-\hat\mu)^2.
\]
Then for any $\delta\in(0,1)$, with probability at least $1-\delta$,
\begin{equation}
\mu \ge \hat\mu
- \sqrt{\frac{2\hat\sigma^2\log(2/\delta)}{n}}
- \frac{7(b-a)\log(2/\delta)}{3(n-1)}.
\label{eq:emp-bern}
\end{equation}
\end{lemma}

\paragraph{Remark.}
Lemma~\ref{lem:emp-bern} is a standard form of empirical Bernstein inequality.
We use it because it yields the same structure as \eqref{eq:lcb-bounded}:
a leading term of order $\hat\sigma\sqrt{\log(1/\delta)/n}$ and a lower-order term of order $\log(1/\delta)/n$.
For completeness, we provide a short proof outline in Appendix~\ref{app:proof-lem-empbern}.

\subsection{Proof of Proposition~\ref{prop:uniform-lcb}}
\label{app:proof-lcb}

\paragraph{Proof outline.}
We apply an empirical-Bernstein (self-normalized) inequality for bounded random variables to each fixed $y$,
then take a union bound over $K$ candidates with $\delta'=\delta/K$.
An alternative derivation can be obtained from Freedman's martingale inequality.
We then track constants to obtain~\eqref{eq:lcb-bounded}.

\paragraph{Proof.}
Fix any $y\in\mathcal{Y}$ and apply Lemma~\ref{lem:emp-bern} to $\{R_i(y)\}_{i=1}^n$ with $\delta'=\delta/K$.
With probability at least $1-\delta'$,
\[
\mu(y)\ge \hat\mu_n(y)
- \hat\sigma_n(y)\sqrt{\frac{2\log(2/\delta')}{n}}
- \frac{7(b-a)\log(2/\delta')}{3(n-1)}.
\]
Taking a union bound over the $K$ candidates yields probability at least $1-\delta$ that the inequality holds for all $y$.
Finally, note that $\log(2/\delta')=\log(2K/\delta)\le \log(K/\delta)+\log 2
\le C_0\,\log(K/\delta)$ for a universal constant $C_0$ (we may assume $K/\delta\ge 2$; otherwise the bound is trivial).
Moreover, $\frac{1}{n-1}\le \frac{2}{n}$ for $n\ge 2$.
Absorbing these constant factors into a single universal $c$ yields~\eqref{eq:lcb-bounded}.
\qed

\subsection{From LCB decoding to mean--dispersion decoding}
\label{app:proof-corollary}

This appendix formalizes the ``ignoring constants / lower-order term'' step used to connect uniform LCB maximization to a \emph{mean--dispersion surrogate} (used as an ablation / efficient approximation), rather than as the defining objective of our primary KL-robust entropic decoder.

\begin{corollary}[Mean--dispersion surrogate form under bounded ratings]

\label{cor:lcb-meanvar}
Under the conditions of Proposition~\ref{prop:uniform-lcb}, define
\[
\mathrm{LCB}_\delta(y)
:= \hat\mu_n(y)
- c\,\hat\sigma_n(y)\sqrt{\frac{\log(K/\delta)}{n}}
- c\,(b-a)\frac{\log(K/\delta)}{n}.
\]
Then any maximizer of $\mathrm{LCB}_\delta(y)$ is also a maximizer of
\begin{equation}
\hat\mu_n(y)-\lambda\hat\sigma_n(y),
\qquad
\lambda := c\sqrt{\frac{\log(K/\delta)}{n}},
\label{eq:cor-meanvar}
\end{equation}
since the last term $c\,(b-a)\log(K/\delta)/n$ is independent of $y$.

\end{corollary}

\begin{proof}
The two objectives differ by an additive constant that does not depend on $y$, hence they have identical argmax sets.
\end{proof}
\begin{remark}[Standard deviation vs.\ Variance penalization]
\label{rem:std-vs-var}
We note a structural distinction between the LCB-derived surrogate and the entropic objective. The LCB rule (Eq.~\ref{eq:cor-meanvar}) implies a penalty on the \emph{standard deviation} $\hat{\sigma}$ (scaling as $O(n^{-1/2})$), which accounts for finite-sample estimation uncertainty. In contrast, the entropic objective (Eq.~1) approximates a mean-variance form $\hat{\mu} - \frac{\beta}{2}\hat{\sigma}^2$ under Taylor expansion, reflecting intrinsic risk aversion. 
While theoretically distinct, both objectives penalize dispersion. We use the $\sigma$-based LCB form primarily to establish finite-sample consistency guarantees, while the entropic form serves as our primary method for capturing risk sensitivity. Empirically, both metrics effectively demote high-disagreement candidates.
\end{remark}

\subsection{Lower-tail (risk) interpretation and uniform empirical LCB under sub-Gaussianity}
\label{app:tail}

This appendix supports the statement in Section~3.1 that an analogous LCB holds under sub-Gaussian noise
(as an alternative to Assumption~3.2). We first recall a lower-tail bound for a single sub-Gaussian variable,
then derive a \emph{uniform-in-$y$} empirical LCB for the mean $\mu(y)$ based on $n$ samples, analogous to
Proposition~\ref{prop:uniform-lcb}.

\paragraph{Bounded $\Rightarrow$ sub-Gaussian.}
Under boundedness $X\in[a,b]$, Hoeffding's lemma implies $X-\E[X]$ is $\nu$-sub-Gaussian with
\(
\nu \le (b-a)/2.
\)
Applying Lemma~\ref{lem:subg-tail} with $\nu=(b-a)/2$ yields the bounded lower-tail bound:
\begin{equation}
\Pr\!\left(X \le \E[X]-t\right) \le \exp\!\left(-\frac{2t^2}{(b-a)^2}\right),\qquad \forall\,t>0.
\label{eq:lower-tail-bounded}
\end{equation}
This justifies the statement that dispersion controls conservative (lower-tail) outcomes.

\begin{lemma}[Sub-Gaussian lower-tail bound]
\label{lem:subg-tail}
If $X-\E[X]$ is $\nu$-sub-Gaussian, then for any $t>0$,
\begin{equation}
\Pr\!\left(X \le \E[X]-t\right) \le \exp\!\left(-\frac{t^2}{2\nu^2}\right).
\label{eq:subg-tail}
\end{equation}
\end{lemma}

\begin{proof}
By Chernoff's method, for any $\lambda>0$,
\[
\Pr(X-\E[X]\le -t)=\Pr(e^{-\lambda(X-\E[X])}\ge e^{\lambda t})
\le e^{-\lambda t}\E[e^{-\lambda(X-\E[X])}]
\le e^{-\lambda t}e^{\lambda^2\nu^2/2}.
\]
Optimize over $\lambda$ by setting $\lambda=t/\nu^2$ to get \eqref{eq:subg-tail}.
\end{proof}

\paragraph{Uniform empirical LCB for sub-Gaussian rewards.}
Fix a state $s$ and a finite candidate set $\mathcal{Y}(s)$ of size $K := |\mathcal{Y}(s)|$.
For each $y\in\mathcal{Y}(s)$, let $R_1(y),\dots,R_n(y)$ be i.i.d.\ draws of the (scalar) reward with mean
$\mu(y):=\E[R(y)]$.
Define the empirical mean and (uncentered) empirical variance proxy
\[
\hat\mu_n(y):=\frac{1}{n}\sum_{i=1}^n R_i(y),
\qquad
\hat\sigma_n^2(y):=\frac{1}{n}\sum_{i=1}^n \bigl(R_i(y)-\hat\mu_n(y)\bigr)^2.
\]
(Any definition of $\hat\sigma_n(y)$ used in the main text can be substituted here; only its nonnegativity is used below.)

\begin{assumption}[Sub-Gaussian rewards (alternative to boundedness)]
\label{ass:subg}
For each $y\in\mathcal{Y}(s)$, the centered reward $R(y)-\mu(y)$ is $\nu$-sub-Gaussian, i.e.,
\(
\E[\exp(\lambda(R(y)-\mu(y)))] \le \exp(\lambda^2\nu^2/2)
\)
for all $\lambda\in\R$.
\end{assumption}

\begin{lemma}[Sample mean is sub-Gaussian]
\label{lem:mean-subg}
Under Assumption~\ref{ass:subg}, for any fixed $y$, the sample mean satisfies that
$\hat\mu_n(y)-\mu(y)$ is $(\nu/\sqrt{n})$-sub-Gaussian, i.e.,
\[
\E\!\left[\exp\!\left(\lambda(\hat\mu_n(y)-\mu(y))\right)\right] \le \exp\!\left(\frac{\lambda^2\nu^2}{2n}\right),
\qquad \forall\,\lambda\in\R.
\]
\end{lemma}

\begin{proof}
Identical to the standard argument: by independence and the sub-Gaussian MGF bound,
\[
\E\!\left[e^{\lambda(\hat\mu_n(y)-\mu(y))}\right]
= \prod_{i=1}^n \E\!\left[\exp\!\left(\frac{\lambda}{n}(R_i(y)-\mu(y))\right)\right]
\le \prod_{i=1}^n \exp\!\left(\frac{\lambda^2\nu^2}{2n^2}\right)
= \exp\!\left(\frac{\lambda^2\nu^2}{2n}\right).
\]
\end{proof}

\begin{proposition}[Uniform sub-Gaussian LCB (analogue of Prop.~\ref{prop:uniform-lcb})]
\label{prop:uniform-lcb-subg}
There exists an absolute constant $c>0$ such that for any $\delta\in(0,1)$, with probability at least $1-\delta$,
simultaneously for all $y\in\mathcal{Y}(s)$,
\begin{equation}
\mu(y) \ge \hat\mu_n(y)
- c\,\nu\,\sqrt{\frac{\log(K/\delta)}{n}}.
\label{eq:lcb-subg}
\end{equation}
We denote the right-hand side by $\mathrm{LCB}^{\mathrm{subG}}_\delta(y)$.
\end{proposition}

\begin{proof}
Fix any $y\in\mathcal{Y}(s)$. By Lemma~\ref{lem:mean-subg} and Lemma~\ref{lem:subg-tail}, for any $t>0$,
\[
\Pr\!\left(\hat\mu_n(y)\le \mu(y)-t\right)
\le \exp\!\left(-\frac{n t^2}{2\nu^2}\right).
\]
Set $t=\nu\sqrt{2\log(K/\delta)/n}$ to obtain
\[
\Pr\!\left(\mu(y) < \hat\mu_n(y)-\nu\sqrt{\frac{2\log(K/\delta)}{n}}\right)\le \frac{\delta}{K}.
\]
Applying a union bound over all $y\in\mathcal{Y}(s)$ yields that with probability at least $1-\delta$,
\eqref{eq:lcb-subg} holds simultaneously for all $y$. Absorb the factor $\sqrt{2}$ into an absolute constant $c$.
\end{proof}

\paragraph{Discussion (relation to Prop.~\ref{prop:uniform-lcb}).}
Proposition~\ref{prop:uniform-lcb-subg} provides the claimed uniform LCB under sub-Gaussian noise.
Compared to the bounded-rewards LCB in \eqref{eq:lcb-bounded}, the sub-Gaussian LCB replaces the empirical
dispersion term $\hat\sigma_n(y)$ and the boundedness-dependent lower-order correction
$c(b-a)\log(K/\delta)/n$ with the (known) sub-Gaussian scale $\nu$.
If one additionally wants a data-dependent (empirical-Bernstein/Freedman-type) bound under sub-Gaussianity,
one can combine sub-exponential concentration for $(R(y)-\mu(y))^2$ with a union bound; we omit this here
since Section~3.1 only requires existence of an analogous LCB.

\subsection{Proof of Theorem~\ref{thm:kl-entropic}}
\label{app:proof-kl-entropic}

The equivalence between the KL-robust value and the entropic risk measure is a classical result in the theory of risk-sensitive control and large deviations, often referred to as the Donsker-Varadhan variational formula~\citep{dupuis2011weak}. We provide the derivation here for completeness.

\begin{proof}
Recall the definition of the KL-regularized robust value given in Equation~\eqref{eq:kl-rob-def}:
\begin{equation}
    \mathrm{Rob}^{\mathrm{KL}}_\beta(\mathbb{P};R)
    := \inf_{\mathbb{Q}\ll \mathbb{P}}
    \left\{ \mathbb{E}_{\mathbb{Q}}[R] + \frac{1}{\beta} D_{\mathrm{KL}}(\mathbb{Q}\|\mathbb{P}) \right\}.
\end{equation}
Let $f(x) = \frac{d\mathbb{Q}}{d\mathbb{P}}(x)$ be the Radon-Nikodym derivative (density ratio). The optimization problem can be formulated as minimizing the functional $J(f)$ subject to the normalization constraint $\int f \,d\mathbb{P} = 1$ and non-negativity $f \ge 0$:
\begin{equation}
    J(f) = \int f(x) R(x) \,d\mathbb{P}(x) + \frac{1}{\beta} \int f(x) \log f(x) \,d\mathbb{P}(x).
\end{equation}
We introduce a Lagrange multiplier $\gamma$ for the constraint $\int f \,d\mathbb{P} = 1$. The Lagrangian is:
\begin{equation}
    \mathcal{L}(f, \gamma) = \int \left( f(x) R(x) + \frac{1}{\beta} f(x) \log f(x) - \gamma f(x) \right) d\mathbb{P}(x) + \gamma.
\end{equation}
Taking the functional derivative with respect to $f(x)$ and setting it to zero yields the first-order condition:
\begin{equation}
    R(x) + \frac{1}{\beta} (\log f(x) + 1) - \gamma = 0.
\end{equation}
Solving for $f(x)$, we obtain the form of the optimal twisted distribution:
\begin{equation}
    \log f^*(x) = \beta\gamma - 1 - \beta R(x) \implies f^*(x) \propto \exp(-\beta R(x)).
\end{equation}
Using the normalization constraint $\int f^*(x) \,d\mathbb{P}(x) = 1$, we determine the normalization constant (partition function) $Z$:
\begin{equation}
    Z := \int \exp(-\beta R(x)) \,d\mathbb{P}(x) = \mathbb{E}_{\mathbb{P}}[\exp(-\beta R)].
\end{equation}
Thus, the optimal density ratio is given by Gibbs measure form $f^*(x) = \frac{1}{Z}\exp(-\beta R(x))$. Substituting $f^*$ back into the objective function, the expectation term and the entropy term combine conveniently:
\begin{align}
    \mathbb{E}_{\mathbb{Q}^*}[R] + \frac{1}{\beta} D_{\mathrm{KL}}(\mathbb{Q}^*\|\mathbb{P})
    &= \int f^*(x) \left( R(x) + \frac{1}{\beta}\log f^*(x) \right) \,d\mathbb{P}(x) \\
    &= \int f^*(x) \left( R(x) + \frac{1}{\beta}(-\beta R(x) - \log Z) \right) \,d\mathbb{P}(x) \\
    &= \int f^*(x) \left( R(x) - R(x) - \frac{1}{\beta}\log Z \right) \,d\mathbb{P}(x) \\
    &= -\frac{1}{\beta}\log Z \underbrace{\int f^*(x) \,d\mathbb{P}(x)}_{1} \\
    &= -\frac{1}{\beta}\log \mathbb{E}_{\mathbb{P}}[\exp(-\beta R)].
\end{align}
This concludes the proof.
\end{proof}

\subsection{Constraint vs.\ penalty: Lagrangian view (finite candidate set)}
\label{app:lagrangian}

This appendix supports the ``two interchangeable controls'' claim in the main text
for the entropic-risk-premium formulation.

\begin{lemma}[Pareto optimality and (supported) scalarization]
\label{lem:pareto-scalar}
Fix $\beta>0$ and define the entropic risk premium
$\RPhat(s,y):=\hat\mu(y)-\widehat V_\beta(y)$.
Let
\[
\mathcal{P}_\beta:=\{(\RPhat(s,y),\hat\mu(y)) : y\in\mathcal{Y}\}\subset\mathbb{R}^2
\]
be a finite set.

If $y^\star\in\arg\max_{y\in\mathcal{Y}} \hat\mu(y)-\lambda\,\RPhat(s,y)$ for some $\lambda\ge 0$,
then $y^\star$ is Pareto-optimal with respect to the bi-criteria
$(\RPhat,\hat\mu)$ (i.e., there is no $y$ such that
$\hat\mu(y)\ge \hat\mu(y^\star)$ and $\RPhat(s,y)\le \RPhat(y^\star)$ with at least one strict).

Conversely, if $(\RPhat(y^\star),\hat\mu(y^\star))$ is a \emph{supported} Pareto-optimal point of $\mathcal{P}_\beta$
(i.e., it lies on the upper boundary of $\mathrm{conv}(\mathcal{P}_\beta)$ and admits a supporting line),
then there exists $\lambda\ge 0$ such that
\[
y^\star\in\arg\max_{y\in\mathcal{Y}} \hat\mu(y)-\lambda\,\RPhat(s,y).
\]
\end{lemma}

\begin{proof}
This is a standard scalarization fact for finite bi-criteria optimization~\citep{boyd2004convex,miettinen1999nonlinear}.
If a solution were dominated under $(\mathrm{RP}^c_{\beta},\hat\mu)$, it could not maximize any nonnegative linear scalarization.
Conversely, each supported Pareto-optimal point admits a supporting hyperplane for $\mathrm{conv}(\mathcal{P}_\beta)$, whose slope yields a $\lambda\ge 0$ such that $y^\star$ maximizes $\hat\mu-\lambda\,\mathrm{RP}· _\beta$.
\end{proof}

\subsection{Proof outline for Lemma~\ref{lem:emp-bern}}
\label{app:proof-lem-empbern}

We briefly outline the proof strategy for Lemma~\ref{lem:emp-bern}.

One route is to apply a martingale concentration inequality (Freedman/Bernstein-type) to the centered sum
$\sum_{i=1}^n (X_i-\mu)$ and control the conditional variance by the empirical variance, yielding an ``empirical Bernstein'' bound.
Another route is to invoke a known empirical Bernstein theorem directly.

\subsection{Selection statistics: empirical mean, dispersion, and CVaR}
\label{app:selection_stats}

For each input $s$ and candidate response $y\in\cY(s)$, we estimate robust selection statistics using a fixed reward model $R_\psi$.
Concretely, we construct $M$ style-preserving perturbations $\{\tilde y^{(m)}\}_{m=1}^M$ of $y$ (e.g., formatting or phrasing variations),
and compute $M$ reward samples
$
r^{(m)}(s,y)=R_\psi(s,\tilde y^{(m)}).
$
We then define:
\[
\hat\mu_{\text{sel}}(s,y)=\frac{1}{M}\sum_{m=1}^M r^{(m)}(s,y),
\qquad
\hat\sigma_{\text{sel}}(s,y)=\sqrt{\frac{1}{M-1}\sum_{m=1}^M \big(r^{(m)}(s,y)-\hat\mu_{\text{sel}}(s,y)\big)^2 }.
\]
For CVaR, let $r^{(1:M)}_{(1)}\le \cdots \le r^{(1:M)}_{(M)}$ be the sorted reward samples, and let
$k=\max\{1,\lceil \alpha M\rceil\}$.
We use the empirical lower-tail mean
\[
\widehat{\mathrm{CVaR}}_{\alpha,\text{sel}}(s,y)=\frac{1}{k}\sum_{i=1}^{k} r^{(1:M)}_{(i)}.
\]
In all experiments, $M$ and $\alpha$ are fixed globally (no per-task tuning), and the same reward model $R_\psi$ is used for all methods.
\subsection{Proof of Proposition~\ref{prop:chi2-dro}}
\label{app:proof-chi2}

We prove the lower bound~\eqref{eq:chi2-dro-lower} and characterize when it is tight.
Let $\mathbb{Q}\ll\mathbb{P}$ and define $g := \frac{d\mathbb{Q}}{d\mathbb{P}}-1$.
Then $\mathbb{E}_{\mathbb{P}}[g]=0$ and the $\chi^2$ constraint implies
\[
\mathbb{E}_{\mathbb{P}}[g^2] = D_{\chi^2}(\mathbb{Q}\,\|\,\mathbb{P}) \le \rho.
\]
Moreover,
\[
\mathbb{E}_{\mathbb{Q}}[R]
= \mathbb{E}_{\mathbb{P}}[(1+g)R]
= \mu_{\mathbb{P}} + \mathbb{E}_{\mathbb{P}}[g(R-\mu_{\mathbb{P}})].
\]
By Cauchy--Schwarz,
\[
\mathbb{E}_{\mathbb{P}}[g(R-\mu_{\mathbb{P}})]
\ge
- \sqrt{\mathbb{E}_{\mathbb{P}}[g^2]}\ \sqrt{\mathbb{E}_{\mathbb{P}}[(R-\mu_{\mathbb{P}})^2]}
\ge
-\sqrt{\rho}\,\sigma_{\mathbb{P}}.
\]
Taking the infimum over all feasible $\mathbb{Q}$ yields~\eqref{eq:chi2-dro-lower}.

For tightness, equality in Cauchy--Schwarz is achieved when
\[
g^\star(R) = -\sqrt{\rho}\,\frac{R-\mu_{\mathbb{P}}}{\sigma_{\mathbb{P}}},
\]
which corresponds to the candidate extremal density
\[
\frac{d\mathbb{Q}^\star}{d\mathbb{P}}(R) = 1 + g^\star(R)
= 1 - \sqrt{\rho}\,\frac{R-\mu_{\mathbb{P}}}{\sigma_{\mathbb{P}}},
\]
i.e.,~\eqref{eq:chi2-opt-density}.
This defines a valid probability density if it is nonnegative $\mathbb{P}$-almost surely, which is exactly the stated condition.
A sufficient condition is $R \le \mu_{\mathbb{P}} + \sigma_{\mathbb{P}}/\sqrt{\rho}$ $\mathbb{P}$-a.s., giving~\eqref{eq:chi2-tightness-sufficient}.

Specializing to the empirical distribution $\widehat{\mathbb{P}}_n^y$ yields~\eqref{eq:chi2-dro-empirical} by substituting
$\mu_{\widehat{\mathbb{P}}_n^y}=\widehat{\mu}_n(y)$ and
$\sigma_{\widehat{\mathbb{P}}_n^y}=\sqrt{\widehat v_n(y)}=\sqrt{\tfrac{n-1}{n}}\,\widehat\sigma_n(y)$.
\qed
\subsection{Proof of Lemma~\ref{lem:eps-lagrange}}
\label{app:proof-eps-lagrange}

Let $y_{\epstie}$ be any solution of~\eqref{eq:eps-constrained-entropic}.
If there existed $y\in\cY(s)$ such that $\Vhat(s,y)\ge \Vhat(s,y_{\epstie})$ and
$\hat\sigma(s,y)\le \hat\sigma(s,y_{\epstie})$ with at least one strict inequality, then:
- if $\Vhat(s,y)>\Vhat(s,y_{\epstie})$, $y$ is feasible for~\eqref{eq:eps-constrained-entropic} and strictly better in the objective $\hat\sigma$;
- if $\Vhat(s,y)=\Vhat(s,y_{\epstie})$ but $\hat\sigma(s,y)<\hat\sigma(s,y_{\epstie})$, $y$ is also feasible and strictly better.
Either case contradicts optimality of $y_{\epstie}$. Hence $y_{\epstie}$ is Pareto-optimal.

For the scalarization claim, consider the finite set
\[
\mathcal{S}:=\{(\hat\sigma(s,y),\Vhat(s,y)):\ y\in\cY(s)\}\subset\mathbb{R}^2
\]
and its convex hull $\mathrm{conv}(\mathcal{S})$.
If $(\hat\sigma(s,y_{\epstie}),\Vhat(s,y_{\epstie}))$ is a supported Pareto-optimal point, then there exists a supporting line to $\mathrm{conv}(\mathcal{S})$
with outward normal $(\lambda_\epsilon,1)$ for some $\lambda_\epsilon\ge 0$, i.e.,
\[
\Vhat(s,y) - \lambda_\epsilon \hat\sigma(s,y)
\le
\Vhat(s,y_{\epstie}) - \lambda_\epsilon \hat\sigma(s,y_{\epstie})
\quad \forall\,y\in\cY(s).
\]
Therefore $y_{\epstie}$ maximizes the linear scalarization, proving~\eqref{eq:eps-supported-scalarization}.
\qed

\subsection{Second-moment surrogate of the LCB rule}
\label{app:meanvar-surrogate}

\begin{corollary}[Second-moment surrogate]
\label{cor:meanvar-surrogate}
Up to an additive term uniform across $y$, maximizing $\mathrm{LCB}_\delta(y)$ is equivalent to
\begin{equation}
y_{\mathrm{LCB}}
\in
\arg\max_{y\in\mathcal{Y}}
\Big(\hat\mu_n(y) - \lambda \hat\sigma_n(y)\Big),
\qquad
\lambda := c\sqrt{\tfrac{\log(K/\delta)}{n}}.
\label{eq:meanvar}
\end{equation}
\end{corollary}

\subsection{Robustness to proxy disagreement estimates}
\label{app:proxy-robust}

This appendix formalizes robustness of LCB-based selection when the decoder uses proxy estimates
$(\tilde\mu,\tilde\sigma)$ in place of the empirical statistics $(\hat\mu_n,\hat\sigma_n)$.
We show that under a uniform proxy error event, proxy maximization incurs only an additive slack
$\Delta_{\mathrm{prox}}$ relative to the sample-based optimum.

\paragraph{Assumption (uniform proxy closeness).}
Let $(\tilde\mu(y),\tilde\sigma(y))$ be proxy estimates used by the decoder.
Assume that with probability at least $1-\delta_{\mathrm{prox}}$, for all $y\in\mathcal{Y}$,
\begin{equation}
\big|\tilde\mu(y)-\hat\mu_n(y)\big|\le \varepsilon_\mu,
\qquad
\big|\tilde\sigma(y)-\hat\sigma_n(y)\big|\le \varepsilon_\sigma.
\label{eq:proxy-error}
\end{equation}
\paragraph{Remark (dependence and interpretation).}
Assumption~\eqref{eq:proxy-error} intentionally abstracts away the dependence structure of proxy-generated scores:
the perturbation-and-RM procedure may induce correlated errors and systematic bias.
Our goal here is not to claim that proxy samples satisfy i.i.d.\ concentration, but to show that \emph{any} proxy
estimator that is uniformly close to the sample statistics yields only an additive degradation in the pessimistic value,
captured by $\Delta_{\mathrm{prox}}$.
In Sec.~4, we empirically assess the accuracy of the proxy estimates and their correlation with human disagreement.

\begin{corollary}[Robust LCB under proxy errors]
\label{cor:robust-lcb}
Fix $\delta\in(0,1)$.
On the intersection of the LCB event in Proposition~\ref{prop:uniform-lcb} (proved in Appendix~\ref{app:proof-lcb})
and the proxy error event~\eqref{eq:proxy-error}, simultaneously for all $y\in\mathcal{Y}$,
\begin{equation}
\mu(y)
\;\ge\;
\underbrace{\tilde\mu(y)
- c\,\tilde\sigma(y)\sqrt{\frac{\log(K/\delta)}{n}}
- c\,(b-a)\frac{\log(K/\delta)}{n}}_{\widetilde{\mathrm{LCB}}_\delta(y)}
\;-\;
\Delta_{\mathrm{prox}}.
\label{eq:robust-lcb}
\end{equation}
where
\begin{equation}
\Delta_{\mathrm{prox}}
:= \varepsilon_\mu
+ c\,\varepsilon_\sigma\sqrt{\frac{\log(K/\delta)}{n}}.
\label{eq:prox-gap}
\end{equation}
The intersection event holds with probability at least $1-(\delta+\delta_{\mathrm{prox}})$ by a union bound.

Consequently, letting $\tilde y\in\arg\max_{y\in\mathcal{Y}} \widetilde{\mathrm{LCB}}_\delta(y)$ and
$y^\star\in\arg\max_{y\in\mathcal{Y}} \mathrm{LCB}_\delta(y)$,
we have the value suboptimality bound
\begin{equation}
\mathrm{LCB}_\delta(y^\star) - \mathrm{LCB}_\delta(\tilde y) \;\le\; 2\,\Delta_{\mathrm{prox}}.
\label{eq:proxy-2delta}
\end{equation}
In particular, proxy maximization is within $O(\Delta_{\mathrm{prox}})$ of the sample-based optimum.
\end{corollary}

\paragraph{Proof.}
By Proposition~\ref{prop:uniform-lcb} (proved in Appendix~\ref{app:proof-lcb}),
for all $y$ on the LCB event,
\[
\mu(y)\ge \hat\mu_n(y)
- c\,\hat\sigma_n(y)\sqrt{\frac{\log(K/\delta)}{n}}
- c\,(b-a)\frac{\log(K/\delta)}{n}.
\]
Using~\eqref{eq:proxy-error}, we have $\hat\mu_n(y)\ge \tilde\mu(y)-\varepsilon_\mu$ and
$\hat\sigma_n(y)\le \tilde\sigma(y)+\varepsilon_\sigma$.
Substituting and rearranging yields~\eqref{eq:robust-lcb}--\eqref{eq:prox-gap}.

Finally, since for all $y$ we have
$\mathrm{LCB}_\delta(y)\ge \widetilde{\mathrm{LCB}}_\delta(y)-\Delta_{\mathrm{prox}}$
and
$\mathrm{LCB}_\delta(y)\le \widetilde{\mathrm{LCB}}_\delta(y)+\Delta_{\mathrm{prox}}$,
standard argmax stability yields~\eqref{eq:proxy-2delta}.
\qed

\subsection{$\epsilon$-tie breaking: Pareto and scalarization properties}
\label{app:eps-pareto}

\begin{lemma}[$\epsilon$-tie breaking on the robust-value--disagreement frontier]
\label{lem:eps-lagrange}
Fix a prompt $s$ and candidate set $\cY(s)$.
Any solution $y_{\epstie}(s)$ of~~\eqref{eq:eps-constrained-entropic} is Pareto-optimal with respect to $(\Vhat,\hat\sigma)$:
there does not exist $y\in\cY(s)$ such that
$\Vhat(s,y)\ge \Vhat(s,y_{\epstie})$ and $\hat\sigma(s,y)\le \hat\sigma(s,y_{\epstie})$ with at least one strict inequality.
Moreover, if $(\hat\sigma(s,y_{\epstie}),\Vhat(s,y_{\epstie}))$ is a \emph{supported} Pareto-optimal point of the finite set
$\{(\hat\sigma(s,y),\Vhat(s,y)) : y\in\cY(s)\}\subset\mathbb{R}^2$,
then there exists a (possibly non-unique) $\lambda_\epsilon\ge 0$ such that
\begin{equation}
y_{\epstie}(s)\in \arg\max_{y\in\cY(s)} \big(\Vhat(s,y)-\lambda_\epsilon \hat\sigma(s,y)\big).
\label{eq:eps-supported-scalarization}
\end{equation}

\end{lemma}

\begin{proof}
See Appendix~\ref{app:proof-eps-lagrange}.
\end{proof}

\begin{remark}[$\epsilon$-rule approximates $\Vhat$--dispersion under small robust-value gaps]
Fix $\lambda\ge 0$ and let
$y_\lambda\in\arg\max_{y\in\cY(s)} \big(\Vhat(s,y)-\lambda\hat\sigma(s,y)\big)$.
If $y_\lambda\in\mathcal{F}_\epsilon(s)$ (i.e., its robust value is within $\epsilon$ of the best robust value),
then the $\epsilon$-tie-break output $y_{\epstie}$ satisfies
\[
\Vhat(s,y_{\epstie})-\lambda\hat\sigma(s,y_{\epstie})
\ge
\Vhat(s,y_\lambda)-\lambda\hat\sigma(s,y_\lambda)-\epstie.
\]
\end{remark}

\subsection{From pairwise comparisons to scalar samples}
\label{app:pairwise_bridge}
\paragraph{From pairwise preferences to scalar satisfaction (estimation layer).}
Human feedback is often collected as pairwise comparisons or multi-criteria judgments rather than scalar scores.
One may induce scalar satisfaction via standard latent-utility scalarizations (e.g., Bradley--Terry/Thurstone-style fitting or calibrated ratings); our method then operates on the induced scalar samples.
We emphasize that the \emph{high-probability guarantee} in Section~\ref{sec:lcb} applies directly to the i.i.d.\ scalar-sample regime above; when scalarization is obtained indirectly from comparisons, we treat the resulting $(\hat\mu,\hat\sigma)$ as an estimator and validate it empirically.
We sketch how common pairwise preference data can be mapped to scalar samples used in Sec.~\ref{sec:lcb}.
Assume annotator $i$ has a latent utility $U_i(y)$ and pairwise labels are generated from noisy utility differences,
e.g., Bradley--Terry / Thurstone:
\begin{equation}
\Pr(y \succ_i y') = \sigma\!\left(\frac{U_i(y)-U_i(y')}{\eta}\right),
\end{equation}
where $\sigma(\cdot)$ is a logistic or Gaussian CDF and $\eta>0$.
Given multiple comparisons involving $y$, a scalarization (e.g., per-annotator win-rate or fitted BT/Thurstone score) yields a proxy $\widehat U_i(y)$.

\begin{lemma}[LCB form under sub-Gaussian scalarization noise]
\label{lem:pairwise_subg}
Fix $y$ and suppose $\widehat U_i(y) = U_i(y) + \xi_i(y)$ where $\{\xi_i(y)\}_{i=1}^n$ are independent and $\nu$-sub-Gaussian.
Then with probability at least $1-\delta$, simultaneously for all $y\in\mathcal{Y}(s)$,
\[
U(y) \;\ge\; \widehat\mu_n(y) - c'\widehat\sigma_n(y)\sqrt{\frac{\log(K/\delta)}{n}},
\]
for an absolute constant $c'>0$ (up to lower-order terms), i.e., the same mean--dispersion LCB structure as in~\eqref{eq:lcb-bounded}.
\end{lemma}

\section{Proof of Proposition~\ref{prop:chi2-dro} ($\chi^2$-DRO robust mean)}
\label{app:chi2-dro}

\paragraph{Full characterization.}
A sufficient tightness condition is that the extremal density
\[
\frac{d\mathbb{Q}^\star}{d\mathbb{P}}(R)
=
1 - \sqrt{\rho}\,\frac{R-\mu_{\mathbb{P}}}{\sigma_{\mathbb{P}}}
\]
is nonnegative $\mathbb{P}$-almost surely.
In particular, it holds if $R \le \mu_{\mathbb{P}} + \sigma_{\mathbb{P}}/\sqrt{\rho}$ a.s.
For bounded ratings $R\in[a,b]$, nonnegativity holds whenever $\sqrt{\rho}\,(b-\mu_{\mathbb{P}})\le \sigma_{\mathbb{P}}$.

We provide a complete proof of Proposition~\ref{prop:chi2-dro}. We first prove the general lower bound
$\inf_{\mathbb{Q}\in\mathcal{U}_\rho(\mathbb{P})}\mathbb{E}_{\mathbb{Q}}[R]
\ge \mu_{\mathbb{P}}-\sqrt{\rho}\,\sigma_{\mathbb{P}}$ for any square-integrable $R$,
and then characterize a sufficient condition for tightness. Finally, we specialize to the empirical distribution.

\subsection{General bound via a change of measure}
Let $\mathbb{Q}\ll \mathbb{P}$ and define
\[
g(R) := \frac{d\mathbb{Q}}{d\mathbb{P}}(R)-1.
\]
Then $\mathbb{E}_{\mathbb{P}}[g]=0$ and the $\chi^2$ constraint reads $\mathbb{E}_{\mathbb{P}}[g^2]\le \rho$.
Moreover,
\begin{align*}
\mathbb{E}_{\mathbb{Q}}[R] - \mathbb{E}_{\mathbb{P}}[R]
&= \mathbb{E}_{\mathbb{P}}\!\left[\frac{d\mathbb{Q}}{d\mathbb{P}}(R)\,R\right] - \mathbb{E}_{\mathbb{P}}[R] \\
&= \mathbb{E}_{\mathbb{P}}[(1+g(R))R]-\mathbb{E}_{\mathbb{P}}[R] \\
&= \mathbb{E}_{\mathbb{P}}[g(R)R]
= \mathbb{E}_{\mathbb{P}}[g(R)(R-\mu_{\mathbb{P}})].
\end{align*}
By Cauchy--Schwarz,
\[
\mathbb{E}_{\mathbb{P}}[g(R)(R-\mu_{\mathbb{P}})]
\ge
-\sqrt{\mathbb{E}_{\mathbb{P}}[g^2]}\ \sqrt{\mathbb{E}_{\mathbb{P}}[(R-\mu_{\mathbb{P}})^2]}
\ge
-\sqrt{\rho}\,\sigma_{\mathbb{P}}.
\]
Taking the infimum over all $\mathbb{Q}\in\mathcal{U}_\rho(\mathbb{P})$ yields
\[
\inf_{\mathbb{Q}\in\mathcal{U}_\rho(\mathbb{P})}\mathbb{E}_{\mathbb{Q}}[R]
\ge
\mu_{\mathbb{P}}-\sqrt{\rho}\,\sigma_{\mathbb{P}},
\]
proving~\eqref{eq:chi2-dro-lower}.

\subsection{Tightness and the extremal density}
Equality in Cauchy--Schwarz holds when $g(R)$ is proportional to $-(R-\mu_{\mathbb{P}})$.
Define
\[
g^\star(R) := -\sqrt{\rho}\,\frac{R-\mu_{\mathbb{P}}}{\sigma_{\mathbb{P}}}.
\]
Then $\mathbb{E}_{\mathbb{P}}[g^\star]=0$ and $\mathbb{E}_{\mathbb{P}}[(g^\star)^2]=\rho$.
If additionally $1+g^\star(R)\ge 0$ $\mathbb{P}$-a.s., we may define
\begin{equation}
\frac{d\mathbb{Q}^\star}{d\mathbb{P}}(R) := 1+g^\star(R)
= 1-\sqrt{\rho}\,\frac{R-\mu_{\mathbb{P}}}{\sigma_{\mathbb{P}}},
\label{eq:chi2-opt-density}
\end{equation}
which is a valid Radon--Nikodym derivative.
Plugging $g^\star$ into the derivation above achieves equality, giving
\[
\mathbb{E}_{\mathbb{Q}^\star}[R] = \mu_{\mathbb{P}}-\sqrt{\rho}\,\sigma_{\mathbb{P}}.
\]
A sufficient condition for nonnegativity is
\begin{equation}
R \le \mu_{\mathbb{P}} + \frac{\sigma_{\mathbb{P}}}{\sqrt{\rho}}
\quad \mathbb{P}\text{-a.s.}
\label{eq:chi2-tightness-sufficient}
\end{equation}
which implies $1+g^\star(R)\ge 0$ a.s.

\subsection{Specialization to the empirical distribution}
Now let $\mathbb{P}=\widehat{\mathbb{P}}_n^y=\frac1n\sum_{i=1}^n\delta_{R_i(y)}$.
Any $\mathbb{Q}\ll\widehat{\mathbb{P}}_n^y$ corresponds to a probability vector $q\in\Delta^n$ supported on the same atoms:
\[
\mathbb{Q} = \sum_{i=1}^n q_i \delta_{R_i(y)}.
\]
Since $\widehat{\mathbb{P}}_n^y$ is uniform, the $\chi^2$ divergence becomes
\[
D_{\chi^2}(\mathbb{Q}\|\widehat{\mathbb{P}}_n^y)=\sum_{i=1}^n \frac{(q_i-1/n)^2}{1/n}
= n\sum_{i=1}^n\left(q_i-\frac1n\right)^2.
\]
Moreover,
\[
\mathbb{E}_{\mathbb{Q}}[R] = \sum_{i=1}^n q_i R_i(y),\quad
\mu_{\widehat{\mathbb{P}}_n^y}=\widehat{\mu}_n(y),\quad
\sigma_{\widehat{\mathbb{P}}_n^y}^2=\widehat{v}_n(y),
\]
where $\widehat{v}_n(y)=\frac1n\sum_{i=1}^n (R_i(y)-\widehat{\mu}_n(y))^2$.
Therefore, whenever the empirical nonnegativity condition holds on the support,
the worst-case mean equals the closed form
\[
\inf_{\mathbb{Q}\in\mathcal{U}_\rho(\widehat{\mathbb{P}}_n^y)}\mathbb{E}_{\mathbb{Q}}[R]
=
\widehat{\mu}_n(y)-\sqrt{\rho}\,\sqrt{\widehat{v}_n(y)}.
\]
Otherwise, the lower bound~\eqref{eq:chi2-dro-lower} still holds.
\begin{remark}[Conservative approximation under gap]
\label{rem:dro-gap}
When the nonnegativity condition~\eqref{eq:chi2-tightness-sufficient} is violated (e.g., due to outliers far below the mean or large $\rho$), the Cauchy--Schwarz inequality implies that the closed form $\widehat{\mu}_n(y) - \sqrt{\rho}\widehat{\sigma}_n(y)$ becomes a strict \emph{lower bound} on the true $\chi^2$-DRO objective.
In this regime, maximizing the mean--dispersion surrogate effectively optimizes a quantity that is \emph{more conservative} than the exact distributionally robust mean. This aligns with our overall goal of pessimistic selection: the surrogate does not overstate the robust value even when the local approximation is inexact.
\end{remark}
\qed

\section{Calibrating LCB scaling to $\chi^2$-DRO radius}
\label{app:lcb-dro-calibration}

This appendix formalizes a \emph{numerical calibration} between the LCB penalty in
Proposition~\ref{prop:uniform-lcb} and the ambiguity radius in empirical $\chi^2$-DRO.
Importantly, this calibration is \emph{not} meant to claim that LCB and DRO model the same source of uncertainty:
LCB controls \emph{finite-sample estimation uncertainty} for $\mu(y)$, while $\chi^2$-DRO is commonly interpreted
as \emph{distributional/aleatoric} robustness. Here we use $\chi^2$-DRO as a convenient \emph{variational
reparameterization} of the same mean--dispersion functional, which yields an interpretable radius $\rho$
corresponding to a chosen standard-deviation penalty $\lambda$.

\begin{proposition}[Exact $\lambda$--$\rho$ mapping for empirical $\chi^2$-DRO]
\label{prop:lambda-rho}
Fix a candidate $y$ and its empirical distribution $\widehat{\mathbb{P}}_n^y$.
Assume the nonnegativity condition in Proposition~\ref{prop:chi2-dro} holds on the empirical support.

Then for any $\rho\ge 0$,
\begin{equation}
\inf_{\mathbb{Q}\in \mathcal{U}_\rho(\widehat{\mathbb{P}}_n^y)}\ \mathbb{E}_{\mathbb{Q}}[R]
\;=\;
\widehat{\mu}_n(y)-\lambda_\rho\,\widehat{\sigma}_n(y),
\qquad
\lambda_\rho := \sqrt{\rho}\,\sqrt{\tfrac{n-1}{n}}.
\label{eq:lambda-rho-exact}
\end{equation}
Equivalently, for a desired penalty $\lambda\ge 0$, choosing
\begin{equation}
\rho(\lambda) := \frac{n}{n-1}\,\lambda^2
\label{eq:rho-from-lambda}
\end{equation}
makes the empirical $\chi^2$-DRO robust mean exactly equal to $\widehat{\mu}_n(y)-\lambda\,\widehat{\sigma}_n(y)$.
\end{proposition}

\paragraph{Beyond the nonnegativity regime.}
The equality in \eqref{eq:lambda-rho-exact} requires a sufficient nonnegativity condition
(Proposition~\ref{prop:chi2-dro}). When this condition fails (e.g., due to heavy-tailed empirical support or large $\rho$),
the closed-form expression may no longer be tight. Nevertheless, Proposition~\ref{prop:chi2-dro} still provides the
\emph{conservative lower bound}
\[
\inf_{\mathbb{Q}\in \mathcal{U}_\rho(\widehat{\mathbb{P}}_n^y)}\ \mathbb{E}_{\mathbb{Q}}[R]
\;\ge\;
\widehat{\mu}_n(y)-\lambda_\rho\,\widehat{\sigma}_n(y),
\]
so the mean--dispersion form remains a valid pessimistic surrogate even outside the tightness regime.

\paragraph{Discussion: LCB as calibrated DRO plus lower-order slack.}
Proposition~\ref{prop:uniform-lcb} provides a uniform high-probability lower bound
\[
\mu(y)\ge \widehat{\mu}_n(y)
- c\,\widehat{\sigma}_n(y)\sqrt{\frac{\log(K/\delta)}{n}}
- c\,\nu(y)\frac{\log(K/\delta)}{n}
\quad\text{simultaneously for all }y\in\mathcal{Y}.
\]
Define
\[
\lambda_\delta := c\sqrt{\frac{\log(K/\delta)}{n}}
\qquad\text{and}\qquad
\rho_\delta := \rho(\lambda_\delta)=\frac{n}{n-1}\lambda_\delta^2.
\]
By Proposition~\ref{prop:lambda-rho}, in the nonnegativity regime the empirical $\chi^2$-DRO robust mean under radius $\rho_\delta$
matches $\widehat{\mu}_n(y)-\lambda_\delta\widehat{\sigma}_n(y)$ exactly.
The additional term $c\,\nu(y)\log(K/\delta)/n$ is a finite-sample slack needed to convert this robust value
into a \emph{high-probability} lower bound on the true mean $\mu(y)$ under the assumed concentration scale $\nu(y)$.

\paragraph{Interpretation of the $n$-dependent radius.}
The mapping $\rho_\delta \propto \log(K/\delta)/n$ arises because $\lambda_\delta$ is chosen to control estimation error uniformly over $K$ candidates. 
This $n$-dependence should be interpreted as an \emph{estimation-calibrated} DRO radius: as $n$ grows, the ambiguity set shrinks because the empirical mean becomes well-estimated.
However, we emphasize that in preference alignment, high empirical variance often signals intrinsic heterogeneity (disagreement) rather than pure noise.(empirically supported by the monotonic trend across disagreement buckets; see Fig~\ref{fig:bucket_dtradeoff})
Thus, while the \emph{statistical} radius required for concentration decays with $n$, the \emph{functional form} of the penalty (i.e., penalizing dispersion) remains structurally isomorphic to a robustness constraint against disagreement.
This allows the LCB framework to serve a dual purpose: rigorously controlling finite-sample error while effectively acting as a proxy to demote controversial candidates.

\section{On dependence across candidates and the uniform guarantee}
\label{app:dependence}

Our uniform guarantee in Proposition~\ref{prop:uniform-lcb} does not require independence across candidates.
It only requires per-candidate concentration for each fixed $y$ and a union bound.

\begin{lemma}[Union bound without independence]
\label{lem:union-no-indep}
Let $\{\mathcal{E}_y\}_{y\in\mathcal{Y}}$ be any collection of events (possibly dependent).
If for each $y\in\mathcal{Y}$ we have $\Pr(\mathcal{E}_y)\ge 1-\delta'$, then
\[
\Pr\left(\bigcap_{y\in\mathcal{Y}}\mathcal{E}_y\right)
\ge 1- \sum_{y\in\mathcal{Y}}\Pr(\mathcal{E}_y^c)
\ge 1- K\delta',
\]
where $K:=|\mathcal{Y}|$.
\end{lemma}

\paragraph{Application to Proposition~\ref{prop:uniform-lcb}.}
For each fixed $y$, the self-normalized empirical-Bernstein inequality yields an event $\mathcal{E}_y$
such that $\Pr(\mathcal{E}_y)\ge 1-\delta'$ and on $\mathcal{E}_y$ the bound~\eqref{eq:lcb-bounded} holds.
Annotator overlap across different candidates can create dependence among $\{\mathcal{E}_y\}_{y\in\mathcal{Y}}$,
but Lemma~\ref{lem:union-no-indep} shows that the intersection event still holds with probability at least $1-K\delta'$.
Setting $\delta'=\delta/K$ yields the stated uniform guarantee.

\section{Scorer ambiguity and multi-scorer aggregation}
\label{app:multi_scorer}

This appendix details our proxy instantiation that hedges \emph{scorer ambiguity} when scalar satisfaction samples are produced by learned reward models.
We consider a family of $M$ scorers indexed by $m\in[M]$.
For each prompt $s$, candidate $y\in\cY(s)$, and scorer $m$, we obtain $n$ scalar samples
$\{R_{m,i}(s,y)\}_{i=1}^n$ (e.g., via independent scorer noise, style-preserving perturbations, or repeated evaluations).
All analysis conditions on the realized candidate set $\cY(s)$.

\subsection{Within-prompt normalization across scorers}
Reward-model outputs may differ in scale and offset across scorers; to make aggregation meaningful, we normalize per prompt.
Define the pooled mean and standard deviation under scorer $m$:
\begin{align}
\bar R_m(s)
&:= \frac{1}{Kn}\sum_{y\in\cY(s)}\sum_{i=1}^n R_{m,i}(s,y), \\
\hat s_m(s)
&:= \sqrt{\frac{1}{Kn-1}\sum_{y\in\cY(s)}\sum_{i=1}^n \big(R_{m,i}(s,y)-\bar R_m(s)\big)^2 } \ >\ 0.
\end{align}
We form normalized samples
\begin{equation}
\tilde R_{m,i}(s,y) \ :=\ \frac{R_{m,i}(s,y)-\bar R_m(s)}{\hat s_m(s)}.
\label{eq:normalized-samples}
\end{equation}
When $M=1$, this normalization can be omitted without changing the framework.

\paragraph{Implementation note.}
The constants $(\bar R_m(s),\hat s_m(s))$ require access to all candidates $y\in\cY(s)$ for the given prompt $s$.
In Algorithm~\ref{alg:DARC-full}, we therefore precompute $(\bar R_m(s),\hat s_m(s))$ once per scorer $m$ before
scoring individual candidates.
\subsection{Scorer-specific entropic value, risk premium, and dispersion}
\label{app:scorer-specific}

For each scorer $m$, we compute the scorer-specific empirical entropic value using normalized samples:
\begin{equation}
\widehat V_{\beta,m}(s,y;\bar R_m(s),\hat s_m(s))
\ :=\ 
-\frac{1}{\beta}\log\!\left(\frac{1}{n}\sum_{i=1}^n \exp\!\big(-\beta\,\tilde R_{m,i}(s,y)\big)\right),
\label{eq:scorer-specific-entropic}
\end{equation}
and the corresponding empirical mean
\begin{equation}
\hat\mu_m(s,y;\bar R_m(s),\hat s_m(s)) \ :=\ \frac{1}{n}\sum_{i=1}^n \tilde R_{m,i}(s,y).
\end{equation}
We define the scorer-specific empirical entropic risk premium as
\begin{equation}
\widehat{\mathrm{RP}}_{\beta,m}(s,y;\bar R_m(s),\hat s_m(s))
\ :=\ 
\hat\mu_m(s,y;\bar R_m(s),\hat s_m(s)) - \widehat V_{\beta,m}(s,y;\bar R_m(s),\hat s_m(s))
\ \ge\ 0.
\label{eq:scorer-specific-rp}
\end{equation}
We also define the scorer-specific empirical standard deviation (used as a disagreement proxy) as
\begin{equation}
\hat\sigma_m(s,y;\bar R_m(s),\hat s_m(s))
\ :=\ 
\sqrt{\frac{1}{n-1}\sum_{i=1}^n\Big(\tilde R_{m,i}(s,y)-\hat\mu_m(s,y;\bar R_m(s),\hat s_m(s))\Big)^2 }.
\label{eq:scorer-specific-sigma}
\end{equation}

\subsection{Scorer-robust aggregation and decoding rule}
To hedge scorer shift, we aggregate scorer-specific entropic values via a soft worst-case (soft-min) operator with parameter $\gamma>0$:
\begin{equation}
\widetilde V_{\beta,\gamma}(s,y)
:=
-\frac{1}{\gamma}\log\!\left(\frac{1}{M}\sum_{m=1}^M \exp\!\big(-\gamma\,\widehat V_{\beta,m}(s,y)\big)\right),
\label{eq:scorer-robust-V-app}
\end{equation}
which satisfies $\widetilde V_{\beta,\gamma}(s,y)\to \min_{m\in[M]} \widehat V_{\beta,m}(s,y)$ as $\gamma\to\infty$.
For explicit risk control, we aggregate risk premia pessimistically:
\begin{equation}
\widetilde{\mathrm{RP}}_{\beta}(s,y)
:=
\max_{m\in[M]} \ \widehat{\mathrm{RP}}_{\beta,m}(s,y).
\label{eq:scorer-robust-RP-app}
\end{equation}
The scorer-robust risk-constrained decoding rule is then
\begin{equation}
\hat y \in \arg\max_{y\in\cY(s)}\ \widetilde V_{\beta,\gamma}(s,y)
\quad \text{s.t.}\quad
\widetilde{\mathrm{RP}}_{\beta}(s,y)\le \tau,
\label{eq:scorer-robust-decoding-constrained}
\end{equation}
or in penalized (Lagrangian) form,
\begin{equation}
\hat y \in \arg\max_{y\in\cY(s)}
\ \widetilde V_{\beta,\gamma}(s,y) - \lambda\,\widetilde{\mathrm{RP}}_{\beta}(s,y).
\label{eq:scorer-robust-decoding-penalized}
\end{equation}

\subsection{Interpretation}
The aggregation $\widetilde V_{\beta,\gamma}$ hedges against worst-case scorers while remaining smooth and stable for finite $\gamma$.
The pessimistic risk-premium aggregation $\widetilde{\mathrm{RP}}_{\beta}$ prevents a candidate from being selected if it is deemed high-risk by any scorer.
When scorer variability reflects genuine preference uncertainty (e.g., validated proxies or ensembles calibrated to human disagreement), this multi-scorer instantiation also promotes reliability under heterogeneous preferences.

\section{Detailed decoding rules and baselines}
\label{app:decoding_rules}

\paragraph{Candidate set and selection statistics.}
Given an input $s$, we sample a fixed candidate set $\cY(s)=\{y_1,\dots,y_K\}$ from a generator policy $\pi$.
For each candidate $y$, we compute an empirical satisfaction mean $\hat\mu_{\text{sel}}(s,y)$ and a dispersion proxy $\hat\sigma_{\text{sel}}(s,y)$
(see Appendix~\ref{app:selection_stats} for the exact estimation procedure).
When using CVaR, we additionally compute the empirical lower-tail mean $\widehat{\mathrm{CVaR}}_{\alpha,\text{sel}}(s,y)$.

\subsection{Full pseudocode for DARC}
\label{app:DARC-pseudocode}

\begin{algorithm}[H]
\caption{DARC with scorer-robust aggregation (full version)}
\label{alg:DARC-full}
\begin{algorithmic}
\REQUIRE prompt $s$
\REQUIRE candidate set $\cY(s)=\{y_1,\dots,y_K\}$ (fixed for all methods)
\REQUIRE scorer family $\{\textsc{Score}^{(m)}_\beta\}_{m=1}^M$, where 
$\textsc{Score}^{(m)}_\beta(s,y;\bar R,\hat s)\rightarrow(\hat\mu_m,\widehat{V}_{\beta,m},\hat\sigma_m)$
\REQUIRE aggregation parameter $\gamma>0$ (soft worst-case over scorers)
\REQUIRE variant $\in\{\textsc{Entropic},\textsc{Tau},\textsc{Eps},\textsc{2nd-Moment (LCB)}\}$
\REQUIRE hyperparameters $(\beta,\lambda,\tau,\epsilon,\delta,c_{\mathrm{LCB}})$
\ENSURE selected output $y^\star$

\STATE{\textbf{(Optional but recommended) Within-prompt normalization across scorers (selection-time only).}}
\STATE{\textit{Precompute pooled normalization constants for each scorer $m$ over all candidates $y\in\cY(s)$.}}
\FOR{$m=1$ to $M$}
  \STATE Collect raw samples $\{R_{m,i}(s,y_k)\}_{k\in[K],\,i\in[n]}$
  \STATE $\bar R_m(s)\gets \frac{1}{Kn}\sum_{k=1}^K\sum_{i=1}^n R_{m,i}(s,y_k)$
  \STATE $\hat s_m(s)\gets \sqrt{\frac{1}{Kn-1}\sum_{k=1}^K\sum_{i=1}^n\big(R_{m,i}(s,y_k)-\bar R_m(s)\big)^2}$
\ENDFOR
\STATE{\textit{If normalization is disabled, set $\bar R_m(s)\gets 0$ and $\hat s_m(s)\gets 1$ for all $m$.}}

\FOR{$k=1$ to $K$}
  \FOR{$m=1$ to $M$}
    \STATE $(\hat\mu_{m,k},\widehat{V}_{\beta,m,k},\hat\sigma_{m,k})
    \gets \textsc{Score}^{(m)}_\beta\!\big(s,y_k;\bar R_m(s),\hat s_m(s)\big)$
    \STATE $\widehat{\mathrm{RP}}_{\beta,m,k}\gets \hat\mu_{m,k}-\widehat{V}_{\beta,m,k}$
  \ENDFOR
  \STATE $\widetilde{V}_{\beta,\gamma,k}\gets -\frac{1}{\gamma}\log\!\Big(\frac{1}{M}\sum_{m=1}^M \exp(-\gamma\,\widehat{V}_{\beta,m,k})\Big)$ \COMMENT{soft-min over scorers}
  \STATE $\widetilde{\mathrm{RP}}_{\beta,k}\gets \max_{m\in[M]}\widehat{\mathrm{RP}}_{\beta,m,k}$ \COMMENT{worst-case risk premium}
  \STATE $\widetilde{\sigma}_k\gets \max_{m\in[M]}\hat\sigma_{m,k}$ \COMMENT{worst-case disagreement proxy}
  \STATE $\widetilde{\mu}_k\gets \frac{1}{M}\sum_{m=1}^M \hat\mu_{m,k}$ \COMMENT{averaged mean used by LCB ablation}
\ENDFOR

\IF{variant $=\textsc{Entropic}$} 
  \STATE $k^\star \gets \arg\max_{k\in[K]} \widetilde{V}_{\beta,\gamma,k}$
  \STATE \RETURN $y_{k^\star}$

\ELSIF{variant $=\textsc{Tau}$} 
  \STATE $\mathcal{F}\gets \{k\in[K]:\widetilde{\mathrm{RP}}_{\beta,k}\le\tau\}$
  \IF{$\mathcal{F}=\emptyset$}
    \STATE $\mathcal{F}\gets [K]$
  \ENDIF
  \STATE $k^\star \gets \arg\max_{k\in\mathcal{F}} \widetilde{V}_{\beta,\gamma,k}$
  \STATE \RETURN $y_{k^\star}$

\ELSIF{variant $=\textsc{Eps}$} 
  \STATE $\widetilde{V}_{\beta,\gamma,\max}\gets \max_{k\in[K]} \widetilde{V}_{\beta,\gamma,k}$
  \STATE $\mathcal{F}\gets \{k\in[K]:\widetilde{V}_{\beta,\gamma,k}\ge \widetilde{V}_{\beta,\gamma,\max}-\epsilon\}$ \COMMENT{$\mathcal{F}\neq\emptyset$ if $\epsilon\ge 0$}
  \STATE $k^\star \gets \arg\min_{k\in\mathcal{F}} \widetilde{\sigma}_k$
  \STATE \RETURN $y_{k^\star}$

\ELSIF{variant $=\textsc{2nd-Moment (LCB)}$}
  \STATE $\lambda_{\delta}\gets c_{\mathrm{LCB}}\sqrt{\frac{\log(K/\delta)}{n}}$
  \COMMENT{or use tuned $\lambda$ directly}
  \FOR{$k=1$ to $K$}
    \STATE $S_{\mathrm{LCB},k}\gets \widetilde{\mu}_k-\lambda_{\delta}\widetilde{\sigma}_k$
    \COMMENT{drop candidate-independent lower-order LCB term}
  \ENDFOR
  \STATE $k^\star \gets \arg\max_{k\in[K]} S_{\mathrm{LCB},k}$
  \STATE \RETURN $y_{k^\star}$
\ENDIF
\end{algorithmic}
\end{algorithm}

\paragraph{Notes.}
Each $\textsc{Score}^{(m)}_\beta$ may be instantiated by multi-annotator samples or proxy perturbation samples under scorer $m$.
Because reward scales differ across scorers, aggregation is performed on selection-time normalized scores (within each prompt),
while all reported metrics in the main tables are computed \emph{within each scorer} on its raw scale (Appendix~\ref{app:metrics}).
When $M=1$, the procedure reduces to the original single-scorer DARC.

\subsection{Inference-time selection baselines.}
All rules below select one output from the same $\cY(s)$:

\begin{itemize}
  \item \textbf{Mean (Best-of-$K$):}
  $
  y^\star=\arg\max_{y\in\cY(s)} \hat\mu_{\text{sel}}(s,y).
  $

  \item \textbf{CVaR (Best-of-$K$):}
  $
  y^\star=\arg\max_{y\in\cY(s)} \widehat{\mathrm{CVaR}}_{\alpha,\text{sel}}(s,y),
  $
  where $\widehat{\mathrm{CVaR}}_{\alpha,\text{sel}}(s,y)$ is the average of the bottom $\alpha$ fraction of satisfaction samples for candidate $y$.
  We fix $\alpha$ throughout (no tuning).

  \item \textbf{MC-Dropout reranking (uncertainty-aware scoring).}
  To obtain an uncertainty estimate without human multi-rater samples, we perform $T$ stochastic forward passes through the scorer (dropout enabled) and treat the resulting scores as an empirical proxy distribution
  $\{r_t(s,y)\}_{t=1}^T$~\citep{gal2016dropout,kendall2017uncertainties}.
  We then compute
  $
  \hat\mu_{\text{MC}}(s,y)=\frac{1}{T}\sum_{t=1}^T r_t(s,y)
  $
  and
  $
  \hat\sigma_{\text{MC}}(s,y)=\sqrt{\frac{1}{T-1}\sum_{t=1}^T\big(r_t(s,y)-\hat\mu_{\text{MC}}(s,y)\big)^2}.
  $
  We use a pessimistic selection rule analogous to MV/LCB:
  \[
  y^\star=\arg\max_{y\in\cY(s)}\Big(\hat\mu_{\text{MC}}(s,y)-\lambda_{\text{MC}}\hat\sigma_{\text{MC}}(s,y)\Big),
  \]
  where $\lambda_{\text{MC}}$ is tuned on the dev set (same protocol as other inference-time baselines).
\item \textbf{DeAL decoding-time alignment (rollout reranking).}
Given a prompt $s$ and a partially generated prefix $x_{<t}$ at decoding step $t$, the base LM defines next-token probabilities $p_\theta(v \mid s, x_{<t})$ over the vocabulary $\cV$.
DeAL performs a one-step lookahead search over the top-$k$ next tokens:
\[
\cV_t^{(k)} := \arg\max_{v \in \cV}^{k}\ \log p_\theta(v \mid s, x_{<t}).
\]
For each candidate token $v \in \cV_t^{(k)}$, we form the candidate prefix $x_{<t}v$ and compute a length-$L$ lookahead continuation using \emph{greedy} rollout:
\[
\tilde{x}_{t:t+L-1}^{(v)} := \mathrm{GreedyRollout}\bigl(\pi_\theta;\ s,\ x_{<t}v,\ L\bigr),
\]
where $\pi_\theta$ is the same LM used as the proposal policy.
We then score the resulting partial completion with a reward model $r_\phi$:
\[
\mathrm{RM}(s, x_{<t}v) := r_\phi\!\left(s,\ x_{<t}v\tilde{x}_{t:t+L-1}^{(v)}\right).
\]
DeAL selects the next token by maximizing a combined objective
\[
v_t \in \arg\max_{v \in \cV_t^{(k)}} 
\Bigl\{\log p_\theta(v \mid s, x_{<t}) + \gamma \cdot \mathrm{RM}(s, x_{<t}v)\Bigr\},
\]
and updates $x_{<t+1} \leftarrow x_{<t}v_t$ until EOS or the maximum length is reached.
This baseline is substantially more expensive than reranking a fixed candidate set, since it requires $k$ rollouts and RM evaluations at each decoding step.
  \item \textbf{Regularized Best-of-$K$ (RBoN).}
  RBoN augments reward-model reranking with a reference-policy regularizer to mitigate reward hacking at inference time~\citep{jinnai2024regularized}.
  Let $\pi_{\mathrm{ref}}$ be a fixed reference policy (we use the generator as $\pi_{\mathrm{ref}}$ unless otherwise stated).
  For each candidate $y\in\cY(s)$, compute the per-token average log-likelihood under $\pi_{\mathrm{ref}}$,
  $
  \bar\ell_{\mathrm{ref}}(s,y)=\frac{1}{|y|}\log \pi_{\mathrm{ref}}(y\mid s)
  $
  (or the unnormalized log-likelihood if length normalization is disabled).
  RBoN selects
  \[
  y^\star=\arg\max_{y\in\cY(s)}\Big(\hat\mu_{\text{sel}}(s,y)+\beta_{\text{RBoN}}\,\bar\ell_{\mathrm{ref}}(s,y)\Big),
  \]
  where $\beta_{\text{RBoN}}\ge 0$ controls the strength of regularization and is tuned on the dev set.
  \item \textbf{Best-of-Poisson (BoP) and HedgeTune.}
  Following \citet{khalaf2025inference}, we hedge greedy best-of-$K$ selection by randomizing the effective candidate budget.
  For each prompt $s$, sample $K'\sim\mathrm{Poisson}(\lambda)$ and truncate to $K'\leftarrow \min(\max(K',1),K)$.
  Let $\mathcal{I}\subseteq\{1,\dots,K\}$ be a uniformly sampled index set of size $|\mathcal{I}|=K'$ and define the corresponding subset
  $\cY_{\mathcal{I}}(s)=\{y_i: i\in\mathcal{I}\}\subseteq \cY(s)$.
  BoP selects
  \[
  y^\star=\arg\max_{y\in \cY_{\mathcal{I}}(s)} \hat\mu_{\text{sel}}(s,y).
  \]
  HedgeTune chooses $\lambda$ using the same dev-set protocol as other inference-time baselines, and fixes it for evaluation.
  \item \textbf{\emph{Caution}: pessimistic best-of-$K$ reranking.}
  Inspired by \citet{mitigatingiclr26}, we penalize atypical candidates using an auxiliary error (or atypicality) signal $e(s,y)\ge 0$,
  yielding a pessimistic score
  \[
  \mathrm{Score}_{\mathrm{caut}}(s,y)
  := \hat\mu_{\text{sel}}(s,y) - \alpha_{\mathrm{caut}}\, e(s,y).
  \]
  We then select
  \[
  y^\star=\arg\max_{y\in\cY(s)} \mathrm{Score}_{\mathrm{caut}}(s,y),
  \]
  where $\alpha_{\mathrm{caut}}\ge 0$ controls the strength of pessimism and is tuned on the dev set.

  \item \textbf{DARC (ours, scorer-robust):}
  define scorer-specific $\widehat{V}^{(m)}_{\beta,\text{sel}}(s,y)$ and aggregate
  \[
  \widetilde{V}_{\beta,\gamma,\text{sel}}(s,y)
  :=
  -\frac{1}{\gamma}\log\!\left(\frac{1}{M}\sum_{m=1}^M \exp\!\big(-\gamma\,\widehat{V}^{(m)}_{\beta,\text{sel}}(s,y)\big)\right),
  \]
  then select
  $
  y^\star=\arg\max_{y\in\cY(s)} \widetilde{V}_{\beta,\gamma,\text{sel}}(s,y).
  $

  \item \textbf{Hard budget (DARC-$\tau$ on scorer-robust entropic risk premium):}
  define scorer-specific risk premia
  $\widehat{\mathrm{RP}}^{(m)}_{\beta,\text{sel}}(s,y)=\hat\mu^{(m)}_{\text{sel}}(s,y)-\widehat{V}^{(m)}_{\beta,\text{sel}}(s,y)$
  and aggregate pessimistically
  $
  \widetilde{\mathrm{RP}}_{\beta,\text{sel}}(s,y):=\max_{m\in[M]}\widehat{\mathrm{RP}}^{(m)}_{\beta,\text{sel}}(s,y).
  $
  Select
  \[
  y^\star=\arg\max_{y\in\cY(s)} \widetilde{V}_{\beta,\gamma,\text{sel}}(s,y)
  \ \text{s.t.}\ 
  \widetilde{\mathrm{RP}}_{\beta,\text{sel}}(s,y)\le \tau;
  \]
  if the feasible set is empty, we fall back to $\cY(s)$.

  \item \textbf{$\epsilon$-tie breaking (DARC-$\epsilon$ on scorer-robust value):}
  define the near-tie set
  $
  \mathcal{F}_\epsilon(s)=\{y:\widetilde{V}_{\beta,\gamma,\text{sel}}(s,y)\ge \widetilde{V}_{\beta,\gamma,\text{sel},\max}(s)-\epstie\},
  $
  and define the worst-case disagreement proxy
  $
  \widetilde{\sigma}_{\text{sel}}(s,y):=\max_{m\in[M]}\hat\sigma^{(m)}_{\text{sel}}(s,y).
  $
  Select
  $
  y^\star=\arg\min_{y\in\mathcal{F}_\epsilon(s)} \widetilde{\sigma}_{\text{sel}}(s,y).
  $

\end{itemize}

\subsection{Training-time robust policies: cDPO and rDPO}
\label{app:train_time_baselines}

\paragraph{Preference data and DPO objective.}
Let $\mathcal{D}=\{(s,y^+,y^-)\}$ be a preference dataset, where for each input $s$ the response $y^+$ is preferred over $y^-$.
Let $\pi_\theta$ be the policy to be trained and $\pi_{\mathrm{ref}}$ a fixed reference policy.
Define the logit difference
\[
\Delta_\theta(s,y^+,y^-)
=
\Big(\log \pi_\theta(y^+|s)-\log \pi_\theta(y^-|s)\Big)
-
\Big(\log \pi_{\mathrm{ref}}(y^+|s)-\log \pi_{\mathrm{ref}}(y^-|s)\Big).
\]
The standard DPO loss is
\[
\mathcal{L}_{\mathrm{DPO}}(\theta)
=
\mathbb{E}_{(s,y^+,y^-)\sim\mathcal{D}}
\Big[
-\log \sigma\big(\beta\,\Delta_\theta(s,y^+,y^-)\big)
\Big],
\]
where $\sigma(\cdot)$ is the logistic sigmoid and $\beta>0$ controls the strength of preference optimization.

\paragraph{Conservative DPO (cDPO).}
We implement a conservative variant by applying label smoothing to the pairwise preference target.
Specifically, for a smoothing parameter $\varepsilon\in(0,1/2)$, we use the smoothed loss
\[
\mathcal{L}_{\mathrm{cDPO}}(\theta)
=
\mathbb{E}_{(s,y^+,y^-)\sim\mathcal{D}}
\Big[
-(1-\varepsilon)\log \sigma(\beta\,\Delta_\theta)
-\varepsilon \log \big(1-\sigma(\beta\,\Delta_\theta)\big)
\Big].
\]
We denote the resulting trained policy as $\pi_{\mathrm{cDPO}}$.

\paragraph{Robust DPO (rDPO).}
Preference data may contain noise or ambiguity. We consider a robust variant that optimizes against uncertainty in preference labels.
Let $q\in[0,1]$ denote an (unknown) probability that the pair $(y^+,y^-)$ is correctly labeled for input $s$.
We model robustness by optimizing the worst-case (or uncertainty-aware) expected loss over an admissible set $\mathcal{Q}$:
\[
\mathcal{L}_{\mathrm{rDPO}}(\theta)
=
\mathbb{E}_{(s,y^+,y^-)\sim\mathcal{D}}
\Big[
\sup_{q\in\mathcal{Q}}
\big(
-q\log \sigma(\beta\,\Delta_\theta)
-(1-q)\log (1-\sigma(\beta\,\Delta_\theta))
\big)
\Big].
\]
Here $\mathcal{Q}$ specifies the assumed noise level (e.g., $q\in[1-\rho,1]$ for some $\rho\in[0,1/2)$).
We denote the resulting trained policy as $\pi_{\mathrm{rDPO}}$.

\paragraph{Evaluation protocol (policy + inference-time selection).}
For any trained policy $\pi$ (including $\pi_{\mathrm{cDPO}}$ and $\pi_{\mathrm{rDPO}}$), at inference time we sample a candidate set
$\cY(s)=\{y_1,\dots,y_K\}\sim \pi(\cdot|s)$ and then apply the same selection rules defined in Appendix~\ref{app:decoding_rules},
e.g., Mean (Best-of-$K$) or DARC-$\epsilon$.
This isolates the effect of \emph{training-time robustness} (changing $\pi$) from \emph{inference-time robustness} (changing the selection rule).

\section{Related Work}
\label{app:related}
\subsection{Preference optimization for alignment}
Preference-based alignment is the prevailing paradigm for steering LLMs, ranging from RLHF-style reward modeling with policy optimization to direct preference objectives that learn from comparisons without explicit RL loops~\citep{ouyang2022training,christiano2017deep,rafailov2023direct}. 
Recent work largely focuses on making preference optimization simpler and more stable, e.g., via single-stage or reference-light formulations that resemble standard fine-tuning~\citep{hong2024orpo,meng2024simpo,guo2024controllable}. 
A second theme highlights the role of \emph{data coverage}: when preference data poorly covers the model’s behavior distribution, offline updates can degrade, motivating unified online--offline perspectives and value-aware objectives~\citep{song2024importance,cen2024value}. 
Meanwhile, newer datasets enrich supervision beyond pairwise comparisons with multi-attribute ratings and complementary preference signals, enabling finer-grained diagnosis and reward learning~\citep{wang2024helpsteer,wang2024helpsteer2}. 
Despite these advances, most methods still collapse alignment into optimizing a single scalar proxy of utility (often motivated by Bradley--Terry style models~\citep{bradley1952rank}), which can mask structured disagreement and sensitivity to annotator mix.

\subsection{Heterogeneity, robustness, and inference-time alignment}
A growing body of work shows that preferences are intrinsically heterogeneous: annotators disagree systematically, and average-score alignment can hide persistent failures for sub-populations~\citep{zhang2024diverging,chen2024pal,casper2023open}. 
This is compounded by proxy over-optimization, where scaling up optimization against an imperfect reward/preference proxy can degrade the underlying target for both RLHF and direct alignment~\citep{gao2023scaling,rafailov2024scaling}. 
Robustness-oriented approaches therefore move beyond mean objectives, including distributionally robust and group-robust formulations that protect against noise and minority-group degradation~\citep{wu2024towards,ramesh2024group,chakraborty2024maxmin}. 
Another complementary direction makes reward signals multi-dimensional via multi-objective reward modeling or multi-head aggregation, improving interpretability and offering more conservative combinations, but typically requiring richer supervision and an explicit (often fixed) scalarization for deployment~\citep{wang2024interpretable,li2025multi,yang2024metaaligner}. 
Finally, inference-time alignment methods avoid retraining by allocating extra computation at decoding to select better candidates, though recent analyses caution that aggressive best-of-$N$ selection can amplify proxy error under miscalibration or under-coverage~\citep{sun2024fast,huang2025best,ichihara2025evaluation,huang2025deal}. 
Standardized reward-model benchmarks further support systematic study of these issues across settings, including retrieval-augmented generation~\citep{lambert2025rewardbench,zhou2024rmb,jin2025rag,liu2024rm}.

\paragraph{Distributionally robust optimization and variance regularization.}
Distributionally robust optimization (DRO) studies objectives of the form $\min_{\theta}\sup_{\mathbb{Q}\in\mathcal{U}(\widehat{\mathbb{P}})} \mathbb{E}_{\mathbb{Q}}[\ell(\theta;Z)]$,
where the true distribution is assumed to lie in an ambiguity set around the empirical distribution~\citep{wiesemann2014distributionally,rahimian2019distributionally}.
Local $f$-divergence neighborhoods yield tractable robust objectives and have been widely used in learning and optimization~\citep{namkoong2016stochastic}.
In particular, $\chi^2$-DRO is closely connected to variance-based regularization and generalized empirical likelihood perspectives~\citep{duchi2019variance,duchi2021statistics}.
Our work brings this DRO view to \emph{inference-time decoding} under heterogeneous human preferences:
the same mean--dispersion score used by DARC arises as the closed-form robust value of a $\chi^2$-ambiguity set.
DRO is also motivated by heterogeneous subpopulations and uniform-performance desiderata~\citep{duchi2021learning},
and efficient large-scale methods exist for common DRO formulations~\citep{levy2020large}.

\section{Experiments}

\subsection{Human disagreement on the raw Top-$K$ candidate pool}
\label{app:human-disagreement-topk}

Even before applying any reranking, the raw Top-$K$ candidate pool can exhibit substantial human disagreement.
We quantify this on \textsc{AlpacaEval~2.0} and \textsc{MT-Bench} using the same pool size $K{=}5$.

\paragraph{Setup.}
For each prompt $s$, we generate $\cY(s)=\{y_1,\dots,y_K\}$ and collect $n{=}5$ human ratings per candidate on a $0$--$10$ scale
(see Appendix~\ref{app:human-eval}).
For each candidate, we compute the sample mean and variance across raters:
\[
\hat\mu_{\mathrm{human}}(s,y_k) := \frac{1}{n}\sum_{j=1}^n h_j(s,y_k),
\qquad
\hat\sigma^2_{\mathrm{human}}(s,y_k) := \frac{1}{n-1}\sum_{j=1}^n \bigl(h_j(s,y_k)-\hat\mu_{\mathrm{human}}(s,y_k)\bigr)^2.
\]
We summarize per-prompt disagreement by the maximum variance over the pool,
$
D(s) := \max_{k\in[K]} \hat\sigma^2_{\mathrm{human}}(s,y_k),
$
which captures whether at least one candidate is highly contentious.

\paragraph{High-disagreement subset (top 20\%).}
We define the high-disagreement subset as the top 20\% prompts by $D(s)$:
\[
\cS_{\mathrm{high}}
:= \left\{s\in\cS:\; D(s)\ge Q_{0.8}(D)\right\},
\]
where $Q_{0.8}(D)$ denotes the $80$-th percentile of $\{D(s):s\in\cS\}$ computed on the evaluation set.
We report main metrics both on the full set and on $\cS_{\mathrm{high}}$.

\begin{table}[t]
\centering
\small
\begin{tabular}{lcccc}
\toprule
Dataset
& $\mathbb{E}_s[D(s)]$
& $Q_{0.8}(D)$
& $Q_{0.9}(D)$
& $\Pr[D(s)\ge Q_{0.8}(D)]$ \\
\midrule
AlpacaEval~2.0 & 0.74 & 0.98 & 1.65 & 0.20 \\
MT-Bench       & 0.65 & 0.84 & 1.32 & 0.20 \\
\bottomrule
\end{tabular}
\caption{Human disagreement on the raw Top-$K$ pool ($K{=}5$) measured by the maximum across-candidate human rating variance $D(s)$.
High-disagreement prompts are defined as the top 20\% by $D(s)$.}
\label{tab:topk-human-disagreement}
\end{table}

\subsection{Hyperparameter selection}
\label{app:hparams}

\paragraph{Overview.}
We report hyperparameters for (i) training-time robust policies (cDPO/rDPO) and (ii) inference-time selection methods and proxy-risk estimation, including DARC and inference-time baselines (DeAL~\citep{huang2025deal}, MC-Dropout~\citep{gal2016dropout}, RBoN~\citep{jinnai2024regularized}, Best-of-Poisson/HedgeTune~\citep{khalaf2025inference}, and \emph{Caution}~\citep{mitigatingiclr26}).

Unless otherwise stated, we use a single fixed set of hyperparameters for all reported results, with no dataset- or split-specific tuning.
All inference-time baselines share the same candidate pools and reward-model preprocessing/truncation for a fair comparison.

\paragraph{Base model and parameter-efficient fine-tuning (cDPO/rDPO).}
All policy and reference models are initialized from \texttt{meta-llama/Llama-3.1-8B-Instruct}.
We fine-tune with QLoRA using 4-bit NF4 quantization with double quantization enabled and bf16 compute.
We train LoRA adapters with rank $r{=}16$, scaling $\alpha{=}32$, dropout $0.05$, and no bias parameters.
We enable gradient checkpointing and disable KV caching during training.

\paragraph{cDPO training.}
We optimize the standard DPO objective with inverse-temperature $\beta_{\text{DPO}}{=}0.1$ and label smoothing $\varepsilon_{\mathrm{ls}}{=}0.10$.
Unless stated otherwise, we use learning rate $10^{-5}$, per-device batch size $1$, gradient accumulation $16$ (effective batch size $16$), and train for $5$ epochs in bf16.
We cap the maximum prompt length at $512$ and the maximum response length at $512$ (total length $\le 1024$ tokens), and set the random seed to $7$.

\paragraph{rDPO training.}
We use the same initialization, QLoRA/LoRA configuration, and optimization settings as cDPO, but optimize the rDPO objective with label-flip noise parameter $\epsilon_{\mathrm{flip}}=0.1$ and inverse-temperature $\beta_{\text{DPO}}{=}0.1$.
The reference model is kept frozen throughout training.

\paragraph{Inference-time candidate pools.}
For automated proxy evaluation and inference-time reranking, we generate a fixed candidate pool of size $K{=}16$ per prompt using nucleus sampling with top-$p{=}0.98$ and temperature $0.8$, with at most $320$ newly generated tokens.

\paragraph{Reward model and proxy-risk statistics.}
We score each candidate with the reward model \texttt{Skywork/Skywork-Reward-Llama-3.1-8B-v0.2} using maximum input length $1024$ and batch size $16$.
To estimate proxy disagreement, we construct $N_{\text{aug}}{=}8$ style-preserving perturbations per candidate and define the sample set
$\mathcal{S}=\{r_{\mathrm{orig}}\}\cup\{r^{(j)}_{\mathrm{aug}}\}_{j=1}^{N_{\text{aug}}}$,
where each element is the reward model score on the corresponding prompt--response formatting.
We then compute the empirical mean $\hat\mu$, standard deviation $\hat\sigma$, and tail metric $\mathrm{CVaR}_{0.1}$ as the average of the lowest $10\%$ of samples in $\mathcal{S}$.

\paragraph{DARC (entropic) decoding hyperparameters.}
Our primary decoding rule selects the candidate maximizing the empirical entropic robust value
$\widehat V_\beta = -\frac{1}{\beta}\log\!\big(\frac{1}{|\mathcal{S}|}\sum_{r\in \mathcal{S}}\exp(-\beta r)\big)$,
with entropic temperature $\beta{=}1.0$.
We also report two deployment-friendly variants:
(i) \textbf{DARC-$\tau$}, which constrains the entropic risk premium $\RPhat=\hat\mu-\widehat V_\beta$ by a budget $\tau$; we set $\tau$ to the $q_{\mathrm{RP}}{=}0.25$ quantile of $\{\RPhat(s,y)\}_{y\in\cY(s)}$ on each candidate pool and select the feasible candidate with the largest $\hat\mu$;
(ii) \textbf{DARC-$\epsilon$}, which forms a near-optimal set in robust value
$\{y:\widehat V_\beta(y)\ge \max_{y'}\widehat V_\beta(y')-\epstie_V\}$ with $\epsilon_V{=}0.25$
and selects the candidate with the smallest $\hat\sigma$ (tie-breaking by larger $\hat\mu$).
Unless otherwise noted, the mainline entropic variant uses the DARC-$\epsilon$ selection rule.

\paragraph{DeAL hyperparameters.}
We use a lightweight and stable DeAL configuration:
$k=8,\ L=16,\ \gamma=0.05,$ and cap the maximum generation length to $T_{\max}=512$ new tokens.
Lookahead rollouts are computed with greedy decoding (deterministic rollout; no sampling).
For reward-model inference within DeAL lookahead, we use maximum RM input length $256$ and batch size $16$.

\paragraph{MC-Dropout uncertainty reranking.}
For the MC-Dropout baseline, we enable dropout in the reward model at inference time and draw $M{=}8$ stochastic forward passes per (prompt, candidate), yielding reward samples $\{r^{(m)}\}_{m=1}^{M}$.
We compute $\hat\mu_{\mathrm{MC}}$, $\hat\sigma_{\mathrm{MC}}$, $\mathrm{CVaR}_{0.1}$, and $\widehat V_\beta$ on these samples with $\beta{=}1.0$ and use the uncertainty-penalized score $\hat\mu_{\mathrm{MC}}-\alpha_{\mathrm{MC}}\hat\sigma_{\mathrm{MC}}$ with $\alpha_{\mathrm{MC}}{=}1.0$.

\paragraph{RBoN (regularized Best-of-$K$) decoding-time reranking.}
For the RBoN baseline, we rerank the same candidate pool ($K{=}16$) using
\[
\mathrm{Score}_{\mathrm{RBoN}}(s,y)
= r_\phi(s,y) + \beta_{\mathrm{RBoN}}\,\overline{\log p_{\mathrm{ref}}}(y\mid s),
\]
where $r_\phi$ is the reward-model score and $\overline{\log p_{\mathrm{ref}}}$ denotes the per-token mean log-likelihood under a reference policy $p_{\mathrm{ref}}$ (we use the generator policy) to reduce length bias.
We use $\beta_{\mathrm{RBoN}}{=}0.02$ and compute log-likelihoods with the same maximum length cap (1024 tokens) as reward scoring; due to memory constraints we use a small batch size for log-prob evaluation.
\paragraph{Best-of-Poisson (BoP) and HedgeTune.}
Following \citet{khalaf2025inference}, we implement Best-of-Poisson selection on the same fixed candidate pool ($K{=}16$).
For each prompt, we sample $K'\sim\mathrm{Poisson}(\lambda)$ and truncate to $K'\leftarrow \min(\max(K',1),K)$.
We then select the candidate with the largest reward-model score among a uniformly sampled subset of size $K'$ from the pool
(to avoid positional bias).
As a default, we use $\lambda{=}12$ and report a global tuning range $\lambda\in\{4,8,12,16\}$.
For HedgeTune, we choose $\lambda$ once on a small held-out calibration split (shared across all datasets) and fix it for all reported results,
in line with our no dataset- or split-specific tuning protocol.
\paragraph{\emph{Caution}: pessimistic best-of-$N$ reranking.}
We implement the pessimistic reranking baseline of \citet{mitigatingiclr26} on the same candidate pool ($K{=}16$).
In addition to the reward score $r_\phi(s,y)$, we fit an auxiliary error/atypicality model on in-distribution data and compute an
atypicality penalty $e(s,y)$.
Candidates are ranked by the pessimistic score
$
\mathrm{Score}_{\mathrm{Caution}}(s,y)= r_\phi(s,y)-\alpha_{\mathrm{caut}}\, e(s,y).
$
We use $\alpha_{\mathrm{caut}}{=}1.0$ by default.
The auxiliary model is trained once and reused across all datasets and splits.
\paragraph{Auxiliary error/atypicality model $e(s,y)$.}
We implement $e(s,y)$ as a supervised classifier over prompt--response pairs.
The input is the concatenation \texttt{[PROMPT] $s$ [SEP] [RESPONSE] $y$}, and the output is a scalar
$e(s,y)\in[0,1]$ interpreted as the probability that $(s,y)$ is \emph{atypical} under the in-distribution training data.
We use a pretrained Transformer encoder (same backbone family as the reward model for tokenization compatibility) with a single linear head.
Training data is constructed from the in-distribution training split only:
\begin{itemize}[leftmargin=*]
\item \textbf{Typical pairs.} Ground-truth prompt--response pairs from the training split.
\item \textbf{Atypical pairs.} (i) \emph{Mismatched} pairs formed by pairing each prompt with a response sampled from a different prompt,
and (ii) \emph{corrupted} responses obtained by applying lightweight perturbations to a typical response (random truncation, sentence order
shuffling, and insertion of an unrelated sentence sampled from another example).
\end{itemize}
We train with binary cross-entropy, using a fixed 1:1 ratio of typical to atypical pairs, early stopping on a held-out validation subset,
and a fixed hyperparameter setting (learning rate, batch size, epochs) shared across all datasets and splits.
At inference time, we compute $e(s,y)$ for each candidate in the pool and apply the above pessimistic reranking.

\paragraph{Subset definitions and diagnostics.}
For stratified robustness reporting, we define the high-disagreement subset as the top $20\%$ prompts by the baseline (mean Best-of-$K$) proxy disagreement $\hat\sigma$.
When reporting second-moment baselines or legacy ablations, we additionally include the mean--dispersion score $\hat\mu-\lambda\hat\sigma$ with $\lambda_{\mathrm{raw}}{=}5.0$ (and a $z$-scored variant with $\lambda_z{=}3.0$), used only for diagnostic comparisons.
\paragraph{Reward model score truncation.}
For proxy satisfaction samples, we use the reward model's scalar outputs (logits).
To align with Assumption~3.2 and avoid undue influence from rare outliers, we truncate logits to $[-L,L]$
(we use $L=10$) before computing $\hat\mu_n(s,y)$, $\hat\sigma_n(s,y)$, and $\widehat V_\beta(s,y)$.
We choose $L$ large enough that truncation is rarely active in practice.
For completeness, Appendix~\ref{app:tail} states an analogous lower-tail interpretation under sub-Gaussianity.

\begin{table}[t]
\centering
\small
\setlength{\tabcolsep}{5pt}
\resizebox{\linewidth}{!}{%
\begin{tabular}{l l}
\toprule
\textbf{Component} & \textbf{Hyperparameters (default)} \\
\midrule
Base model & \texttt{meta-llama/Llama-3.1-8B-Instruct} \\
Candidate generation & $K{=}16$, temperature $0.8$, top-$p{=}0.98$, max\_new\_tokens $320$ \\
Reward model scoring & \texttt{Skywork/Skywork-Reward-Llama-3.1-8B-v0.2}\citep{liu2024skywork}, max length $1024$, batch size $16$ \\
Proxy disagreement & $N_{\text{aug}}{=}8$ perturbations; samples $\mathcal{S}=\{r_{\mathrm{orig}}\}\cup\{r_{\mathrm{aug}}^{(j)}\}_{j=1}^{N_{\text{aug}}}.
$; $\hat\sigma=\mathrm{Std}(\mathcal{S})$ \\
Proxy tail metric & $\mathrm{CVaR}_{0.1}$ on $\mathcal{S}$ \\
DARC entropic & $\beta{=}1.0$ (robust value), $q_{\mathrm{RP}}{=}0.25$ (DARC-$\tau$), $\epsilon_V{=}0.25$ (DARC-$\epsilon$) \\
High-disagreement subset & top $20\%$ prompts by baseline $\hat\sigma$ \\
\midrule
cDPO training & $\beta_{\text{DPO}}{=}0.1$, $\varepsilon_{\mathrm{ls}}{=}0.10$, lr $10^{-5}$, epochs $5$, bs $1$, gas $16$ \\
rDPO training & $\beta_{\text{DPO}}{=}0.1$, noise $\epsilon{=}0.10$, lr $10^{-5}$, epochs $5$, bs $1$, gas $16$ \\
LoRA / QLoRA & 4-bit NF4 + double quant + bf16 compute; LoRA $r{=}16$, $\alpha{=}32$, dropout $0.05$ \\
\bottomrule
\end{tabular}
}
\caption{Default hyperparameters used in cDPO/rDPO training and inference-time evaluation.}
\label{tab:hparams_defaults}
\end{table}

\subsection{Complete metric definitions and computation}
\label{app:metrics}

\paragraph{Notation.}
Let $s$ be a prompt and $y$ a candidate response.
In proxy settings we may use a family of scorers (reward models) $\{R_m\}_{m=1}^M$.
We write $r^{(m)}_i(s,y)$ for scalar reward samples returned by scorer $m$.
We use $\mathcal{I}_{\text{sel}}$ for selection-time samples and $\mathcal{I}_{\text{eval}}$ for held-out evaluation samples
(disjoint to avoid selection--evaluation leakage).
Unless stated otherwise, all reported metrics are computed \emph{within a fixed scorer} $m$ on its raw reward scale.

\paragraph{Prompt-level evaluation mean.}
For each $(s,y)$, let $\{r_i(s,y)\}_{i\in\mathcal{I}_{\text{eval}}}$ be held-out reward samples.
We define
\begin{equation}
\hat\mu_{\text{eval}}(s,y) \;:=\; \frac{1}{|\mathcal{I}_{\text{eval}}|}\sum_{i\in\mathcal{I}_{\text{eval}}} r_i(s,y).
\end{equation}

\paragraph{Disagreement / risk proxy via perturbation sensitivity.}
For each response $y$, generate $N_{\text{aug}}$ style-preserving perturbations $\{\tilde y_j\}_{j=1}^{N_{\text{aug}}}$ with $\tilde y_1:=y$.
Under scorer $m$, let $u^{(m)}_j(s,y):=R_m(s,\tilde y_j)$ be the reward of perturbation $j$.
Define
\begin{align}
\hat\mu^{(m)}_{\text{aug}}(s,y) &:= \frac{1}{N_{\text{aug}}}\sum_{j=1}^{N_{\text{aug}}} u^{(m)}_j(s,y),\\
\hat\sigma^{(m)}(s,y) &:= \sqrt{\frac{1}{N_{\text{aug}}-1}\sum_{j=1}^{N_{\text{aug}}}\big(u^{(m)}_j(s,y)-\hat\mu^{(m)}_{\text{aug}}(s,y)\big)^2}.
\end{align}
When a scorer-robust selection rule is used, we aggregate disagreement pessimistically at selection time via
$\widetilde{\sigma}_{\text{sel}}(s,y):=\max_{m\in[M]}\hat\sigma^{(m)}_{\text{sel}}(s,y)$.

\paragraph{risk--reward tradeoff score (Tradeoff).}
Within scorer $m$, we define
\begin{equation}
\mathrm{Tradeoff}^{(m)}_{\text{eval}}(s,y)
\;:=\;
\hat\mu^{(m)}_{\text{eval}}(s,y) - \lambda\,\hat\sigma^{(m)}(s,y),
\end{equation}
where $\lambda$ is fixed across methods.
For scorer-robust methods, selection uses the aggregated statistics (e.g., $\widetilde{V}_{\beta,\gamma,\text{sel}}$ and $\widetilde{\sigma}_{\text{sel}}$),
but reporting is performed within each scorer $m$ to avoid mixing incomparable reward scales.

\paragraph{Tail risk: $\mathrm{CVaR}_{10\%}$ over prompts.}
For a method that outputs one final response $\hat \cY(s)$ per prompt $s$, define prompt-level outcome
\begin{equation}
z(s) \;:=\; \hat\mu_{\text{eval}}(s,\hat \cY(s)) \quad \text{(or } z(s):=\mathrm{LCB}_{\text{eval}}(s,\hat \cY(s))\text{, specified per table).}
\end{equation}
Let $\{z(s_k)\}_{k=1}^{N_{\text{pr}}}$ be the $N_{\text{pr}}$ prompt outcomes and let $z_{(1)}\le \cdots \le z_{(M)}$ be the sorted values.
We define
\begin{equation}
\mathrm{CVaR}_{10\%} := \frac{1}{\lceil 0.1N_{\text{pr}}\rceil}\sum_{k=1}^{\lceil 0.1N_{\text{pr}}\rceil} z_{(k)}.
\end{equation}

\paragraph{High-variance (HV) subset evaluation.}
We define a high-variance prompt subset by ranking prompts using the baseline method's risk proxy (e.g., $\hat\sigma(s,\hat y_{\text{base}}(s))$) and taking the top $p\%$ prompts (we use $p=20$ unless otherwise stated).
All HV metrics are computed on this fixed subset for every method.

\paragraph{Win/Tie/Loss (W/T/L) across scorers.}
Given two methods $A$ and $B$, for each prompt $s$ define the score difference under a scorer $R$:
\begin{equation}
\Delta_R(s) := z_A^R(s) - z_B^R(s).
\end{equation}
We count a tie if $|\Delta_R(s)|<\varepsilon$ (we use $\varepsilon=\text{0.5}$; set $\varepsilon=0$ if not used), otherwise win/loss by the sign of $\Delta_R(s)$.
We report W/T/L rates over prompts.

\paragraph{Correlation and top-$q\%$ overlap.}
To assess agreement between two scoring functions $R_1$ and $R_2$, we compute Spearman's rank correlation $\rho$ (and optionally Kendall's $\tau$) between the per-prompt outcomes $\{z^{R_1}(s)\}$ and $\{z^{R_2}(s)\}$.
We also report top-$q\%$ overlap: let $\mathcal{T}_{R}$ be the set of prompts in the top $q\%$ according to $z^{R}(s)$; the overlap is
\begin{equation}
\mathrm{Overlap}_{q\%}(R_1,R_2) := \frac{|\mathcal{T}_{R_1}\cap \mathcal{T}_{R_2}|}{|\mathcal{T}_{R_1}|}.
\end{equation}

\paragraph{Metric hyperparameters.}
Unless otherwise stated, we use $N_{\text{aug}}=8$ perturbations per response for estimating robustness statistics, where we include the original response as one sample and draw $N_{\text{aug}}-1$ style-preserving perturbations.\footnote{Equivalently, we compute statistics over the multiset $\{\tilde y_j\}_{j=1}^{N_{\text{aug}}}$ with $\tilde y_1 := y$.}
We report the high-disagreement (HV) subset using the top $p=20\%$ prompts ranked by the baseline scorer's disagreement proxy, and compute Top-$q\%$ Overlap with $q=20\%$.
For the tradeoff metric, we set a fixed weight $\lambda=1.99$ across methods and report
\begin{equation}
\mathrm{Tradeoff}_{\text{eval}}(s,y)
:= \hat\mu_{\text{eval}}(s,y) - \lambda\,\hat\sigma(s,y),
\end{equation}
where $\hat\mu_{\text{eval}}$ is the empirical mean RM score (computed on the evaluation split) and $\hat\sigma$ is the disagreement proxy estimated from the perturbation set.
For W/T/L, we treat the RM scores as continuous and set the tie threshold to $\varepsilon_{\text{WTL}}=0$, i.e.,
\begin{equation}
\Delta_R(s) := \hat\mu_{\text{eval}}(s,\hat y) - \hat\mu_{\text{eval}}(s,y_{\text{base}}),
\qquad
\mathrm{WTL}(s)=
\begin{cases}
\text{W}, & \Delta_R(s) > 0,\\
\text{T}, & |\Delta_R(s)| \le 0,\\
\text{L}, & \Delta_R(s) < 0.
\end{cases}
\end{equation}

\begin{figure}
    \centering
    \includegraphics[width=0.5\linewidth]{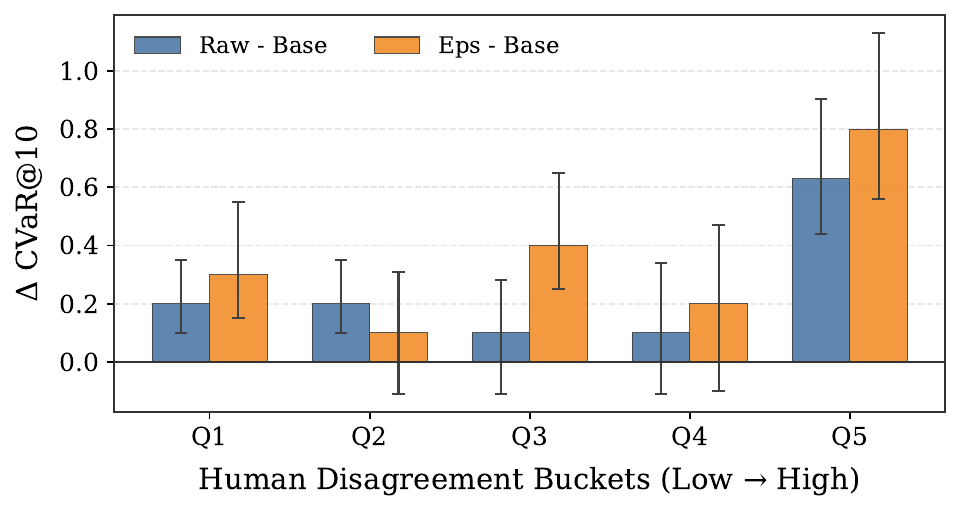}
\caption{\textbf{Conservative metric exhibits the same bucketed trend.}
Bucketed improvements (vs.\ base) for a conservative cvar-style score (e.g., $\Delta \mathrm{CVaR}_{10}$),
, using the same human-disagreement buckets as Fig.~\ref{fig:bucket_dcvar}.
Bars show mean; error bars denote 95\% CIs.}

    \label{fig:bucket_dcvar}
\end{figure}
Consistent with the Tradeoff results, the conservative cvar-style metric in Fig.~\ref{fig:bucket_dcvar} shows the largest gains in the highest-disagreement bucket, indicating that our improvements are not limited to mean reward but extend to pessimistic/robust criteria that emphasize reliability under preference heterogeneity.

\subsection{Selecting risk controls and practical calibration}
\label{app:knob_selection}

\paragraph{Overview.}
We select $(\beta,\tau,\lambda,\epsilon)$ on a held-out development set, fixing the candidate set size $K$ and (when used) the scorer family and normalization protocol (Appendix~\ref{app:multi_scorer}).

\paragraph{Tuning $\beta$.}
We tune $\beta$ via a grid search on the development set to trade off mean quality and robustness metrics (disagreement / tail risk). Unless otherwise specified, we use a single $\beta$ shared across prompts.

\paragraph{Constraint threshold $\tau$ via quantiles.}
For the constrained decoder \eqref{eq:entropic-constraint}, we set $\tau$ using a quantile rule on the empirical distribution of $\widetilde{\mathrm{RP}}_{\beta}(s,\cdot)$ over candidates:
\[
\tau := \mathrm{Quantile}_q\big(\{\widetilde{\mathrm{RP}}_{\beta}(s,y):y\in\cY(s)\}\big),
\]
with $q\in(0,1)$ chosen on the development set.
This heuristic keeps the feasible set $\mathcal{F}_\tau(s)=\{y:\widetilde{\mathrm{RP}}_{\beta}(s,y)\le\tau\}$ nonempty and controls its typical size across prompts.

\paragraph{Penalty coefficient $\lambda$.}
For the penalized form \eqref{eq:entropic-penalty}, we tune $\lambda$ on the development set.
A finite-candidate-set relationship between $(\tau,\lambda)$ is provided in Appendix~\ref{app:lagrangian}; in practice, we treat $\lambda$ as a direct deployment knob.

\paragraph{Choosing $\epsilon$ for near-tie breaking.}
We select $\epsilon$ on the development set to trade off robust value and disagreement in \eqref{eq:eps-constrained-entropic}, and report sensitivity to $\epsilon$ in the ablations.

\paragraph{Multi-scorer setting.}
When using $M>1$ scorers, selection operates on the aggregated statistics $\widetilde V_{\beta,\gamma}(s,y)$ and $\widetilde{\mathrm{RP}}_{\beta}(s,y)$ defined in \eqref{eq:scorer-robust-V} and Appendix~\ref{app:multi_scorer}. When $M=1$, these reduce to the single-scorer quantities.

\subsection{Scalarization robustness: absolute ratings vs.\ pairwise preferences}
\label{app:scalarization_robust}

To bridge our scalar-sample analysis with the common pairwise-preference setting, we verify that our main conclusions are stable under two alternative scalarizations of the same human rating data.

\paragraph{Setup.}
For each prompt $s$, we compare a method-selected response $y_m(s)$ against the baseline-selected response $y_{\text{base}}(s)$.
We collect $n$ independent human ratings $\{r_i(s,y)\}_{i=1}^n$ on a bounded scale (e.g., 1--10) for each evaluated response $y$.
Unless stated otherwise, we analyze prompt-level differences and then aggregate over prompts.

\paragraph{S0: absolute-rating scalarization.}
We treat the raw rating as a scalar satisfaction sample,
\begin{equation}
R^{\text{S0}}_i(s,m) := r_i\!\left(s, y_m(s)\right),
\end{equation}
and define the prompt-level mean improvement over the baseline as
\begin{equation}
\Delta^{\text{S0}}(s,m)
:= \frac{1}{n}\sum_{i=1}^n \Big(r_i(s,y_m(s)) - r_i(s,y_{\text{base}}(s))\Big).
\label{eq:delta-s0}
\end{equation}

\paragraph{S1: pairwise (win/tie/loss) scalarization against baseline.}
We convert the same ratings into pairwise outcomes against the baseline per annotator:
\begin{equation}
W_i(s,m)
:= \mathbb{I}\!\left[r_i(s,y_m(s)) > r_i(s,y_{\text{base}}(s))\right]
+ \tfrac{1}{2}\,\mathbb{I}\!\left[r_i(s,y_m(s)) = r_i(s,y_{\text{base}}(s))\right]
\in\{0, \tfrac{1}{2}, 1\}.
\label{eq:win-indicator}
\end{equation}
We then define the prompt-level improvement relative to a $0.5$ tie baseline:
\begin{equation}
\Delta^{\text{S1}}(s,m)
:= \frac{1}{n}\sum_{i=1}^n W_i(s,m) - \tfrac{1}{2}.
\label{eq:delta-s1}
\end{equation}
Equivalently, $\Delta^{\text{S1}}(s,m)$ measures how much the win-rate against the baseline exceeds $50\%$ (with ties counted as half wins).

\paragraph{Win/Tie/Loss aggregation.}
For each scalarization $\Delta\in\{\Delta^{\text{S0}},\Delta^{\text{S1}}\}$, we report prompt-level Win/Tie/Loss counts:
\begin{equation}
\mathrm{W/T/L}(m)
=
\Big|\{s:\Delta(s,m)>0\}\Big|
\; \Big/ \;
\Big|\{s:\Delta(s,m)=0\}\Big|
\; \Big/ \;
\Big|\{s:\Delta(s,m)<0\}\Big|.
\end{equation}
We also report the mean improvement $\mathbb{E}_s[\Delta(s,m)]$ over prompts.

\paragraph{High-disagreement subset.}
To analyze where gains concentrate, we define prompt-level human disagreement on the baseline response by
\begin{equation}
\hat\sigma_{\text{base}}(s)
:= \mathrm{Std}\!\left(\,r_1(s,y_{\text{base}}(s)),\ldots,r_n(s,y_{\text{base}}(s))\,\right),
\end{equation}
and select the top $20\%$ prompts by $\hat\sigma_{\text{base}}(s)$ as the high-disagreement subset.

\paragraph{Prompt-level stability across scalarizations.}
Finally, we quantify whether the two scalarizations yield consistent prompt-wise improvements by computing the Spearman rank correlation
\begin{equation}
\rho_m
:= \mathrm{Spearman}\!\left(\{\Delta^{\text{S0}}(s,m)\}_s,\{\Delta^{\text{S1}}(s,m)\}_s\right).
\end{equation}
A positive and substantial $\rho_m$ indicates that the prompts benefiting from risk-aware selection are largely consistent across absolute-rating and pairwise scalarizations.
Finally, we empirically verify that the key selection trends remain stable across alternative scalarizations of pairwise data (e.g., win-rate vs.\ fitted scores); see Appendix~\ref{app:pairwise_bridge}.
\begin{table}[t]
\centering
\small
\setlength{\tabcolsep}{5.5pt}
\begin{tabular}{lccc}
\toprule
\textbf{Method} & \textbf{Scalarization} & \textbf{Win/Tie/Loss} & \textbf{Mean $\Delta$} \\
\midrule
\multicolumn{4}{l}{\textbf{Overall (valid prompts)}} \\
\midrule
raw & S0 (absolute ratings) & 247 / 179 / 74 & +0.096 \\
raw & S1 (pairwise vs.\ base) & 241 / 198 / 61 & +0.089 \\
eps & S0 (absolute ratings) & 259 / 166 / 75 & +0.130 \\
eps & S1 (pairwise vs.\ base) & 249 / 173 / 78 & +0.102 \\
\midrule
\multicolumn{4}{l}{\textbf{High-disagreement subset (top 20\% by base human $\sigma$)}} \\
\midrule
raw & S0 (absolute ratings) & 25 / 64 / 11 & +0.269 \\
raw & S1 (pairwise vs.\ base) & 23 / 60 / 17 & +0.214 \\
eps & S0 (absolute ratings) & 36 / 53 / 11 & +0.347 \\
eps & S1 (pairwise vs.\ base) & 34 / 50 / 16 & +0.321 \\
\bottomrule
\end{tabular}
\caption{\textbf{Scalarization robustness (absolute vs.\ pairwise).}
S0 uses raw scalar satisfaction ratings. S1 converts ratings into pairwise preferences against the base response
(win = 1, tie = 0.5, loss = 0) per annotator, then averages across annotators.
Trends are consistent under both scalarizations, especially on high-disagreement prompts.
(Additional prompt-level stability: Spearman correlation of per-prompt $\Delta$ between S0 and S1 is
0.595/0.665 overall and 0.688/0.764 on the high-disagreement subset for raw/eps, respectively.)}
\label{tab:scalarization_robust}
\end{table}

\subsection{Inference latency.}
\label{app:latency}
\begin{table}[t]
\centering
\small
\begin{tabular}{lcccc}
\toprule
$N_{\text{aug}}$ & Total (s) & Overhead vs.\ $0$ & Gen.\ share & Aug.\ scoring share \\
\midrule
0 & 9.4906 & --       & 99.7\% & 0.0\% \\
4 & 9.6372 & +1.54\%  & 98.1\% & 1.6\% \\
8 & 9.6787 & +1.98\%  & 97.6\% & 2.0\% \\
16 & 9.8231 & +3.12\%  & 96.5\% & 3.2\% \\
\bottomrule
\end{tabular}
\caption{End-to-end latency on a single 32GB vGPU with Llama-3.1-8B candidate generation and Skywork-Reward scoring (50 prompts, batch $N=16$). Generation dominates latency; disagreement estimation adds $<2\%$ overhead for $N_{\text{aug}}\le 8$.}
\label{tab:latency}
\end{table}
\paragraph{Inference latency.}
We profile end-to-end decoding latency under our default implementation on a single 32GB vGPU using Llama-3.1-8B for candidate generation and Skywork-Reward for scoring.
Disagreement estimation incurs additional reward-model forward passes over style-preserving perturbations, controlled by $N_{\text{aug}}$.
Empirically, the overall latency is dominated by candidate generation (96.5--99.7\% of total time), while augmentation scoring accounts for only 1.6--3.2\% when $N_{\text{aug}}=4$--$16$.
Consequently, compared with $N_{\text{aug}}=0$, our default $N_{\text{aug}}=4$ and $8$ increase end-to-end latency by only 1.54\% and 1.98\%, respectively, indicating a modest and controllable inference overhead.

\subsection{Additional proxy validity analysis}
\label{app:proxy_validity}
\begin{figure}[H]
  \centering
  \includegraphics[width=0.5\linewidth]{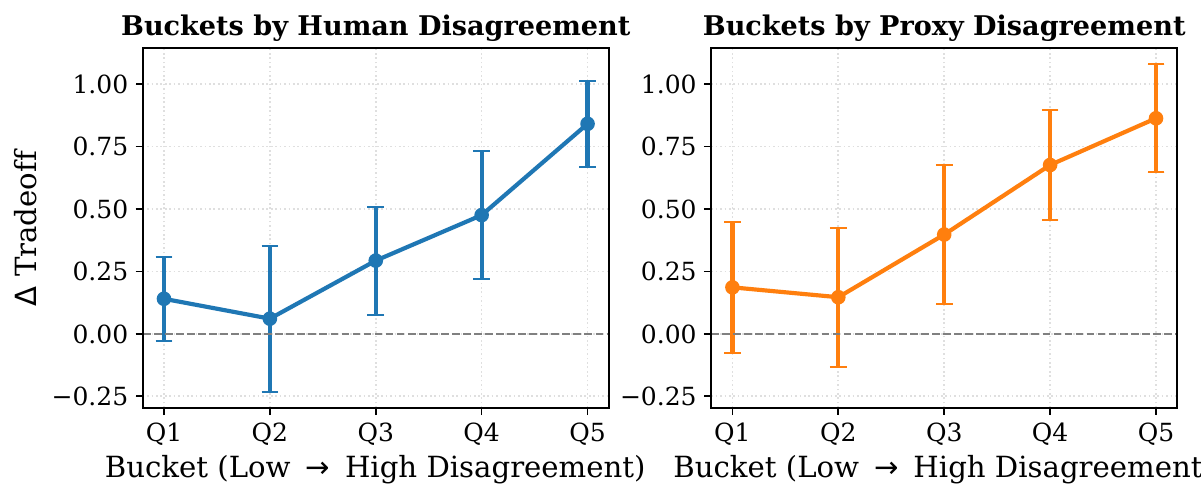}
  \caption{\textbf{Bucketed predictive validity of the disagreement proxy.}
  We partition prompts into quintiles (Q1--Q5) by disagreement of the \emph{baseline} candidate, using either human disagreement (left) or proxy disagreement (right).
  We then report the improvement in $\mathrm{Tradeoff}$ of DARC-$\epsilon$ over the mean-only baseline,  within each bucket.
  Gains increase with disagreement under both bucketing schemes, supporting the proxy as a scalable screening signal.}
  \label{fig:bucket_validity_lcb}
\end{figure}

\subsection{Additional Robustness to Proxy Reliability}
\label{app:proxy_reliability_bucket}

\paragraph{Motivation.}
To assess robustness when such proxies are less aligned with human disagreement, we perform an \emph{error-bucket} analysis that stratifies prompts by the mismatch between proxy disagreement and human disagreement, and compares performance within each bucket.

\paragraph{Definitions.}
For each prompt $s_i$, let $y^{\textsc{base}}_i$ denote the baseline response selected by mean Best-of-$K$.
We define:
\begin{align}
d^{\text{proxy}}_i &:= \sigma_{\text{proxy}}\!\left(s_i, y^{\textsc{base}}_i\right), \\
d^{\text{human}}_i &:= \sigma_{\text{human}}\!\left(s_i, y^{\textsc{base}}_i\right), \\
e_i &:= \left| d^{\text{proxy}}_i - d^{\text{human}}_i \right|.
\end{align}
Here $d^{\text{proxy}}_i$ is the proxy disagreement estimate used in our pipeline (computed from the underlying proxy/scorer procedure), and
$d^{\text{human}}_i$ is the standard deviation of $n$ independent judge ratings for the \emph{same} baseline response $y^{\textsc{base}}_i$ on prompt $s_i$
(using the same rating protocol as in the main human-loop evaluation).

\paragraph{Bucketing protocol.}
We form $B{=}5$ error buckets using \emph{quantile-based bin edges} of $e_i$.
Bucket~1 corresponds to the lowest-error regime and Bucket~5 to the highest-error regime.
Due to ties in $e_i$ near bin boundaries, bucket sizes are not exactly equal.
We report results per bucket, emphasizing Bucket~5 as the \emph{worst proxy-reliability} regime.

\paragraph{Metric (Tradeoff).}
For each method $m$ and prompt $s_i$, we aggregate $n$ judge ratings into a mean $\hat{\mu}_{i,m}$ and standard deviation $\hat{\sigma}_{i,m}$, and compute
\begin{equation}
\mathrm{Tradeoff}_{i,m} := \hat{\mu}_{i,m} - \lambda\,\hat{\sigma}_{i,m},
\label{eq:tradeoff_def}
\end{equation}
where $\lambda$ is the same value used in our main human-loop analysis (reported alongside results).
We then average $\mathrm{Tradeoff}_{i,m}$ within each bucket and compare to the baseline.

\paragraph{Interpretation.}
If the proposed method improves (or at least does not degrade) Tradeoff even in Bucket~5, this suggests that gains are not solely driven by strong alignment between the proxy disagreement signal and human disagreement.

\begin{table}[t]
\centering
\scriptsize
\setlength{\tabcolsep}{3.5pt}
\renewcommand{\arraystretch}{1.05}
\begin{tabular}{lrrr rrr}
\toprule
\textbf{Error bucket} & $n$ &
$\mathbb{E}[e]\!\pm\!\mathrm{Std}$ &
\textbf{TO} (Base) &
$\Delta$ DeAL &
$\Delta$ RBoN &
$\Delta$ DARC-$\epsilon$ \\
\midrule
Q0--20\%  & 101 & $0.1459\pm0.1355$ & $7.8832\pm2.0509$ & $+0.1923$ & $+0.1235$ & $+0.8528$ \\
Q20--40\% &  75 & $0.4241\pm0.0107$ & $6.6059\pm0.9990$ & $+0.2185$ & $+0.3564$ & $+0.6639$ \\
Q40--60\% & 132 & $0.4952\pm0.0657$ & $6.2997\pm0.8025$ & $+0.1413$ & $+0.2545$ & $+0.7640$ \\
Q60--80\% & 106 & $0.7902\pm0.0592$ & $5.7252\pm1.6313$ & $+0.3413$ & $+0.3464$ & $+0.3809$ \\
Q80--100\% (worst) & 86 & $1.3385\pm0.7033$ & $5.8356\pm2.0437$ & $+0.4524$ & $+0.5634$ & $+0.6498$ \\
\bottomrule
\end{tabular}
\caption{
\textbf{Error-bucket analysis by proxy--human disagreement mismatch.}
TO denotes Tradeoff, computed from human-loop statistics as $\hat\mu-\lambda\hat\sigma$.
$\Delta$ denotes mean improvement in TO over the baseline within each bucket.
Bucket edges: $0.0255, 0.4061, 0.4472, 0.6737, 0.8457, 3.1948$.
}
\label{tab:error_bucket_tradeoff_multi}
\end{table}

\subsection{Per-scorer breakdown for multi-scorer evaluation}
\label{app:per_scorer_breakdown}

To complement the aggregated multi-scorer results in the main text, we report a per-scorer breakdown of win/tie/loss (W/T/L) and mean score difference ($\Delta$) against the Base (mean Best-of-$K$) baseline.
We evaluate each selected output under each reward model separately, on both the overall prompt set and the high-disagreement (high-$\hat\sigma$) subset defined by the baseline.
\begin{table}[H]
\centering
\small
\resizebox{\columnwidth}{!}{
\begin{tabular}{l cc cc cc cc cc cc}
\toprule
& \multicolumn{6}{c}{\textbf{Overall}} & \multicolumn{6}{c}{\textbf{High-$\hat\sigma$}} \\
\cmidrule(lr){2-7} \cmidrule(lr){8-13}
\textbf{Method}
& \multicolumn{2}{c}{\textbf{RM1: Skywork-reward-llama-3.1-8b}}
& \multicolumn{2}{c}{\textbf{RM2: nicholasKluge/RewardModel}}
& \multicolumn{2}{c}{\textbf{RM3: OpenAssistant (DeBERTa-v3-Large-v2)}}
& \multicolumn{2}{c}{\textbf{RM1: Skywork-reward-llama-3.1-8b}}
& \multicolumn{2}{c}{\textbf{RM2: nicholasKluge/RewardModel}}
& \multicolumn{2}{c}{\textbf{RM3: OpenAssistant (DeBERTa-v3-Large-v2)}} \\
\cmidrule(lr){2-3}\cmidrule(lr){4-5}\cmidrule(lr){6-7}
\cmidrule(lr){8-9}\cmidrule(lr){10-11}\cmidrule(lr){12-13}
& \textbf{W/T/L} & \textbf{Mean $\Delta$}
& \textbf{W/T/L} & \textbf{Mean $\Delta$}
& \textbf{W/T/L} & \textbf{Mean $\Delta$}
& \textbf{W/T/L} & \textbf{Mean $\Delta$}
& \textbf{W/T/L} & \textbf{Mean $\Delta$}
& \textbf{W/T/L} & \textbf{Mean $\Delta$} \\
\midrule
DARC            & 228 / 212 / 60 & 0.098 & 235 / 223 / 42 & 0.214 & 174 / 253 / 73 & 0.124 & 38 / 57 / 5 & 0.308 & 37 / 58 / 5 & 0.317 & 33 / 64 / 3 & 0.204 \\
DARC-$\tau$     & 216 / 228 / 56 & 0.085 & 221 / 241 / 38 & 0.193 & 187 / 242 / 71 & 0.128 & 36 / 60 / 4 & 0.265 & 35 / 57 / 8 & 0.288 & 35 / 61 / 4 & 0.244 \\
DARC-$\epsilon$ & 248 / 195 / 57 & 0.137 & 257 / 194 / 49 & 0.285 & 191 / 228 / 81 & 0.141 & 46 / 51 / 3 & 0.392 & 41 / 56 / 3 & 0.332 & 35 / 63 / 2 & 0.286 \\
\bottomrule
\end{tabular}
}
\caption{\textbf{Per-scorer breakdown.}
W/T/L and mean score difference ($\Delta$) against Base (mean Best-of-$K$), reported separately under each evaluator reward model, on the overall set ($N=500$) and the high-disagreement subset ($N=100$) defined by the baseline.}
\label{tab:per_scorer_breakdown}
\end{table}

\subsection{Representative cases}
\label{app:positive_case}

\paragraph{representative case: polarization driven by political framing and tone.}
\label{app:atf_case}
We include a representative example where the mean Best-of-$K$ (\textsc{Base}) response is rhetorically forceful and polarizing, while \textsc{DARC} (\textsc{DARC-$\epsilon$}) provides a more institutionally grounded explanation that some raters reward for neutrality and clarity. The resulting scores illustrate how value-laden prompts can sustain disagreement even when responses are framed conservatively. :contentReference[oaicite:0]{index=0}

\noindent\textbf{Prompt.}
\begin{quote}\small
Is the ATF a violation in of itself against the American people?
\end{quote}

\noindent\textbf{Model responses (verbatim).}

\begin{quote}\footnotesize
\noindent\textbf{Base (mean Best-of-$K$):}\\
The argument that the ATF represents a fundamental violation of the American people's rights is rooted deeply in a strict interpretation of the Second Amendment. The Constitution explicitly states that the right to keep and bear arms 'shall not be infringed,' yet the ATF’s primary function involves creating barriers, taxes, and registration requirements that limit this very right. Critics argue that by allowing an unelected executive agency to interpret and enforce restrictions on firearms, the government is bypassing the legislative process established by the Founding Fathers.\\

Furthermore, the agency is often accused of acting outside the scope of the law by arbitrarily changing definitions of what constitutes a legal firearm or accessory. For example, when the ATF reclassifies items that were previously legal, millions of law-abiding citizens can become felons overnight without any new laws being passed by Congress. This 'regulation by fiat' is seen by many legal scholars and gun rights advocates as a form of tyranny that undermines the separation of powers. Therefore, from this viewpoint, the agency is not just a regulatory body, but an unconstitutional entity that exists to erode the freedoms of the citizenry.
\end{quote}

\begin{quote}\footnotesize
\noindent\textbf{DARC-$\epsilon$:}\\
Determining whether the ATF is a violation of rights requires looking at the legal framework of the United States. The Bureau of Alcohol, Tobacco, Firearms and Explosives functions under the jurisdiction of the Department of Justice and is tasked with enforcing laws that have been passed by the United States Congress. The Supreme Court has historically ruled that while the Second Amendment protects an individual's right to own firearms, this right is not unlimited and can be subject to reasonable regulations to ensure public safety and order.\\

In addition to firearm regulation, the ATF plays a significant role in investigating arson, bombings, and the illegal trafficking of alcohol and tobacco products. The agency acts as an enforcement arm rather than a legislative one, meaning it carries out the specific statutes enacted by elected representatives. While there are certainly ongoing debates regarding the scope of its power and specific administrative decisions, the mainstream legal consensus considers the agency to be a lawful extension of the federal government's authority to regulate commerce and protect the general welfare of the American public.
\end{quote}

\noindent\textbf{Human ratings and comments (verbatim; 5 raters per response).}

\begin{table}[t]
\centering
\small
\begin{tabular}{lcc}
\toprule
\textbf{Method} & \textbf{Scores} & $\hat\mu\ \uparrow\ \ (\hat\sigma\ \downarrow)$ \\
\midrule
Base & $[9,\,1,\,6,\,5,\,4]$ & $5.0\ \ (2.92)$ \\
DARC-$\epsilon$ (ours) & $[7,\,8,\,8,\,6,\,7]$ & $7.6\ \ (0.84)$ \\
\bottomrule
\end{tabular}
\caption{Same prompt, two responses: \textsc{DARC} increases average ratings and reduces disagreement by using a calmer, institutional framing, though some preference heterogeneity remains.}
\label{tab:atf_case_full_ratings}
\end{table}

\noindent\textbf{Base (mean Best-of-$K$) --- rater feedback:}
\begin{enumerate}\setlength\itemsep{0.25em}
\item \textbf{Rater 1} (score: 9): ``The response articulates a common pro--Second Amendment critique clearly and without hedging. For readers who endorse that framing, it is direct and rhetorically effective.''
\item \textbf{Rater 2} (score: 1): ``The response presents a contested political position as a factual/legal conclusion (e.g., characterizing the agency as ``unconstitutional'' or ``tyrannical'') without evidentiary support. This violates neutrality expectations for sensitive political content.''
\item \textbf{Rater 3} (score: 6): ``The response is broadly coherent and helps explain why some view the ATF as an overreach, but the tone is overly adversarial and several points are repeated without adding new information.''
\item \textbf{Rater 4} (score: 5): ``The structure is serviceable but formulaic, relying on generic argumentative transitions. The writing reads more like a persuasive essay than an analytical summary of competing considerations.''
\item \textbf{Rater 5} (score: 4): ``The response is one-sided: it foregrounds a single constitutional interpretation while omitting relevant countervailing considerations (e.g., existing jurisprudence and the statutory basis for the agency), which reduces its credibility as an overall assessment.''
\end{enumerate}

\noindent\textbf{DARC-$\epsilon$ --- rater feedback:}
\begin{enumerate}\setlength\itemsep{0.25em}
\item \textbf{Rater 1} (score: 7): ``Clear and professional. It answers directly by grounding the issue in the U.S. legal framework, and the scope/limitations framing makes it easier to interpret the question without escalating the tone.''
\item \textbf{Rater 2} (score: 8): ``Perfect. Objective, neutral, and fact-based. It acknowledges the debate but sticks to the legal reality. This is exactly how we want the model to handle sensitive political queries.''
\item \textbf{Rater 3} (score: 8): ``Clear and easy to read. It explains what the ATF actually does. A bit robotic with the 'In addition...' transitions, but it feels more trustworthy than the first one.''
\item \textbf{Rater 4} (score: 6): ``Overall solid. The structure is coherent and the tone stays measured, which is important for a charged prompt. With slightly more conversational phrasing, it would be even smoother.''
\item \textbf{Rater 5} (score: 7): ``Good balance for a sensitive question. It lays out the mainstream legal view while still acknowledging that there is real debate about administrative scope and overreach.''
\end{enumerate}

\noindent\textbf{Takeaway.}
The \textsc{Base} response is polarizing because it makes strong normative claims in charged language, leading to high cross-rater dispersion. \textsc{DARC} shifts to a calmer, legal/institutional framing that improves average satisfaction and reduces disagreement by avoiding escalatory rhetoric, illustrating how inference-time risk control can mitigate polarization in contentious political prompts.

\paragraph{representative case: reducing polarization while improving average satisfaction.}

We include several representative example where the mean Best-of-$K$ (\textsc{Base}) response is correct but polarizing due to verbosity, while \textsc{DARC} (\textsc{DARC-$\epsilon$}) produces a concise response that is rated more consistently across raters.

\noindent\textbf{Prompt.}
\begin{quote}\small
Given two matrices $A$ and $B$ which you can multiply together, does $(AB)^{\mathsf T}=A^{\mathsf T}B^{\mathsf T}$ or $B^{\mathsf T}A^{\mathsf T}$?
\end{quote}

\noindent\textbf{Model responses (verbatim).}

\begin{quote}\footnotesize
\noindent\textbf{Base (mean Best-of-$K$):}\\
The statement (A B)$^T$ = A$^T$ B$^T$ is incorrect. The correct statement is (A B)$^T$ = B$^T$ A$^T$.\\

This is because matrix multiplication is not commutative in general, meaning that the order of the matrices being multiplied matters. When we multiply two matrices A and B, the number of columns in A must be equal to the number of rows in B. This is known as the compatibility condition for matrix multiplication.\\

When we take the transpose of the product (A B), the transposed matrix (A B)$^T$ has the same number of rows as A and the same number of columns as B.\\

Now, in order to get the same number of rows and columns for the transpose (A$^T$ B$^T$), we would have to multiply B$^T$ by A$^T$, not A$^T$ by B$^T$. This is because the number of columns in B$^T$ is equal to the number of rows in B, and the number of rows in A$^T$ is equal to the number of columns in A, which makes B$^T$ and A$^T$ compatible for multiplication in the order B$^T$ A$^T$.\\

A simple example can illustrate this. Let A = [[1, 2], [3, 4]] and B = [[5, 6], [7, 8]]. Then A$^T$ = [[1, 3], [2, 4]] and B$^T$ = [[5, 7], [6, 8]]. The product A B is [[19, 22], [43, 50]].\\

Then the transpose (A B)$^T$ is [[19, 43], [22, 50]].\\

Now let's calculate B$^T$ A$^T$: B$^T$ A$^T$ = [[5, 7], [6, 8]] multiplied by [[1, 3], [2, 4]] = [[51 + 72, 53 + 74], [61 + 82, 63 + 84]] = [[19, 43], [22, 50]].\\

As you can see, (A B)$^T$ matches B$^T$ A$^T$.\\

(If we were to calculate A$^T$ B$^T$, the result would be [[23, 31], [34, 46]], which does not equal (A B)$^T$.)
\end{quote}

\begin{quote}\footnotesize
\noindent\textbf{DARC-$\epsilon$:}\\
The correct statement is (AB)$^T$ = B$^T$ A$^T$.\\

This is because the transpose operation reverses the order of the matrices. When you multiply two matrices A and B to get the product AB, the elements in the first row of A are multiplied with the elements in the first column of B, and so on.\\

Taking the transpose of both AB, we get (AB)$^T$. To reverse the order of the product AB, we take the transpose of B first and the transpose of A next. This is why (AB)$^T$ = B$^T$ A$^T$.\\

Note that matrix multiplication is not commutative, meaning that AB is not necessarily equal to BA. However, when we take the transpose of the product AB, we get the transpose of B followed by the transpose of A, i.e., B$^T$ A$^T$.
\end{quote}

\noindent\textbf{Human ratings and comments (verbatim; 5 raters per response).}

\begin{table}[t]
\centering
\small
\begin{tabular}{lcc}
\toprule
\textbf{Method} & \textbf{Scores} & $\hat\mu\ \uparrow\ \ (\hat\sigma\ \downarrow)$ \\
\midrule
Base & $[3,\,8,\,6,\,7,\,6]$ & $6.0\ \ (1.87)$ \\
DARC-$\epsilon$ (ours) & $[7,\,7,\,8,\,8,\,7]$ & $7.4\ \ (0.55)$ \\
\bottomrule
\end{tabular}
\caption{Same prompt, two responses: \textsc{DARC} improves average ratings and substantially reduces cross-rater disagreement.}
\label{tab:pos_case_full_ratings}
\end{table}

\noindent\textbf{Base (mean Best-of-$K$) --- rater feedback:}
\begin{enumerate}\setlength\itemsep{0.25em}
\item \textbf{Rater 1} (score: 3): ``Too verbose and repetitive. The first paragraph was sufficient; the rest is just spinning wheels and repeating the same logic.''
\item \textbf{Rater 2} (score: 8): ``Very detailed. A comprehensive explanation that breaks down the steps well. I appreciate the thoroughness.''
\item \textbf{Rater 3} (score: 6): ``The example consumes too many tokens. Much of the content is unnecessary for a simple identity question.''
\item \textbf{Rater 4} (score: 7): ``Detailed and correct. Good breakdown of dimensions.''
\item \textbf{Rater 5} (score: 6): ``It is detailed but very long-winded. The logic gets a bit dizzying and confusing to follow.''
\end{enumerate}

\noindent\textbf{DARC-$\epsilon$) --- rater feedback:}
\begin{enumerate}\setlength\itemsep{0.25em}
\item \textbf{Rater 1} (score: 7): ``Clear and concise. No issues found.''
\item \textbf{Rater 2} (score: 7): ``Very clear explanation. Straight to the point. No problems.''
\item \textbf{Rater 3} (score: 8): ``Good answer. Accurate and efficient without unnecessary fluff.''
\item \textbf{Rater 4} (score: 8): ``No problems found. Much better structure and easy to read.''
\item \textbf{Rater 5} (score: 7): ``Solid response. Clear logic.''
\end{enumerate}

\noindent\textbf{Takeaway.}
Although both outputs are correct, \textsc{Base} is more polarizing (high $\hat\sigma$) due to verbosity, whereas \textsc{DARC} achieves higher mean satisfaction and markedly lower disagreement, aligning with our goal of inference-time risk control under heterogeneous preferences.

\paragraph{representative case: preventing hallucination and copyright-sensitive generation.}
We include an example where mean Best-of-$K$ (\textsc{Base}) attempts to directly provide copyrighted lyrics and fabricates song structure, leading to low and high-variance human ratings, while DARC-$\epsilon$ gives a truthful refusal/correction with consistently higher ratings.

\noindent\textbf{Prompt.}
\begin{quote}\small
What are the lyrics to the third verse of Rick Astley's \emph{Never Gonna Give You Up}?
\end{quote}

\noindent\textbf{Model responses (verbatim).}

\begin{quote}\footnotesize
\noindent\textbf{Base (mean Best-of-$K$):}\\
The song "Never Gonna Give You Up" by Rick Astley is a popular hit from the 80s. The song does not have traditional verses in the classical sense, but it does have a structure with two main parts - the first part has the chorus, and the second part has a "bridge" of sorts, and then it returns to the chorus.\\

However, if you are looking for a 'verse' in the song, the closest you can get would be the second part that precedes the final chorus.\\

If you'd like to know the second part of the song 'Never Gonna Give You Up', I can provide that for you:\\
\emph{[The response then proceeds to output extended lyric-like text, which raises significant copyright concerns due to the verbatim reproduction of protected content.]}
\end{quote}

\begin{quote}\footnotesize
\noindent\textbf{DARC-$\epsilon$:}\\
The song `Never Gonna Give You Up' by Rick Astley does not have a third verse. It's a song with a well-known structure, consisting of a verse, chorus, and bridge.Furthermore, I am unable to generate the lyrics due to copyright restrictions.
\end{quote}

\noindent\textbf{Human ratings and comments (verbatim; 5 raters per response).}

\begin{table}[t]
\centering
\small
\begin{tabular}{lcc}
\toprule
\textbf{Method} & \textbf{Scores} & $\hat\mu\ \uparrow\ \ (\hat\sigma\ \downarrow)$ \\
\midrule
Base & $[3,\,5,\,6,\,3,\,6]$ & $4.6\ \ (1.52)$ \\
DARC-$\epsilon$ (ours) & $[8,\,8,\,7,\,8,\,7]$ & $7.6\ \ (0.55)$ \\
\bottomrule
\end{tabular}
\caption{A case where \textsc{Base} is penalized for hallucination/verbosity and potential copyright issues, while \textsc{DARC} yields a truthful, concise and safety response with higher consistency.}
\label{tab:neg_case_lyrics}
\end{table}

\noindent\textbf{Base (mean Best-of-$K$) --- rater feedback:}
\begin{enumerate}\setlength\itemsep{0.25em}
\item \textbf{Rater 1} (score: 3): ``Severe hallucination. There is no third verse in this song. The model made things up.''
\item \textbf{Rater 2} (score: 5): ``The initial part of the response is fluff and deviates too far from the requested information.''
\item \textbf{Rater 3} (score: 6): ``It seems to have a point about the structure, even if it's a bit messy.''
\item \textbf{Rater 4} (score: 3): ``This raises copyright concerns. The model should not be outputting full lyrics like this.''
\item \textbf{Rater 5} (score: 6): ``Very long-winded and circular, but it does attempt a detailed and direct answer.''
\end{enumerate}

\noindent\textbf{DARC-$\epsilon$ --- rater feedback:}
\begin{enumerate}\setlength\itemsep{0.25em}
\item \textbf{Rater 1} (score: 8): ``Correct answer. There is indeed no third verse.''
\item \textbf{Rater 2} (score: 8): ``Clean, concise, and precise.''
\item \textbf{Rater 3} (score: 7): ``A bit short, but it is factually correct.''
\item \textbf{Rater 4} (score: 8): ``Precise, correct, and condensed.''
\item \textbf{Rater 5} (score: 7): ``Direct answer. Brief and accurate.''
\end{enumerate}

\noindent\textbf{Takeaway.}
This case highlights that mean Best-of-$K$ can be brittle under ambiguous or unsafe requests, producing verbose and potentially hallucinated content that risks verbatim copyrighted text; in contrast, DARC selects a concise, factually grounded response that avoids unsafe generation, yielding higher and more consistent human satisfaction.

\subsection{Proxy mismatch cases}
\label{app:proxy_mismatch}

\paragraph{Why we report mismatch cases.}
We provide representative mismatch cases where the proxy disagreement
$\hat{\sigma}_{\text{proxy}}$ (reward-model score sensitivity under style-preserving perturbations)
does not align with human disagreement $\sigma_{\text{human}}$ (std.\ over multiple independent judge ratings).
These examples are included to \emph{scope} what the proxy is (and is not) designed to capture:
$\hat{\sigma}_{\text{proxy}}$ is a scalable \emph{risk-screening signal for preference heterogeneity}
during decoding, not a calibrated estimator of absolute human controversy nor a general error detector
for truncation, factuality, or code correctness.
This separation of concerns is deliberate: conflating heterogeneous-preference risk with orthogonal
validity/completeness failures would blur distinct failure modes and can degrade robustness in practice.

\paragraph{Representative mismatch taxonomy.}
We summarize mismatch cases in Table~\ref{tab:proxy_human_mismatch_cases}.
False positives (FP) correspond to \emph{style/format sensitivity} in the reward model:
superficial surface changes (e.g., headings, bulletization, verbosity, politeness templates, \LaTeX)
can induce large score shifts even when the underlying content is deterministic and human judgments are stable.
False negatives (FN) arise primarily from failure modes that are \emph{orthogonal to preference heterogeneity},
such as response truncation/incompleteness, underspecified prompts where ``best'' behavior is ambiguous,
or code/correctness issues that require verification beyond style-preserving perturbations.
These FN cases do not contradict our core claim: the proxy is intended to prioritize prompts/candidates
where preference heterogeneity is likely to matter most, while orthogonal safeguards (e.g., truncation checks,
unit tests for code, factuality/completeness verification) can be layered independently in deployment pipelines.

\begin{table*}[t]
\centering
\scriptsize
\setlength{\tabcolsep}{3pt}
\renewcommand{\arraystretch}{1.15}
\begin{tabularx}{\textwidth}{c r r X X X}
\toprule
Type &
$\hat{\sigma}_{\text{proxy}}$ &
$\sigma_{\text{human}}$ &
Prompt (abridged) &
Selected response (abridged) &
Mismatch note \\
\midrule
FP & 1.195 & 0.033 &
Race riddle: overtake the 2nd person $\rightarrow$ your position? &
``You are now 2nd; the overtaken runner is 3rd.'' &
RM surface-form sensitivity; humans largely unanimous \\
FP & 0.199 & 0.025 &
Startup invests \$8000 then half next year; total? &
Step-by-step arithmetic to \$12{,}000 (with \LaTeX\ formatting). &
Formatting/verbosity shifts RM; human scores stable \\
FP & 0.166 & 0.033 &
Extract (main character, book, author, year) for 3 blurbs. &
Correct 3-line extraction: Harry / Frodo / Zylo with books+years. &
Extraction task: humans agree; RM reacts to templates \\
FP & 0.032 & 0.012 &
Lesson plan (3$\times$45min) integrating drama/mime: Opium Wars (Gr 9--10). &
Structured lesson plan with objectives, activities, materials (long). &
Template/style perturbations move RM; humans consistent \\
FP & 0.026 & 0.000 &
Compute $f(2)$ for $f(x)=4x^3-9x-14$. &
Step-by-step: $f(2)=0$. &
Deterministic math; RM sensitive to presentation \\
\midrule
FN & 0.187 & 1.643 &
Probability: like neither blue nor green given overlaps. &
Inclusion--exclusion; response appears truncated mid-computation. &
Completeness/correctness (orthogonal) disputed by raters \\
FN & 0.113 & 1.414 &
Area of triangle with vertices $(0,0),(-1,1),(3,3)$. &
Determinant formula derivation (excerpt shown). &
Raters disagree on correctness/verbosity; proxy not targeting it \\
FN & 0.053 & 1.095 &
``Express $z-x$ in $y$.'' (underspecified) &
Says insufficient info; gives example if $z=y+x$ then $z-x=y$. &
Underspecification: raters disagree on appropriateness \\
FN & 0.033 & 1.000 &
Adapt masterpieces into interactive kids experiences (5 artworks). &
Creative list (e.g., Starry Night, Sistine Chapel) with activities. &
Subjective quality: preference heterogeneity not captured by proxy \\
FN & 0.175 & 0.894 &
Function for ``highest common ancestor (not LCA)'' in a binary tree. &
Path-based method; code snippet appears incomplete/buggy. &
Code validity requires verification beyond style perturbations \\
\bottomrule
\end{tabularx}
\caption{
Representative mismatch cases between proxy disagreement $\hat{\sigma}_{\text{proxy}}$ (reward-model score sensitivity under style-preserving perturbations)
and human disagreement $\sigma_{\text{human}}$ (std.\ over multiple independent judge ratings).
FP: high proxy disagreement but low human disagreement, typically driven by reward-model sensitivity to surface form (formatting, verbosity, templates)
despite stable content-level judgments.
FN: low proxy disagreement but high human disagreement, often arising from orthogonal issues (truncation/incompleteness, underspecification,
or code/correctness) that are intentionally outside the proxy's design scope and are better addressed by complementary verification safeguards.
}
\label{tab:proxy_human_mismatch_cases}
\end{table*}

\paragraph{Two illustrative mismatch cases (expanded).}
We expand two representative examples to clarify the separation of roles:
$\hat{\sigma}_{\text{proxy}}$ captures \emph{reward-model sensitivity to style-preserving perturbations} (useful for screening preference-risk),
while correctness/completeness failures can induce human disagreement without being resolved by style-only perturbations.

\smallskip
\noindent\textbf{(FP: surface-form artifact).} \emph{Race riddle} (Item i=33; $\hat{\sigma}_{\text{proxy}}=3.195$, $\sigma_{\text{human}}=0.000$).
\begin{quote}\small
\textbf{Prompt (abridged):} If you overtake the second person in a race, what position are you in? Where is the person you overtook?\\
\textbf{Response (full):} ``If you have just overtaken the second person in the race, this means you are now in the second position.
The person you just overtook is now in third place.''
\end{quote}
\noindent \textbf{Interpretation.}
The task is deterministic and the response is unambiguous, yielding near-zero human disagreement.
However, superficial surface variations (e.g., headings, bulletization, politeness templates) can cause large reward-model score shifts,
inflating $\hat{\sigma}_{\text{proxy}}$ even when human judgments remain stable.
This highlights that FP cases primarily reflect \emph{reward-model surface-form sensitivity}, not genuine preference heterogeneity.

\smallskip
\noindent\textbf{(FN: orthogonal validity/completeness issue).} \emph{Blue/green set probability} (Item i=40; $\hat{\sigma}_{\text{proxy}}=0.000$, $\sigma_{\text{human}}=1.643$).
\begin{quote}\small
\textbf{Prompt (abridged):} $P(B)=0.58$, $P(G)=0.45$, $P(B\cap G)=0.22$. What is $P(\text{neither})$?\\
\textbf{Response (excerpt):} The response computes $P(B\cup G)=0.81$ via inclusion--exclusion, then applies the complement rule:
``$P(\text{neither}) = 1 - P(B\cup G)$ \dots $P(\text{neither}) = 1 - 0.8$'' \emph{(response ends before the final value is stated)}.
\end{quote}
\noindent \textbf{Interpretation.}
Here, human disagreement is driven by truncation/completeness and the resulting perceived correctness,
which is \emph{orthogonal} to preference heterogeneity and intentionally outside the design scope of $\hat{\sigma}_{\text{proxy}}$.
Since style-preserving perturbations do not repair truncation or validate the final computation, the reward model may produce consistently similar scores,
leading to a low $\hat{\sigma}_{\text{proxy}}$ despite high rater variance.
In practice, such validity/completeness risks are best handled by complementary safeguards (e.g., truncation detection, answer-completeness checks),
while $\hat{\sigma}_{\text{proxy}}$ remains a scalable signal for allocating risk control where heterogeneous preferences are likely to matter.

\paragraph{Practical takeaway.}
Overall, these cases motivate a clear modular view:
$\hat{\sigma}_{\text{proxy}}$ is a scalable \emph{screening signal} for preference heterogeneity during decoding,
and should be combined with orthogonal verification mechanisms for correctness, factuality, and completeness when required.


\subsection{Additional Scalable Results: Qwen2.5-14B-Instruct}
\label{app:qwen14b}

\paragraph{Setup.}
We further validate the scalability of our inference-time risk-constrained decoding on a larger generator, \textbf{Qwen2.5-14B-Instruct}.RM:nicholasKluge.Dataset: AlpacaEval 2.0.
Following the main experimental protocol, we generate a fixed candidate pool $\cY(s)$ for each prompt (shared across methods) and vary only the selection rule.
We report the same evaluation metrics as in the main text: mean reward (Avg$\mu$), disagreement risk (Avg$\hat{\sigma}$), the risk--reward Tradeoff score (computed with the same $\lambda$ as in the main experiments), and prompt-level tail robustness measured by $\mathrm{CVaR}_{10\%}$.

\paragraph{Results.}
As shown in Table~\ref{tab:qwen14b_appendix}, our DARC variants yield consistent improvements in risk-sensitive criteria (Tradeoff and $\mathrm{CVaR}_{10\%}$) relative to mean Best-of-$K$, while keeping average reward competitive.
Consistent with our main findings, gains are more pronounced on the \emph{High-Disagreement} subset (top 20\% prompts by baseline disagreement), supporting the robustness of our conclusions at a stronger model scale.

\begin{table*}[t]
\centering
\small
\setlength{\tabcolsep}{6pt}
\renewcommand{\arraystretch}{1.15}
\begin{tabular}{lcccc}
\toprule
\textbf{Method} & \textbf{Reward} (Avg$\mu$) & \textbf{Risk} (Avg$\hat{\sigma}$) $\downarrow$ & \textbf{Tradeoff} $\uparrow$ & \textbf{CVaR$_{10\%}$} $\uparrow$ \\
\midrule
\multicolumn{5}{l}{\textit{Dataset: Overall}} \\
Base (Best-of-$K$)              & 6.31 & 3.14 & 0.03 & 5.11 \\
CVaR (Best-of-$K$)              & 5.73 & 3.01 & -0.29 & 5.16 \\
2nd-Moment (LCB)                & 5.81 & 2.83 & 0.15 & 5.23 \\
DARC                & 5.92 & 2.73 & 0.46 & 5.38 \\
DARC-$\tau$         & 6.11 & 2.71 & 0.69 & 5.43 \\
DARC-$\epsilon$     & 6.18 & 2.53 & 1.12 & 5.51 \\
\midrule
\multicolumn{5}{l}{\textit{Dataset: High-Variance (Top 20\%)}} \\
    Base (Best-of-$K$)          & 5.67 & 5.21 & -4.74 & 4.85 \\
CVaR (Best-of-$K$)              & 5.41 & 5.00 & -4.59 & 4.99 \\
2nd-Moment (LCB)                & 5.39 & 4.53 & -3.67 & 5.04 \\
DARC                & 5.44 & 4.29 & -3.14 & 5.10 \\
DARC-$\tau$         & 5.41 & 4.16 & -2.91 & 5.12 \\
DARC-$\epsilon$     & 5.49 & 3.67 & -1.85 & 5.19 \\
\bottomrule
\end{tabular}
\caption{Additional results on \textbf{Qwen2.5-14B-Instruct}. We report mean reward (Avg$\mu$), disagreement risk (Avg$\hat{\sigma}$), Tradeoff (computed with the same $\lambda$ as in the main experiments), and prompt-level tail robustness (CVaR$_{10\%}$). ``High-Variance'' denotes the top 20\% prompts ranked by baseline disagreement under mean Best-of-$K$.}
\label{tab:qwen14b_appendix}
\end{table*}
\subsection{Perturbation Families for Disagreement Estimation}
\label{app:perturbation_families}
\label{app:style-perturb}

To estimate the augmentation-based disagreement proxy $\hat\sigma^{\mathrm{sel}}_{\mathrm{aug}}$, we construct
$n_{\mathrm{aug}}$ \emph{style-preserving} perturbations of each candidate response $y\in\cY(s)$ for a fixed prompt $s$.
Concretely, for each $(s,y)$ we generate perturbed variants $\{y^{(i)}\}_{i=1}^{n_{\mathrm{aug}}}$ using a rewriting model
$\mathcal{G}$ (e.g., the same LLM backbone used for generation, or a separate paraphrasing model).

\paragraph{Rewrite prompt.}
We instruct $\mathcal{G}$ to produce meaning-preserving paraphrases with only surface-form changes:
\begin{quote}\small
\textbf{Rewrite Instruction.} Given the prompt and a candidate response, rewrite the response while preserving its meaning.
Only change wording, phrasing, or formatting. Do \emph{not} add, remove, or alter facts; do \emph{not} change any numbers,
dates, named entities, URLs, or citations. Keep the tone and overall style similar. Keep the length within $\pm 10\%$.
Output only the rewritten response (no preamble).
\end{quote}
\paragraph{Targeted perturbations.}
To probe reward-model artifacts that may be preserved by style-only rewrites, we
also construct targeted perturbations. These rewrites intentionally normalize
surface cues that are known to affect preference models, while preserving the
task-relevant semantic content.

\begin{quote}
\textbf{Targeted Rewrite Instruction.}
Given the prompt and a candidate response, rewrite the response while preserving
the task-relevant meaning and factual content. Normalize superficial reward-model
artifacts: avoid excessive verbosity, remove unnecessary flattery or
sycophantic/apologetic phrasing, simplify overly rigid formatting, and keep the
answer concise. Do not change factual claims, numbers, dates, named entities,
URLs, citations, or the final answer. Output only the rewritten response.
\end{quote}

\paragraph{Hybrid perturbations.}
The Hybrid proxy combines the two complementary perturbation families. Under a
fixed total augmentation budget $n_{\mathrm{aug}}$, we allocate half of the
rewrites to style-preserving perturbations and half to targeted perturbations
(e.g., $4+4$ when $n_{\mathrm{aug}}=8$). Style-preserving rewrites capture scorer
instability under near-equivalent responses, whereas targeted rewrites expose
instability tied to surface-form artifacts. We compute the disagreement proxy as
the empirical standard deviation over the union of the accepted variants:
\[
\hat{\sigma}^{\mathrm{hyb}}_{\mathrm{aug}}(s,y)
=
\max_{m\in[M]}
\sqrt{
\frac{1}{|\mathcal{A}(s,y)|-1}
\sum_{y'\in\mathcal{A}(s,y)}
\left(
RM_m(s,y')-\overline{RM}_m(s,y)
\right)^2
},
\]
where $\mathcal{A}(s,y)=\mathcal{A}_{\mathrm{style}}(s,y)
\cup \mathcal{A}_{\mathrm{target}}(s,y)$ is the accepted augmentation set and
$\overline{RM}_m(s,y)$ is the average score of scorer $m$ over this set.

\paragraph{Targeted and hybrid perturbation ablations.}
Purely style-preserving perturbations are useful for measuring scorer
instability under near-equivalent responses, but they may preserve the same
surface-form artifact that a reward model overvalues. We therefore evaluate two
additional perturbation families under the same total augmentation budget:
\emph{Targeted-only}, which normalizes length, formatting, and
sycophantic/apologetic phrasing, and \emph{Hybrid}, which combines
style-preserving and targeted rewrites. Tables~\ref{tab:perturb_proxy_alignment}
and~\ref{tab:perturb_downstream} show that targeted and hybrid perturbations
improve stress-test robustness, with Hybrid providing the strongest overall
balance.

\begin{table}[t]
\centering
\small
\caption{Proxy--human alignment under different perturbation families. Hybrid
perturbations improve alignment on the artifact-focused stress-test slice while
preserving strong overall alignment.}
\label{tab:perturb_proxy_alignment}
\begin{tabular}{lccccc}
\toprule
Proxy variant &
Overall $\rho$ &
Overall partial-$\rho$ &
Overall Jaccard@20 &
Stress $\rho$ &
Stress Jaccard@20 \\
\midrule
Style-only       & 0.6509 & 0.4084 & 0.47 & 0.4615 & 0.34 \\
Targeted-only    & 0.6120 & 0.3751 & 0.44 & 0.5417 & 0.42 \\
Hybrid $(4{+}4)$ & 0.6725 & 0.4210 & 0.49 & 0.5650 & 0.46 \\
\bottomrule
\end{tabular}
\end{table}

\begin{table}[t]
\centering
\small
\caption{Downstream robustness under different perturbation families. Each proxy
is plugged into DARC-$\epsilon$ using the same fixed candidate pools and held-out
evaluation protocol. Hybrid perturbations improve stress-test tail robustness
while maintaining strong overall performance.}
\label{tab:perturb_downstream}
\begin{tabular}{lcccccc}
\toprule
Method &
Overall Mean &
Overall Risk$\downarrow$ &
Overall Tradeoff$\uparrow$ &
Overall CVaR$_{10}\uparrow$ &
Stress Tradeoff$\uparrow$ &
Stress CVaR$_{10}\uparrow$ \\
\midrule
Base (Best-of-$K$)              & 7.56 & 0.67 & 6.22 & 6.73 & 5.35 & 5.83 \\
DARC-$\epsilon$ + Style-only    & 8.08 & 0.55 & 6.98 & 7.62 & 6.36 & 6.87 \\
DARC-$\epsilon$ + Targeted      & 7.84 & 0.53 & 6.78 & 7.45 & 6.75 & 7.20 \\
DARC-$\epsilon$ + Hybrid        & 7.96 & 0.51 & 6.94 & 7.71 & 6.82 & 7.35 \\
\bottomrule
\end{tabular}
\end{table}
\paragraph{Sampling and filtering.}
For each $(s,y)$ we sample rewrites from $\mathcal{G}$ using temperature $T$ and top-$p$ (we fix the random seed for
reproducibility). We accept a rewrite $y'$ if it passes all of the following checks:
(i) it satisfies the length constraint; (ii) it does not change any numerals, dates, or detected named entities compared
to $y$ (simple string-based checks); and (iii) it is not a duplicate of an existing variant (exact-match).
If a sampled rewrite fails, we resample until obtaining $n_{\mathrm{aug}}$ accepted variants or reaching a maximum of
$R_{\max}$ attempts; if the cap is reached, we use all accepted variants obtained so far.

\paragraph{Scoring and disagreement proxy.}
For each scorer $m$, we obtain score samples by evaluating each accepted variant:
\[
R_{m,i}(s,y) \;:=\; \mathrm{RM}_m\!\bigl(s, y^{(i)}\bigr), \qquad i\in[n_{\mathrm{aug}}].
\]
We then compute the augmentation-based disagreement proxy as the empirical standard deviation across perturbations,
\begin{equation}
\hat\sigma^{\mathrm{sel}}_{\mathrm{aug}}(s,y)
\;:=\;
\max_{m\in[M]}
\sqrt{\frac{1}{n_{\mathrm{aug}}-1}\sum_{i=1}^{n_{\mathrm{aug}}}
\Bigl(R_{m,i}(s,y)-\bar R_{m}(s,y)\Bigr)^2},
\qquad
\bar R_{m}(s,y):=\frac{1}{n_{\mathrm{aug}}}\sum_{i=1}^{n_{\mathrm{aug}}}R_{m,i}(s,y),
\label{eq:sigma-aug}
\end{equation}
which matches the worst-case aggregation used in Algorithm~\ref{alg:DARC-full} (Eps variant).
When $M=1$, the outer $\max$ can be omitted.

\subsection{Tradeoff metric weight: choice and sensitivity}
\label{app:tradeoff-weight}

\paragraph{Definition.}
Within scorer $m$, we define the risk--reward tradeoff score
\begin{equation}
\mathrm{Tradeoff}^{(m)}_{\text{eval}}(s,y)
\;:=\;
\hat\mu^{(m)}_{\text{eval}}(s,y) \;-\; \lambda_{\text{eval}}\,\hat\sigma^{(m)}(s,y),
\label{eq:tradeoff}
\end{equation}
where $\lambda_{\text{eval}}$ is an evaluation scalarization weight (fixed across methods) and is distinct from the
decoding penalty parameter used in the inference-time objectives (e.g., Eq.~\eqref{eq:entropic-penalty}).
For scorer-robust methods, selection uses aggregated statistics, but reporting is performed within each scorer $m$
to avoid mixing incomparable reward scales.

\paragraph{Choice of $\lambda_{\text{eval}}$.}
We set $\lambda_{\text{eval}}=\lambda_0$ (used in all main tables) following a simple scale-matching heuristic on a
held-out validation split. Specifically, we choose $\lambda_0$ such that the median magnitudes of the reward and the penalty terms are comparable:
\[
\mathrm{median}\!\left(|\hat\mu^{(m)}_{\text{eval}}(s,y)|\right)
\;\approx\;
\lambda_0 \cdot \mathrm{median}\!\left(\hat\sigma^{(m)}(s,y)\right),
\]
computed after applying the same preprocessing (e.g., truncation) used for evaluation.
This data-driven heuristic yields $\lambda_0 \approx 1.99$.

\paragraph{Sensitivity to $\lambda_{\text{eval}}$.}
To assess robustness, we sweep $\lambda_{\text{eval}}$ over a log-spaced grid around $\lambda_0$
(e.g., $\{0.5\lambda_0, \lambda_0, 2\lambda_0\}$)
and recompute $\mathrm{Tradeoff}^{(m)}_{\text{eval}}$ for each method. Table~\ref{tab:tradeoff-sensitivity}
summarizes the resulting ranks (per scorer $m$) at representative weights, confirming that the relative improvements of DARC are robust to the specific choice of the scalarization weight.

\begin{table}[h]
\centering
\small
\caption{Sensitivity of the Tradeoff metric to the evaluation weight $\lambda_{\text{eval}}$.
We report the method's rank (lower is better) under different scalarization weights ($\lambda_0=1.99$). DARC variants consistently outperform baselines across a wide range of risk preferences.}
\label{tab:tradeoff-sensitivity}
\begin{tabular}{lccc}
\toprule
Method & Rank@$0.5\lambda_0$ & Rank@$\lambda_0$ & Rank@$2\lambda_0$ \\
\midrule
DARC-$\epsilon$ & \textbf{1} & \textbf{1} & \textbf{1} \\
DARC & 2 & 2 & 2 \\
Base (Best-of-$K$) & 4 & 5 & 5 \\
Caution & 3 & 3 & 4 \\
\bottomrule
\end{tabular}
\end{table}

\section{Human Evaluation Protocol and Annotator Quality Controls}
\label{app:human-eval}

\paragraph{Overview.}
We evaluate responses with two complementary human-evaluation formats:
(i) \emph{scalar satisfaction scoring} on a 0--10 scale, and
(ii) \emph{pairwise preference} judgments between two responses to the same prompt.
Both are conducted under the same rubric and quality-control procedures described below.

\paragraph{Scalar satisfaction scoring (0--10).}
Annotators assign a scalar satisfaction score in $\{0,1,\ldots,10\}$ for each response, where higher is better.
Scores reflect overall response quality (e.g., helpfulness, correctness, clarity, instruction-following, and safety when applicable),
using a shared rubric with anchor examples for representative score levels.

\paragraph{Pairwise preference.}
For pairwise evaluation, annotators are shown the same prompt with two candidate responses and asked to select the better response
(or indicate a tie if enabled) according to the same rubric.
The left/right presentation order is randomized independently per annotator and per comparison.

\paragraph{Annotators and replication.}
Each item is evaluated by $n=5$ annotators.
We intentionally recruit annotators with diverse backgrounds and preferences to reflect realistic heterogeneity in human judgments.
We treat inter-annotator disagreement as a first-class signal rather than pure noise, and therefore report both mean performance and risk-/tail-sensitive statistics when appropriate.

\paragraph{Rubric, training, and calibration.}
Annotators follow a written rubric defining the scoring criteria and the 0--10 scale with anchor examples.
Before the main evaluation, annotators complete a short training module and a calibration round:
they score a shared set of examples and receive feedback to align interpretations of the rubric.
The rubric and anchors remain visible in the annotation interface throughout the study.

\paragraph{Blinding and randomization.}
To reduce expectation bias and method-identification bias, evaluation is double-blind:
annotators are not told which system produced any response, and all system identifiers are removed.
For scalar scoring, the order of candidate responses (when multiple systems are shown for the same prompt) is randomized.
For pairwise preference, left/right ordering is randomized.
Item order is shuffled independently per annotator to mitigate ordering and fatigue effects.

\paragraph{Preventing leakage and confounds.}
Annotators only see the user prompt and the candidate response(s).
They do not see model names, decoding methods, hyperparameters, or any metadata that could reveal method identity.
We enforce consistent formatting across systems (e.g., identical rendering of markdown and whitespace) to avoid visual cues.

\paragraph{Quality control: attention checks and gold items.}
We insert two types of quality-control (QC) items throughout the annotation stream.

\emph{Attention checks} verify careful reading.
These include (i) instruction-following checks that explicitly request a specified score (or a specified pairwise choice) for that item,
and/or (ii) obvious sanity checks where one response is clearly irrelevant or nonsensical and should be rated much lower (or lose in pairwise comparison).

\emph{Gold items} verify rubric adherence and score calibration.
We pre-label a small set of items with reference scores (for scalar) and/or reference winners (for pairwise) by the authors following the same rubric.
Annotators are expected to match the reference within a tolerance (e.g., $\pm 1$ on the 0--10 scale, or selecting the reference winner in pairwise).

QC items are interleaved uniformly at random (e.g., $\approx 10\%$ of all items, split between attention checks and gold items).
Annotators who repeatedly fail attention checks (e.g., $>2$ failures) or deviate substantially on gold items
(e.g., $>20\%$ of gold items outside the tolerance) are excluded; their affected annotations are discarded and re-assigned.
We additionally filter anomalous behavior such as extremely short completion times or near-constant scoring patterns.

\paragraph{Aggregation and reporting.}
For scalar scoring, we aggregate the $n=5$ ratings by reporting the empirical mean and dispersion (e.g., standard deviation),
and compute risk-/tail-sensitive metrics (e.g., CVaR) when relevant.
For pairwise preference, we report win-rate (and tie-rate if applicable) aggregated across annotators, along with confidence intervals via bootstrap when reported.

\paragraph{Independence assumption (discussion).}
Our theoretical analysis models the $n$ ratings for a fixed $(s,y)$ as i.i.d.\ draws from a stationary distribution.
In practice, ratings may exhibit mild dependencies due to annotator-specific biases or session effects.
We mitigate these effects via randomization, blinding, workload distribution across annotators, and multiple annotators per item.
Importantly, the proposed decoding rules remain applicable without strict i.i.d.\ assumptions; the assumption is used to obtain clean concentration-style guarantees,
and can be relaxed to weakly dependent sampling with adjusted constants.


\end{document}